%% file: NoisyTRNNM.tex
\documentclass[twocolumn]{IEEEtran}
\usepackage[utf8]{inputenc}
\usepackage{amsthm}
\usepackage{amsmath}
\usepackage{import}
\usepackage{algorithm,algorithmic}
\usepackage{lipsum}
\usepackage{color} 
\usepackage{multirow}
\usepackage{graphicx}
\usepackage{float}
\usepackage{subfigure}
\usepackage{booktabs}
\usepackage{cite}
\usepackage{multirow}
 \usepackage[normalem]{ulem}
 \usepackage{adjustbox}
 \usepackage{emptypage}
 \usepackage{amsfonts}
 \usepackage{bm}
 \usepackage{mcode}
 \usepackage{amsmath,amssymb}
 
\usepackage[resetlabels,labeled]{multibib}
\newcites{s}{References}
 
 \input{mathCommand.tex}

\newcommand{\Po}[1]{ \ensuremath{\mathfrak{X}{(#1)} }}
\newcommand{\Poo}[1]{ \ensuremath{\mathfrak{X}^*{(#1)} }}
\newcommand{\errorT}{\ensuremath{\tensor{E}}}
\newcommand{\errorM}{\ensuremath{\mat{E}_{(k,s)} }}
\newcommand{\hatT}{\ensuremath{\hat{\tensor{T}}}}
\newcommand{\starT}{\ensuremath{\tensor{T}^*}}
\newcommand{\mTk}{\ensuremath{\tensor{M}^k} }
\newcommand{\mMk}{\ensuremath{\mat{M}_{(k,s)}^k }}
\newcommand{\tT}{\ensuremath{\tensor{T}}}
\newcommand{\tM}{\ensuremath{\mat{T}_{(k,s)} }}
\newcommand{\trnn}[1]{\ensuremath{\|#1 \|_{\text{trnn}}}}
\newcommand{\trop}[1]{\ensuremath{\|#1 \|_{\text{{trnn}}}^*}}
\newcommand{\tropL}[1]{\ensuremath{\|#1 \|_{\emph{trnn}}^*}}
\newcommand{\frakX}[1]{\ensuremath{\mathfrak{X}(#1)}}
\newcommand{\frakXt}[1]{\ensuremath{\mathfrak{X}^*(#1)}}
\newcommand{\starTm}{\ensuremath{\mat{T}_{(k,s)}^*}}

\newcommand{\nn}[1]{\ensuremath{\|#1 \|_*}}
\newcommand{\fro}[1]{\ensuremath{\|#1 \|_{F}}}
\newcommand{\fros}[1]{\ensuremath{\|#1 \|_{F}^2 }}
\newcommand{\expectation}[1]{\ensuremath{\mathbb{E}[#1] }}
\newcommand{\expectations}[1]{\ensuremath{ \mathbb{E}^{2}[#1] }}
\newcommand{\tTt}{\ensuremath{\tilde{\tensor{T}}}}
\newcommand{\tMt}{\ensuremath{\tilde{\mat{T}}_{(k,s)}}}

\begin{document}	
	\title{Noisy Tensor Completion via Low-Rank Tensor Ring }
	\author{Yuning~Qiu, Guoxu~Zhou, Qibin~Zhao, \IEEEmembership{Senior Member, IEEE}, and Shengli~Xie, \IEEEmembership{Fellow, IEEE}
		 \thanks{This work is supported in part by Natural Science Foundation of China under Grant 61673124, Grant 61903095, Grant 61727810, and Grant 61973090, in part by Guangdong  Natural Science Foundation under Grant 2020A151501671. (\emph{Corresponding authors: Guoxu Zhou.})}
		\thanks{Yuning Qiu is with the School of Automation, Guangdong University of Technology, Guangzhou 510006, China (e-mail: yuning.qiu.gd@gmail.com).}
		\thanks{Guoxu Zhou is with the School of Automation, Guangdong University of Technology, Guangzhou 510006, China and also with the Key Laboratory of Intelligent Detection and The Internet of Things in Manufacturing, Ministry of Education, Guangdong University of Technology, Guangzhou 510006, China (e-mail: gx.zhou@gdut.edu.cn).}
		\thanks{Qibin Zhao is with the Center for Advanced Intelligence Project (AIP), RIKEN, Tokyo, 103-0027, Japan (e-mail: qibin.zhao@riken.jp).}
		\thanks{Shengli Xie is with the School of Automation, Guangdong University of Technology, Guangzhou 510006, China, and also with the Guangdong- Hong Kong-Macao Joint Laboratory for Smart Discrete Manufacturing, Guangdong University of Technology, Guangzhou 510006, China (e-mail: shlxie@gdut.edu.cn).}
	}
	\markboth{Journal of \LaTeX\ Class Files,~Vol.~14, No.~8, August~2015}%
	{Shell \MakeLowercase{\textit{et al.}}: Bare Demo of IEEEtran.cls for IEEE Journals}
	 \maketitle
	\begin{abstract}
	Tensor completion is a fundamental tool for incomplete data analysis, where the goal is to predict missing entries from partial observations. However, existing methods often make the explicit or implicit assumption that the observed entries are noise-free to provide a theoretical guarantee of exact recovery of missing entries, which is quite restrictive in practice. To remedy such drawback, this paper proposes a novel noisy tensor completion model, which complements the incompetence of existing works in handling the degeneration of high-order and noisy observations. Specifically, the tensor ring nuclear norm (TRNN) and least-squares estimator are adopted to regularize the underlying tensor and the observed entries, respectively. In addition, a non-asymptotic upper bound of estimation error is provided to depict the statistical performance of the proposed estimator. Two efficient algorithms are developed  to solve the optimization problem with convergence guarantee, one of which is specially tailored to handle large-scale tensors by replacing the minimization of TRNN of the original tensor equivalently with that of a much smaller one in a heterogeneous tensor decomposition framework. Experimental results on both synthetic and real-world data demonstrate the effectiveness and efficiency of the proposed model in recovering noisy incomplete tensor data compared with state-of-the-art tensor completion models.
	\end{abstract}
	
	\begin{IEEEkeywords}
		Tensor completion, tensor ring decomposition, low-rank tensor recovery, image/video inpainting
	\end{IEEEkeywords}
 
\section{Introduction}
\IEEEPARstart{A}{tensor} is an  array of numbers, giving a faithful and effective representation to maintain the intrinsic structure of multi-dimensional data \cite{TamaraG.KoldaBrettW.Bader}. Many data collected in real-world applications can be naturally expressed as high-order tensors. For example, a color video sequence can be viewed as a fourth-order tensor due to its spatial, color  and  temporal variables; a light field image can be formulated as a fifth-order tensor indexed by one color, two spatial, and two angular variables. For this reason, numerous theoretical and numerical tools for tensor data analysis have been  developed and applied in many fields, including computer vision \cite{hu2016moving}, multi-view clustering \cite{xie2018unifying,zhang2018generalized},  blind source separation \cite{cichocki2015tensor}, bio-informatics \cite{zhou2016linked},  pattern recognition \cite{guo2016support}, etc.
Among them, tensor completion is one of the most practical and significant problems, which aims at predicting the missing entries of a low-rank tensor with partial observations. Recent works have demonstrated that the incomplete low-rank tensor can be exactly or approximately recovered under some appropriate assumptions, and have been successfully implemented to broad tensor completion applications \cite{Ee2014,Yuan2016,Huang2020a}, such as, images and videos inpainting \cite{Hu2017, Zhao2015b, Xue2019},  multi-relational link prediction \cite{ermics2015link,nickel2013tensor},  and achieve promising performance. 

Low-rank tensor completion (LRTC)   can be viewed as a multi-dimensional extension of low-rank matrix completion  (LRMC)  which aims at recovering the intrinsic low-rank matrix with incomplete observations \cite{candes2009exact,candes2010power,recht2010guaranteed,Koltchinskii2000,Klopp2014}. 
However, this extension is rather nontrivial since it is difficult to find a well-defined  tensor rank for complex multilinear structure. Two most popular tensor rank definitions are CANDECOMP/PARAFAC (CP)  rank \cite{kiers2000towards} and Tucker rank \cite{tucker1966some}. For a given tensor, CP decomposition factorizes a tensor into the sum of rank-1 components, and the minimum number of these rank-1 components is defined as the CP rank. However, CP decomposition often suffers from the issue of ill-posedness, i.e., the optimal low-rank approximation does not always exist, which may lead to a poor fit in practical applications \cite{haastad1990tensor,de2008tensor}.
Compared with CP rank, Tucker rank is a more promising alternative, which is defined on the rank of  unfolding matrices  along each mode \cite{TamaraG.KoldaBrettW.Bader}. 
Such treatment is often convenient and reasonable, since it can obtain multi-dimensional low-rank structure maintained in high-order tensors  \cite{JiLiu2010b,Liu2016,Xu2015a}. 
However, this unbalanced (i.e., one versus the others) unfolding  scheme often yields fat matrices that have been blamed for their poor restoration performance in LRMC problem. As a result, the performance of Tucker rank based tensor completion methods often tends to degrade due to this inevitable issue \cite{Bengua2017a,Ee2014}.

\begin{figure*}[t]
	\centering
	\includegraphics[width=18cm]{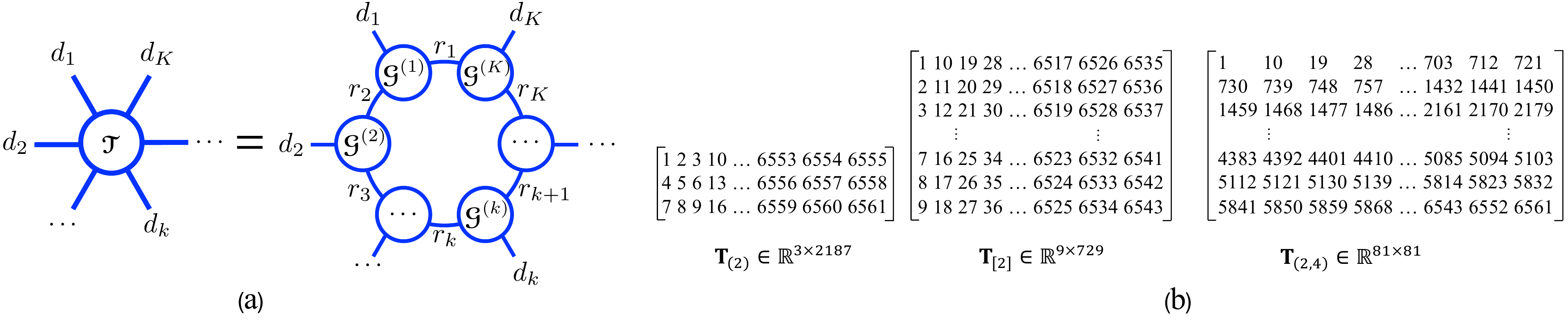}
	\caption{Illustration of TR decomposition and different unfolding matrices. (a) A graphical representation of TR decomposition; (b) comparisons of different unfolding matrices  on an eighth-order tensor $\tT\in \mathbb{R}^{3\times3\times3\times3\times3\times3\times3\times3} $ where the rank of unfolding matrices is associated with its Tucker, TT and TR rank, respectively. The synthetic tensor $\tT$ is generated by MATLAB command: $ \tT =\mcode{reshape}(1:6561, [3,3,\cdots,3])$.}
	\label{fig-trd}
\end{figure*}

In recent years, several novel tensor rank definitions have been proposed to achieve state-of-the-art tensor completion performance, e.g., tensor train (TT) rank \cite{Wen2008}, tensor tree rank \cite{Liu2019b}, and tensor ring (TR) rank \cite{zhao2016tensor}. Particularly, TR rank is one of the most promising methods. According to the circular structure of TR decomposition, given a high-order tensor,  the TR rank minimization problem can be squared to multiple rank minimization subproblems of square unfolding matrix \cite{Yu2019c}, which can conceptually solve the existing issues of Tucker and TT rank. Fig. \ref{fig-trd} shows the intuitive representation of TR decomposition and differences between three unfolding schemes. In \cite{Wang2017}, Wang et al. first introduced a low-rank TR completion model by alternatively minimizing the core tensors. To avoid manually adjusting TR rank, Yuan et al. \cite{Yuan2019a} established a new TR rank minimization scheme by minimizing nuclear norm on all core tensors. 
In \cite{Yu2019c}, Yu et al. proposed a tensor ring nuclear norm (TRNN) minimization model for LRTC. To guarantee the completion performance of TRNN, Huang et al. \cite{Huang2020a} further demonstrated that given a tensor satisfied strong TR incoherence condition, it can be exactly recovered with high probability. In \cite{Yu2020b}, Yu et al. proposed a parallel matrix factorization based method for low-TR-rank tensor completion.

Most LRTC methods mentioned above assume that the observed entries are noiseless. 
However, in real-world applications, data with missing issue caused by sensor failures, storage errors, or other mistakes also usually suffers from noise corruptions. 
Thus, it is rather desirable to predict tensor data from both noisy and incomplete observations. Although noisy tensor completion problem is more practical and worthwhile, there is very limited literature and only a few related works on this aspect. 
Tensor decomposition or total variation regularization based models have been proposed to alleviate the noisy completion problems \cite{acar2011scalable,Zhao2015b,yokota2019simultaneous}, however the statistical performance of all these methods is still unclear.  For this reason, Wang et. al introduced several tensor tubal rank based methods to solve third-order noisy tensor completion problem with theoretical guarantee \cite{wang2017near,Wang2018,Wang2019}.
Albeit interesting and promising, tensor tubal rank is very sensitive to the choice of the third-mode, and particularly difficult to capture  complex intra-mode and inter-mode correlations for high-order tensors.

In this paper, we propose to recover tensor data from  incomplete noisy observations, which substantially generalize noisy matrix completion problem by not only deriving theoretical results and practical applications  of noisy completion to high-order tensors, but also providing more scalable and efficient algorithms.  The main contributions of this paper can be summarized as following.
\begin{enumerate}
	\item We propose a novel tensor completion model to predict the missing entries with noisy observations based on low-rank TR model, which complements the incompetence of existing works in handling the degeneration of high-order and noisy observations. We analyze a non-asymptotic upper bound of the estimation error to reveal the statistical performance of the proposed model, which has been further proved to be optimal up to a logarithm factor in a minimax sense.
	\item We derive two algorithms based on alternating direction method of multipliers (ADMM)  to solve the optimization problem with convergence guarantee, namely, noisy tensor ring completion (NTRC) and Fast NTRC (FaNTRC). For FaNTRC, we minimize the TRNN equivalently on a much smaller tensor in a heterogeneous tensor decomposition framework, which has been shown to significantly improve the computational efficiency, especially dealing with large-scale low-TR-rank tensor data.
	\item The proposed NTRC and FaNTRC are successfully implemented in various noisy tensor data completion problems, the experimental results demonstrate their superiority compared with state-of-the-art LRTC methods.
\end{enumerate}

The remainder of this paper is organized as follows.
Section \ref{section-notations} presents a brief review on some notations and preliminaries. In Section \ref{section-proposed-method}, we introduce the noisy tensor ring completion model,  and analyze its non-asymptotic upper bound on Frobenius norm,  which has been further proved to be near-optimal in a minimax sense. In Section \ref{section-optmization-algorithm}, we develop two efficient algorithms to solve the optimization problem with convergence guarantee. Section \ref{section-experiments} gives experimental results of noisy tensor completion tasks on both synthetic and real-world tensor data.

\begin{table}[]
	\caption{Notations Used in This Paper}
	\label{table-Notations}
	\centering
	\begin{tabular}{lp{6cm}p{6cm}} \hline \hline
		Notations & Descriptions \\ \hline 
		$\tensor{A}$, \ensuremath{\mat{A}}, \ensuremath{\mat{A}(:,r)}, \ensuremath{\mat{A}(i,r)}      & A tensor, matrix, the $r$th column and $(i,r)$th entry of $\mat{A}$, respectively.  \\
		$\mat{I}_{R_k}$& An identity matrix of size $R_k \times R_k$.  \\
		$\left \langle \tensor{A}, \tensor{B} \right \rangle$ & Inner product of two tensor $\tensor{A}$ and $\tensor{B}$: $\langle \tensor{A} , \tensor{B} \rangle = \sum_{i_1} \sum_{i_2} \cdots \sum_{i_K} \tensor{A}({i_1i_2 \cdots i_K})  \tensor{B} ({i_1 i_2 \cdots i_K})$. \\
		vec(\ensuremath{\mat{A}})                         & Vectorizing a matrix \ensuremath{\mat{A}} into a vector.  \\
		$\| \mat{A} \|$ & Spectral norm of $\mat{A}$.\\
		$\| \mat{A} \|_*$ & Nuclear norm of $\mat{A}$.\\
		$\| \tensor{A} \|_{\infty}$ & $\ell_{\infty}$ norm of  \tensor{A}: $\| \tensor{A} \|_{\infty} = \|\text{vec}(\tensor{{A}})\|_{\infty}$. \\
		$\| \tensor{A} \|_F$ & Frobenius norm of $\tensor{A}$: $\| \tensor{A}\|_F = \sqrt{\langle \tensor{A}, \tensor{A} \rangle}$.\\
		$\trnn{\tensor{A}}$ & Tensor ring nuclear norm of $\tensor{A}$. \\ 
		$\trnn{\tensor{A}}^*$ & The dual norm of tensor ring nuclear norm.  \\
		$\text{rank}_{\text{tr}}(\tensor{A})$ & TR rank of   $\tensor{A}$.  \\
		$[K]$ & The set of all positive integer less than $K$, i.e., $[K]=\{1,2,\cdots, K\}$.  \\
		\ensuremath{\otimes}                              & The Kronecker product. \\
		\hline \hline 
	\end{tabular}
\end{table}

\section{Notations and Preliminaries} \label{section-notations}
\subsection{Notations}
We list some notations and their corresponding abbreviations in Table \ref{table-Notations}. 
Let $\bm{\varepsilon} \in \mathbb{R}^{N} $ be i.i.d. Rademacher sequences. Let $c$ and its derivatives like $c_0$, $c_1$, etc., be the generic absolute constants. Let $D$ be the number of total entries of the given tensor, i.e., $D=\prod_{k=1}^{K} d_k$. For any $a, b \in \mathbb{R}$, we let $a\wedge b = \min \{a, b\}$, and $a\vee b = \max \{a,b\}$. For the circular mode-$(k,s)$  unfolding matrix $\mat{T}_{(k,s)}$ of size $d_{1,k} \times d_{2,k}$, we let $\tilde{d}_k = d_{1,k} + d_{2,k}$,  $\hat{d}_{k} =   {d}_{1,k} \wedge  {d}_{2,k} $, $\check{d}_{k} =  d_{1,k} \vee d_{2,k}$, and $k^* = \arg\min_{k\in[K]} (d_{1,k} \wedge d_{2,k} )$. 

\subsection{Tensor Preliminaries}

\begin{definition} [TR decomposition \cite{zhao2016tensor}]
	The tensor ring (TR) decomposition represents a $K$th-order tensor $\tT \in  \mathbb{R}^{d_1 \times \cdots \times d_K}$ by the circular multilinear product over a sequence of third-order core tensors, i.e., $\tensor{T} = \emph{TR}( \tensor{G}^{(1)}, \cdots, \tensor{G}^{(K)} )$, where $\tensor{G}^{(k)} \in \mathbb{R}^{r_k \times d_k \times r_{k+1}}$, $k \in [K]$, and $r_{K+1} = r_{1}$. Element-wisely, it can be represented as
	\begin{equation}
		\tensor{T}({i_1, i_2, \cdots, i_K}) = \sum_{\upsilon_1, \cdots , \upsilon_K }^{r_1, \cdots, r_K} \prod_{k=1}^{K} \tensor{G}^{(k)} ( \upsilon_{k},  i_k,\upsilon_{k+1} ).
	\end{equation}
	The size of cores, $r_k, k=1,\cdots, K$, denoted by a vector $[r_1, \cdots, r_K]$, is called {TR rank}. 
\end{definition}

\begin{remark}[Multiple states in TR decomposition] \label{remark-states-of-TR} 
	TR decomposition can be divided into three states, i.e., supercritical ($d_k < r_k r_{k+1} $), critical ($d_k = r_k r_{k+1} $), and subcritical ($d_k > r_k r_{k+1}$) states \cite{ye2018tensor}. Previous work \cite{Wang2017} has also experimentally found this fact in some real-world tensor data in the sense that letting $r_k r_{k+1} > d_k$ to achieve more favorable recovery performance.  
\end{remark}

\begin{lemma} \label{lemma-multi-state}
	Tensor data with TR supercritical or critical states is  full-Tucker-rank. 
\end{lemma}
Proof of Lemma \ref{lemma-multi-state} is a part of proof of Lemma \ref{lemma-tr-tucker}, it can be found in Appendix \ref{proof-tr-tucker} of supplementary material. According to Lemma \ref{lemma-multi-state},  incomplete tensor with TR supercritical or critical states may not be recovered by Tucker rank minimization-based methods, since it is a full-Tucker-rank approximation problem. 

Next, we give three tensor-matrix unfolding schemes, and show their relationship with tensor ranks.

\begin{definition}[Multi-index operation \cite{cichocki2014era}]
The multi-index operation is given by
	\begin{equation}
		\begin{split}
					\overline{i_1,i_2,\cdots,i_K}  = & i_1 + (i_2 -1) d_1 + (i_3 - 1 ) d_1d_2 + \cdots  \\
					& + (i_K-1)d_1\cdots d_{K-1},
		\end{split}
	\end{equation}
where $i_k \in [d_k], k \in [K]$.
\end{definition}

\begin{definition} [Canonical mode-$k$ unfolding  \cite{TamaraG.KoldaBrettW.Bader}] \label{def-canonical-mode-k}
	Let $\tT$ be a $K$th-order tensor, its canonical mode-$k$ unfolding is denoted by $\mat{T}_{(k)}$ of size $d_k \times \prod_{j\neq k}^{K} d_j$, whose elements are given by 
	\begin{equation}
	\mat{T}_{(k)} (i_k, \overline{i_1 \cdots i_{k-1} i_{k+1} \cdots i_{K}}) = \tT (i_1, i_2,\cdots, i_K).
	\end{equation}
\end{definition}

\begin{definition} [First $k$-modes unfolding \cite{Wen2008}]  \label{def-first-mode-k}
	Let $\tT$ be a $K$th-order tensor, its first $k$-modes unfolding is denoted by $\mat{T}_{[k]}$ of size $\prod_{i=1}^k d_i \times \prod_{j=k+1}^{K} d_j$, whose elements are given by 
	\begin{equation}
	\mat{T}_{[k]} (\overline{i_1 i_2 \cdots i_k}, \overline{i_{k+1} \cdots  i_{K}}) = \tT (i_1, i_2,\cdots, i_K).
	\end{equation}
\end{definition}

\begin{definition} [Circular mode-$(k,s)$ unfolding \cite{Yu2019c}]
	Let $\tT$ be a $K$th-order tensor, its circular mode-$(k,s)$  unfolding is a matrix, denoted by $\mat{T}_{(k,s)}$ of size $d_{1,k} \times d_{2,k}$, where $d_{1,k} = \prod_{u=l+1}^{k-1} d_{u}$, $d_{2,k} = \prod_{v=k}^{l} d_{v}  $, and $s$ is the number of indexes included in $d_{2,k}$, and 
	\begin{equation} \label{eq-definition-circular-k}
	l = \left\{\begin{split}
	&k+s-1, ~~~~~~~~~ k+s \leq K, \\
	&k+s-1-K,~~~ \emph{ otherwise}.
	\end{split} \right.
	\end{equation}
	Alternatively, its element-wise form is given by $\mat{T}_{(k,s)}( \overline{i_{l+1} \cdots i_{k-1}}, \overline{\underbrace{i_{k} \cdots i_{l}}_{s \emph{ indexes}} }  ) = \tT(i_1,i_2,\cdots,i_K )$.
\end{definition} 
 Given an arbitrary tensor $\tT \in \mathbb{R}^{d_1\times d_2 \times \cdots \times d_K}$, the circular mode-$(k,s)$ unfolding of $\tT$ can be easily implemented by functions \mcode{reshape} and \mcode{permute} in MATLAB, i.e.,
\begin{equation*}
	\begin{split}
			\mat{T}_{(k,s)} = \mcode{reshape}(\mcode{permute}&(\tT, [l+1, \cdots, k-1, \underbrace{k, \cdots, l}_{s \text{ indexes}}]), \\
		& \prod_{u=l+1}^{k-1} d_u, \prod_{v=k}^{l}d_v).
	\end{split}
\end{equation*}

\begin{remark} [{ Relationship between different tensor ranks and unfolding schemes}] \label{remark-1}
	{Given an arbitrary $K$th-order tensor with TR rank $[r_1, \cdots, r_K]$,  the rank of each circular mode-$(k,s)$ unfolding matrix is bounded by $r_k r_{k+s}$ \cite{Yu2019c}. 
	This relationship can also be found in Tucker (or TT) rank between canonical mode-$k$ (or first $k$-modes) unfolding \cite{Wen2008,TamaraG.KoldaBrettW.Bader}. Therefore, the complex tensor rank minimization problem can be equivalently  reduced to a series of matrix rank minimization subproblems.}
\end{remark}
{ 
According to Remark \ref{remark-1}, to minimize TR rank, a natural option is to consider the sum of rank of unfolding matrices:
\begin{equation} \label{eq-tr-rank-minimization}
	\min_{\tT} ~\sum_{k=1}^{K} \alpha_k \text{rank}(\mat{T}_{(k,s)} ).
\end{equation}
However, problem (\ref{eq-tr-rank-minimization}) is computational intractable in general. Motivated by the proxy of rank function, the sum of nuclear norm  has been adopted as a convex surrogate of (\ref{eq-tr-rank-minimization}), and revealed in the following definition. }
\begin{definition}[Tensor ring nuclear norm (TRNN) \cite{Yu2019c}]
Let $\tT$ be a tensor with TR rank $[r_1, \cdots, r_K]$, then its TRNN is given by:
	\begin{equation} \label{eq-trnn-definition}
		\|{\tT}\|_{\emph{trnn}} = \sum_{k=1}^{K} \alpha_k \| \tM \|_*,
	\end{equation}
	where $\alpha_{k} \in [0,1]$ and $\sum_{k=1}^{K} \alpha_{k} = 1$  corresponds to the weight of mode-$(k,s)$ unfolding.
\end{definition}

	{
 
%
TRNN is defined by the sum of nuclear norm on circular mode-$(k,s)$ unfolding matrices, which is similar with Tucker nuclear norm (TcNN) \cite{JiLiu2010b} and tensor train nuclear norm (TTNN) \cite{Bengua2017a}.  
TcNN always minimizes the nuclear norm on multiple fat matrices due to its one versus the others unfolding (see Definition \ref{def-canonical-mode-k}), while the unfolding matrices of TTNN are usually fat or thin matrices when $k$ approaches to one or $K$ (see Definition \ref{def-first-mode-k}). 
Instead, by simply setting $s=\lceil K/2 \rceil$, TRNN is defined on a series of square matrices. 
Interestingly,  minimizing nuclear norm on these square matrices often tends to obtain exact recovery guarantee with small number of observations in LRMC problem \cite{Candes2012a}, since it is easier to achieve $\text{rank} (\mat{T})\ll d_1 \wedge d_2$ than the fat or thin matrices with $d_1 \wedge d_2 \ll d_1 \vee d_2 $. 
Additionally, according to Lemma \ref{lemma-multi-state}, incomplete tensor with TR supercritical state and TR critical state may not be recovered via TcNN. However, TRNN can avoid such issue, since the number of column and row of the circular unfolding matrix is of order $\mathcal{O}(d^{\lceil K/2 \rceil} )$ and $\mathcal{O}(d^{\lfloor K/2\rfloor})$, respectively, while the rank is merely of order $\mathcal{O}(r^2)$, indicating the sufficient low-rank structure in circular unfolding matrix. Moreover, due to the square unfolding on each order,  the weights $\alpha_k, k\in [K]$ of TRNN can be simply set to $1/K$ rather than carefully adjusting the optimal value as TcNN and TTNN. 
}

Next, we give the dual norm of TRNN in the following Lemma, which will play a significant role in statistical analysis of the proposed estimator. 
\begin{lemma} \label{lemma-dual-norm}
The  dual norm of TRNN is defined by
	\begin{equation}
	\| \tensor{T} \|_{\emph{trnn}}^* = \inf_{\tensor{Y}^1 + \cdots + \tensor{Y}^K =\tensor{T}} \max_{k=1,\cdots, K} {\alpha_k}^{-1} \| \mat{Y}_{(k,s)}^{k} \|,
	\end{equation}
	where $\tensor{Y}^k$ denotes the $k$th latent component.
\end{lemma}
The proof of the  Lemma  \ref{lemma-dual-norm} is given in Appendix \ref{sec-proof-dual-norm} of supplementary material. 

\begin{definition} [Tensor uniform sampling]\label{assumption-uniform-sampling}
	The sampling tensors $\tensor{X}_n$ are i.i.d. random tensor bases drawn from uniform distribution ${\Pi}$ on the set $\{ \mat{e}_{i_1} \circ \mat{e}_{i_2} \circ \cdots \circ \mat{e}_{i_K}: \forall (i_1, i_2, \cdots, i_K) \in [d_1] \times [d_2] \times \cdots \times [d_K] \}$.
\end{definition}

\section{Noisy Tensor Completion via Low-rank Tensor Ring} \label{section-proposed-method}
Herein, we introduce a noisy LRTC model based on low-rank tensor ring. Subsequently, we analyze its statistical performance by establishing a non-asymptotic upper bound of estimation error on Frobenius norm. Before giving the proposed model, we first introduce two assumptions on the unknown tensor and the distribution of noise, respectively.

\begin{assumption} \label{assumption-bound-infty}
	Suppose $\ell_{\infty}$ norm of $\tT$ is upper bounded by a positive constant $\delta$, that is, $\| \tT \|_{\infty} \leq \delta$.
\end{assumption}

\begin{assumption}  \label{assumption-noise-sub-exp} 
	Suppose random variables $\xi_n, n \in [N]$, are independent and centered sub-exponential variables with unit variance, that is, there exists a constant $K_{\bm{\xi}} >0$, such that, $\max_{n\in[N]} \expectation{|\xi_n| /K_{\bm{\xi}}} < \infty$.
\end{assumption}
Note that these assumptions have also been widely used for noisy matrix and tensor  completion problems in previous works \cite{Candes2011a,Negahban2012c,Wang2018a,Wang2019}. Assumption \ref{assumption-bound-infty} reveals that the  entries of $\tT$ are non-spiky, which is a common phenomenon in real-world tensor data, such as images, video sequences, and recommendation systems. For assumption \ref{assumption-noise-sub-exp}, the sub-exponential distribution is a general distribution class,  including Gaussian, sub-Gaussian exponential, and Poisson distributions.
 
\subsection{The Problem Formulation}
Recovering a $K$th-order tensor from $N$ noisy and partially observed entries can be formulated as the following noisy tensor completion problem:
	\begin{equation} \label{eq-observation}
	\begin{split}
			y_n = \langle\tT, \tensor{X}_n \rangle + \sigma \xi_n ,~~ n \in [N],
	\end{split}
	\end{equation}
where $y_n$ is the $n$-th observed entry, $\tT \in \mathbb{R}^{I_1 \times \cdots \times I_K}$ denotes the unknown low-rank tensor,  $\tensor{X}_n, n \in [N]$, are i.i.d. random tensor basis drawn from uniform distribution, $\xi_n$ and $\sigma$ are the random noise and standard deviation, respectively. 

Note that (\ref{eq-observation}) is an ill-posed problem without introducing any priors in $\tT$. Such dilemma can be handled by the following TR rank minimization problem:
\begin{equation} \label{eq-tensor-ring-regularization}
	\begin{split}
		&\min_{\tensor{T}}  \operatorname{rank}_{\operatorname{tr}}(\tensor{T}),  \\
		&\text{s.t. }\frac{1}{2} \| \mat{y} - \mathfrak{X}(\tensor{T}) \|_2^2 \leq \epsilon, \| \tT \|_{\infty} \leq \delta, \\
	\end{split}
\end{equation}
where $\epsilon$ controls the noise level,  and $\text{rank}_{\text{tr}}(\cdot)$ denotes the TR rank.  $\mathfrak{X}(\cdot)$ is the tensor uniform sampling operator, given by $\mathfrak{X}(\tT)=(\langle \tensor{X}_1, \tT\rangle, \langle \tensor{X}_2, \tT\rangle, \cdots, \langle \tensor{X}_N, \tT\rangle)^{\top} \in \mathbb{R}^{N}$. Since (\ref{eq-tensor-ring-regularization}) is an NP-hard problem, according to (\ref{eq-trnn-definition}),  we relax it by minimizing the following convex TRNN function:
\begin{equation} \label{eq-tensor-ring-regularization-norm}
\begin{split}
&\min_{\tensor{T}}  \trnn{\tensor{T}},  \\
&\text{s.t. }\frac{1}{2} \| \mat{y} - \mathfrak{X}(\tensor{T}) \|_2^2 \leq \epsilon, \| \tensor{T} \|_{\infty} \leq \delta. 
\end{split}
\end{equation}
Furthermore, we can equivalently formulate the inequality constraint to the following  regularization problem:
\begin{equation} \label{eq-tr-noise-model}
\begin{split}
	&\min_{\tensor{T}} \frac{1}{2} \| \mat{y} - \mathfrak{X}(\tensor{T}) \|_2^2 + \lambda  \trnn{\tensor{T}} , \\
& \text{s.t. } \| \tensor{T} \|_{\infty} \leq \delta, 
\end{split}
\end{equation}
where $\lambda$ is a positive scalar to achieve a trade-off between the degree of fitting error of the observed entries and low-TR-rank penalty. 

\subsection{Non-asymptotic Upper Bound} 
Based on the above noisy tensor completion problem, we analyze a non-asymptotic upper bound  of the estimation error on the Frobenius norm to reveal the statistical performance of the proposed estimator (\ref{eq-tr-noise-model}). 
\begin{theorem} \label{theorem-main-result}
	Let $\mathfrak{X}(\cdot)$ be a  tensor uniform sampling operator, the true low-TR-rank tensor $\tT^*$ and random noise variables $\xi_n$ satisfy Assumptions \ref{assumption-bound-infty} and \ref{assumption-noise-sub-exp}, respectively. If $\lambda \geq c_0 \sigma \sqrt{{N \log (\tilde{d}_{k^*})}/{\hat{d}_{k^*}}}$, then we have probability at least $1-3/\tilde{d}_{k^*}$, such that 
	\begin{equation} \label{eq-main-result}
	\begin{split}
	\frac{\fros{\tT^* - \hat{\tT} }}{D} \leq  c_1\max &\left\{    ( \delta^2 \vee \sigma^2)  \frac{K \check{d}_{k^*}R_{tr} \log (\tilde{d}_{k^*}  )} { N } , \right. \\ 
	& ~~~~~~~~~~~~~~~~~~~~\left.\delta^2 \sqrt{ \frac{ \log \tilde{d}_{k^*} }{ N}}\right\},
	\end{split}
	\end{equation}
	where $R_{tr}=(\sum_{k=1}^{K} \alpha_k \sqrt{r_k r_{k+s}} )^2$, $c_o$ and $ c_1$ denote positive constants, $\hatT$ is the flexible solution of problem (\ref{eq-tr-noise-model}). 
\end{theorem}

{ 
We leave the proof of Theorem \ref{theorem-main-result} in Appendix \ref{appendices-a} of supplementary material. The result shows that for incomplete low-TR-rank tensor, the per-entry estimation error can be upper bounded by the righ-hand side of (\ref{eq-main-result}) with high probability. 
Note that when $K=2$, TRNN reduces to matrix nuclear norm.  
Thus, the noisy matrix completion is a special case of the proposed method, and the non-asymptotic upper bound of noisy matrix completion in Theorem 7 in \cite{Klopp2014} is a special case of our Theorem \ref{theorem-main-result}.  
Additionally, the tensor algebras are more complicated than those in noisy matrix cases, making the proofs different from \cite{Klopp2014}. 
For example,  our proofs require establishing the upper bound of the spectral norm on random tensors (see Lemmas \ref{lemma-bound-tr-op} and \ref{lemma-bound-tr-expectation}). It is more challenging than matrix completion since some important results in matrix's case, e.g., Bernstein inequality, do not have equivalent form for  high-order tensors.   
The proofs of Theorem \ref{theorem-main-result} also require generalizing several important properties of TRNN, such as the decomposability of TRNN and  inequality relationship between Frobenius norm and TRNN on tensors.
Furthermore, as discussed in the following paragraph, our result in Theorem \ref{theorem-main-result}  is more promising than that of matrix and Tucker rank-based methods. 
 }

 { 
For simplicity, we assume $r_k=r, d_k =d$ and $\alpha_k = 1/K, k \in [K]$, Theorem \ref{theorem-main-result} states that the estimation error satisfies 
\begin{equation} \label{eq-sample-complexity}
		\frac{\fros{\tT^* - \hat{\tT} }}{D}   \lesssim \frac{r^2 K d^{ \lceil \frac{K}{2} \rceil } \log(d^{ \lfloor \frac{K}{2} \rfloor}  +  d^{\lceil \frac{K}{2} \rceil  } ) }{N},
\end{equation}
with high probability. Thus, the sample complexity of the proposed estimator is
\begin{equation} \label{eq-sample-complexity-2}
	{\mathcal{O}(  r^2 K d^{ \lceil \frac{K}{2} \rceil  }    \log (d^{\lfloor \frac{K}{2} \rfloor }  + d^{\lceil \frac{K}{2} \rceil } ) )}.
\end{equation} 
Note that (\ref{eq-sample-complexity-2}) is suboptimal  when compared to TR decomposition with degree of freedom $\mathcal{O}(r^2dK)$. This is actually a common issue in sum of nuclear norm convex surrogates for low-rank tensor recovery problems  \cite{oymak2015simultaneously,tomioka2011statistical,Gu2014a,Li2019d}. 
Nevertheless, the result in Theorem \ref{theorem-main-result} is still promising. The sample complexity is substantially much lower than the number of tensor entries $d^K$. Additionally, even when $r$ approaches $d$, the sample complexity of our estimator is still significantly lower than $d^K$. This result is surprising since it is not aligned with matrix and Tucker rank-based methods which require even full observations as the rank approaches to $d$. 



}


\subsection{Minimax Lower Bound}  
In this section, we present a minimax lower bound of the proposed Theorem \ref{theorem-main-result}. We let $\inf_{\hat{\tT}}$ be the infimum over all the flexible solution of our model, and $\sup_{\tT^*}$ be the supremum over all the "true" tensor $\tT^*$. We come up with the following theorem. 
\begin{theorem} \label{theorem-minimax-result}
	Suppose that the random variables $\xi_n$ are \emph{i.i.d.} Gaussian $\mathcal{N}(0,\sigma^2)$, $n\in [N]$, and $\sigma >0$.  Then there exist absolute constants $\varrho \in (0,1)$ and $c>0$, such that 
	\begin{equation} \label{eq-minimax-result}
		\inf_{\hat{\tT}} \sup_{\tT^*} \mathbb{P}_{\tT^*} \left( \frac{\| \tT^* - \hat{\tensor{T}} \|_F^2}{D} > c (\delta^2 \wedge \sigma^2) \frac{K \check{d}_{k^*} R_{tr}  }{N} \right) \geq \varrho,
	\end{equation}
	where $R_{tr}=(\sum_{k=1}^{K} \alpha_k \sqrt{r_k r_{k+s}} )^2$.
\end{theorem} 
The proof of Theorem \ref{theorem-minimax-result} can be found in Appendix \ref{proof-minimax-result} of supplementary material. 
Comparing (\ref{eq-main-result}) and (\ref{eq-minimax-result}), we can observe that the rate is minimax optimal up to a logarithm factor.

\subsection{Comparisons with Previous Work}
\label{sec-comparisions-with-previous-work}

Both tensor ring with balanced unfolding (TRBU) \cite{Huang2020a} and the proposed estimator are able to give statistical or exact recovery performance for low-TR-rank tensor completion. Here, we analyze the superiority of the proposed estimator. 
\begin{enumerate}
	\item TRBU requires structural assumption on core tensors, which  extends matrix strong incoherence conditions by letting canonical mode-2 unfolding of core tensors not be aligned with the standard basis, that is,
	\begin{equation}
	\| {\mat{G}_{(2)}^{(k) }{\mat{G}_{(2)}^{(k)}}^\top}   - \frac{r_k r_{k+1} }{d_k} \mat{I}_{d_k} \|_{\infty} \leq \mu_k \sqrt{\frac{ r_k r_{k+1} }{d_k}}, ~k\in[K],
	\end{equation}
	where $\mu_k >0, k \in [K]$. Compared with TRBU, priors on our estimator are significantly mild. The only assumption on unknown tensor is its upper bound on $\ell_{\infty}$ norm,  which can be easily verified in real-world data. However, it is still unclear how to estimate the existence of strong TR incoherence conditions on incomplete tensor data. Furthermore, even if the incomplete tensor is coherent, the proposed model can still be reliable since it haven't involved with the structure assumption on the intrinsic low-TR-rank tensor. 
	\item TRBU can only obtain exact recovery performance with high probability in noiseless tensor completion case. The proposed model achieves the statistical performance for noisy observations, which is more practical and feasible in real-world applications. 
\end{enumerate}
{  
	We show the superiority of the proposed estimator compared with tensor completion (tensor recovery) methods based on CP rank, Tucker rank and TT rank, respectively.
	\begin{enumerate}
		\item In \cite{yuan2016tensor}, Yuan and Zhang claimed that recovering incomplete third-order tensor  of size $d\times d\times d$ with CP rank $r_{\text{cp}}$ required sample size $r_{\text{cp}}^{1/2}(d\log d)^{3/2}$. However, the optimal low-CP-rank approximation is NP-complete and is typically intractable. 
		\item  Tomioka et al. \cite{tomioka2011statistical} claimed that recovering $K$th-order  tensor of size $d\times \cdots \times d$ with Tucker rank $[r_{\text{tc}},\cdots,r_{\text{tc}}]$ requires $\mathcal{O}(r_{\text{tc}}d^{K-1} )$ Gaussian measurements for tensor compressive sensing problem, and Mu et. al further reduced the complexity to                                                                                                                                                                                                                                                                                                                                                                                                                                                                                                                                                                                                                                                                                                                                                                                                                                                                                                                                                                                                                                                                                                                                                                                                                                                                                                                                                                                                                                                                                                                                                                                                                                                                                                                                                                                                                                                                                                                                                                                                                                                                                                                                                                                                                                                                                                                                                                                                                                                                                                                                                                                                                                                                                                                                                                                                                                                                                                                                                                                                                                                                                                                                                                                                                                                                                                                                                                                                                                                                                                                                                                                                                                                           $\mathcal{O}(r_{\text{tc}}^{\lfloor \frac{K}{2} \rfloor } d^{\lceil \frac{K}{2} \rceil} )$. However, Tucker rank-based method suffers from the similar flaw as LRMC, that is, as $r_{\text{tc}}$ approaches $d$, nearly full observations are required to recover the incomplete tensor.
		\item Rauhut et al.  \cite{Rauhut2017} proposed to solve  low-TT-rank tensor compressive sensing using iterative hard thresholding (IHT) algorithm, and provided a convergence result based on tensor restricted isometry property (TRIP).  The authors claimed that recovering a $K$-order tensor of size $d \times d \times  \cdots \times d$ with TT rank $[r_{\text{tt}},\cdots,r_{\text{tt}}]$ using sub-Gaussian linear maps satisfied TRIP with level $\delta_{ r}$ required  $\mathcal{O} ((\delta_{ r}^{-2} (K-1)r_{\text{tt}}^{3} + Kdr_{\text{tt}})\log (Kr_{\text{tt}}) ) $ measurements. Unfortunately, the sampling condition for tensor completion \emph{does not} hold for the TRIP. Additionally, the optimal low-rank projection for IHT algorithm is also intractable in practice. 
	\end{enumerate}
}
Now, we demonstrate the advantages of the proposed model over noisy tubal rank tensor completion models. 
\begin{enumerate}
	\item  Tensor tubal rank based models \cite{wang2017near,Wang2018,Wang2019}, transform a tensor into  Fourier domain in the specific third-order and then estimate its low-rank approximation on all frontal slice matrices. 
	In contrast, the proposed model can not only be easily extended to high-order tensor, but also be insensitive to the dimensional orientation.
	\item As mentioned in previous  works \cite{TamaraG.KoldaBrettW.Bader,Wen2008,zhao2016tensor}, high-order tensors usually maintain low-rank structure in multiple orders. The proposed method can efficiently exploit the low-rank structure in high-order tensors, which is seen as one of the  most effective advantages compared with matrix and tubal rank based methods. The experimental results in Section \ref{section-experiments} also verified this statement.
\end{enumerate}

\section{Optimization Algorithms} \label{section-optmization-algorithm}
In this section, we introduce two ADMM based algorithms, namely, noisy tensor ring completion (NTRC) and fast NTRC (FaNTRC) to solve the optimization problem. 
\subsection{NTRC Optimization Algorithm}
 In problem (\ref{eq-tr-noise-model}), the variable $\{ \tM \}_{k=1}^{K}$ shares the same entries $K$ times in computing TRNN, which makes it difficult to solve the optimization problem directly. Thus, we introduce $K$ auxiliary  variables $\{ \mTk \}_{k=1}^{K}$ to relax such constraint: 
\begin{equation} \label{eq-constrainted-problem}
\begin{split}
\min_{\tensor{T}} & \frac{1}{2} \| \mat{y} - \Po{\tT} \|_2^2 + \lambda  \sum_{k=1}^{K} \alpha_{k}\|\mMk\|_* + \kappa_\delta^{{\infty}}(\tT), \\
& \text{s.t. }  \mTk = \tT, k \in [K], 
\end{split}
\end{equation}
where $\tensor{M}^k \in \mathbb{R}^{d_1\times d_2 \times \cdots \times d_K}$ denotes the $k$th auxiliary variable. The $\kappa_\delta^{\infty}(\cdot)$ denotes the $\ell_{\infty}$ norm indicator function:
\begin{equation}
	\kappa_{\delta}^{\infty} (\tT) =\left\{ \begin{split}
		0, ~~ &\text{if } \tT \in \mathbb{S}_{\delta},\\
		\infty, ~~&\text{otherwise}, \\
	\end{split} \right.
\end{equation}
where $\mathbb{S}_{\delta} := \{ \tT \in \mathbb{R}^{d_1\times d_2 \cdots \times d_K},  \| \tT \|_{\infty} \leq \delta \}$. In order to solve the equality constraints, we formulate the augmented Lagrangian function of (\ref{eq-constrainted-problem}) as 
\begin{align} 
  \nonumber \ell_{\mu}^{1}\left( \{ \tensor{M}^k\}_{k=1}^K, \{ \tensor{Q}^k\}_{k=1}^{K},\tensor{T} \right) =\frac{1}{2}  \| \mat{y} - \mathfrak{X}(\tensor{T} &) \|_2^2  + \kappa_\delta^{{\infty}}(\tT)  \\
\label{eq-lagrangian-problem} +  \sum_{k=1}^{K} \left( \lambda\alpha_k \nn{\mMk} 
+ \langle \tensor{Q}^k ,   \mTk - \tT \rangle + \right. &\left.\frac{\mu}{2} \fros{\tensor{M}^k - \tensor{T}} \right),
\end{align}
where $\tensor{Q}^k $ denotes the $k$th dual variable, and $\mu$ is a positive penalty scalar.  Note that it is rather difficult to simultaneously optimize multiple sets of variable in this objective function. An alternative scheme is to solve each set of variables using  alternating direction method (ADM) \cite{boyd2011distributed}. 

\begin{algorithm}[t]
	\caption{NTRC Optimization Algorithm}
	\begin{algorithmic}[1]
		\renewcommand{\algorithmicrequire}{\textbf{Input:}}
		\REQUIRE \ensuremath{\mat{y}, \{\tensor{X}_n \}_{n=1}^{N} }.
		\\ \textit{Initialisation}: \ensuremath{\{ \alpha_{k} \}_{k=1}^{K}, \mu^0=10^{-4}, \nu =1.1, \mu_{\text{max}} = 10^{10} },  $\text{tol}=10^{-6} $, zero filled with \ensuremath{\tensor{M}^k, \tensor{Q}^k, k\in[K]}, t=1.
		\WHILE {not convergenced}
		\FOR  {\ensuremath{k=1,\cdots,K}}
		\STATE {Update $\tensor{M}^{k,t+1}$ using (\ref{eq-update-Mk})}. 
		\ENDFOR 
		\STATE {Update $\tensor{T}^{t+1} $ using (\ref{eq-update-T-ntrc}).}
		\FOR  {\ensuremath{k=1,\cdots,K}}
		\STATE {Update $\tensor{Q}^{k,t+1}$ using (\ref{eq-update-Qk-ntrc}) }.
		\ENDFOR 
		\STATE \ensuremath{\mu^{t+1} =\min(\mu_{\text{max}}, \nu \mu^t).}
		\STATE Check the convergence condition: $\frac{\| \tensor{T}^{t+1} - \tensor{T}^t  \|_F} {\|  \tensor{T}^t\|_F} \leq \text{tol}$.
		\STATE {$t=t+1$}.
		\ENDWHILE
	\end{algorithmic} 
\end{algorithm}

\subsubsection{The $\{ \tensor{M}^k \}_{k=1}^K$-subproblems}
For $\{\mMk \}_{k=1}^{K}$ subproblems, they can be solved by updating $\mMk$ while fixing the others:
\begin{align}
\nonumber \mat{M}_{(k,s)}^{k,t+1} &=  \arg \min_{\mMk} \tau\nn{\mat{M}_{(k,s)}^{k} } \\
\nonumber & \qquad \qquad \quad + \frac{1}{2} \fros{ \mat{M}_{(k,s)}^{k}  - \mat{T}_{(k,s)}^{t} + \frac{1}{\mu} \mat{Q}^{k,t}_{(k,s)} } \\
&\label{eq-update-Mk} = \mathcal{P}_{\tau}^* ( \mat{T}_{(k,s)}^{t} - \frac{1}{\mu^t}\mat{Q}_{(k,s)}^{k,t}  ), 
\end{align}
where $\tau = {\lambda \alpha_k}/{\mu^t}$, and $\mathcal{P}_{\tau}^*(\mat{A})$ denotes the singular value thresholding (SVT)  \cite{cai2010singular} method:
\begin{equation}
	\mathcal{P}_{\tau}^{*} (\mat{A}) = \mat{U}  \text{max}(\bm{\Sigma }- \tau, 0) \mat{V}^{\top}, 
\end{equation}
where $[\mat{U}, \bm{\Sigma}, \mat{V}]= \text{SVD}(\mat{A})$. 

\subsubsection{The $\tT$-subproblem}
The $\tT$-subproblem can be equivalently formulated as the following vectorization form:
\begin{align}
	\label{eq-solving-T-proj}\text{vec}(\tensor{T}^{t+1} )= &\arg \min_{\tT} \kappa_{\delta}^{{\infty}} ( \text{vec}(\tT) ) \\
	\nonumber + \| \text{vec} (\tT) - &(\mat{X}^{\top} \mat{X} + \mu^t K \mat{I}_{D} )^{-1} \\
	&\qquad ( \mu^t \sum_{k=1}^{K} \text{vec}(\tensor{Q}^{k,t} + \tensor{M}^{k,t}) -\mat{X}\mat{y} ) \|_2^2, 
\end{align}
which can be solved by deriving the KKT condition: 
\begin{equation} \label{eq-update-T-ntrc}
\text{vec}(\tensor{T}^{t+1}) = \mathcal{P}_{\delta}^{\infty}(\mat{f}^{t+1} ) = \text{sgn}(\mat{f}^{t+1} ) \odot \min( | \mat{f}^{t+1}|, \delta ),
\end{equation}
where $\mat{f}^{t+1} = (\mat{X}^{\top} \mat{X} + \mu^t K \mat{I}_{D} )^{-1} ( \mu^t \sum_{k=1}^{K} \text{vec}(\tensor{Q}^{k,t} + \tensor{M}^{k,t+1}) -\mat{X}\mat{y} )$, and $\mat{X}(:,n)=\text{vec}(\tensor{X}_n)$.
\subsubsection{The $\{\tensor{Q}^{k} \}_{k=1}^{K}$-subproblems}
The dual variables $\{\tensor{Q}^{k} \}_{k=1}^{K}$ can be updated by the gradient ascent method:
\begin{equation} \label{eq-update-Qk-ntrc}
\tensor{Q}^{k,t+1} = \tensor{Q}^{k,t} + \mu^t (\tensor{M}^{k,t+1} - \tensor{T}^{t+1}), k \in [K].
\end{equation}
The details of the optimization procedure are summarized in Algorithm 1.

\begin{algorithm}[t]
	\caption{FaNTRC Optimization Algorithm}
	\begin{algorithmic}[1]
		\renewcommand{\algorithmicrequire}{\textbf{Input:}}
		\REQUIRE \ensuremath{\mat{y}, \{\tensor{X}_n\}_{n=1}^{N} }.
		\\ \textit{Initialisation}: \ensuremath{\{ \alpha_{k} \}_{k=1}^{K}}, $\eta=10^{-4}$,  $\nu$ =1.1, $\eta_{\text{max}} = 10^{10}$, $\text{tol}=10^{-6}$, zero filled with \ensuremath{\tensor{M}^k, \tensor{Q}^k, k\in[K]}, $t=1$.
		\WHILE {not convergenced}
		\FOR {$k=1, \cdots, K$}
		\STATE {Update $\mat{U}_k^{t+1}$ via solving (\ref{eq-update-Uk})}.
		\ENDFOR
		\STATE {Update $\tTt^{t+1}$ using (\ref{eq-update-tTt}).}
		\FOR  {\ensuremath{k=1,\cdots,K}}
		\STATE {Update $\tensor{L}^{k,t+1}$ using (\ref{eq-Lk-subproblem})}. 
		\ENDFOR 
		\STATE {Update $\tT^{t+1}$ using (\ref{eq-update-T})}.
		\STATE {Update $\tensor{P}^{t+1}$ and $\tensor{R}^{k,t+1}, k \in [K], $  using (\ref{eq-update-P}).}
		\STATE \ensuremath{\eta =\min(\eta_{\text{max}}, \nu \eta).}
		\STATE Check the convergence condition: $\frac{\| \tensor{T}^{t+1} - \tensor{T}^t  \|_F} {\|  \tensor{T}^t\|_F} \leq \text{tol}$.
		\STATE $t=t+1$.
		\ENDWHILE
	\end{algorithmic} 
\end{algorithm}
\subsection{Fast NTRC Optimization Algorithm}  \label{section-fasterNTRC}
The main computational cost of NTRC is singular value decomposition (SVD) for nuclear norm minimization on  unfolding matrices, which may prevent its applications in large-scale tensor completion problem. In order to alleviate such bottleneck, we introduce a fast NTRC (FaNTRC) algorithm, i.e., to minimize TRNN equivalently on a much smaller tensor. To establish the relationship between these two algorithms, we further show that its global optimal solution is exactly  the same as  that of NTRC.  

We first explore the relationship between TR  and Tucker decomposition. 
\begin{lemma} \label{lemma-tr-tucker}
	Given an arbitrary tensor $\tT\in \mathbb{R}^{d_1 \times \cdots \times d_K}$. If the  TR decomposition of $\tT$ is $\tT=\emph{TR}(\tensor{G}^{(1)}, \cdots, \tensor{G}^{(K)} )$, and $\tensor{G} \in \mathbb{R}^{r_k \times d_k \times r_{k +1}}$, $k \in [K]$. Then $\tT$ can be equivalently represented by Tucker decomposition format $\tT = \tilde{\tensor{T}} \times_1 \mat{U}_1 \cdots \times_K \mat{U}_K$, where $\mat{U}_k \in \mathbb{R}^{d_k \times R_k }, k \in [K]$, and $R_k=  ( r_k r_{k+1} \wedge d_k) $, denote the column orthogonal matrices.
\end{lemma}
We leave the proof of Lemma \ref{lemma-tr-tucker} in Appendix \ref{proof-tr-tucker} of supplementary material.  Lemma \ref{lemma-tr-tucker} depicts that given a tensor with TR rank $[r_1, \cdots, r_K]$,  it can  be also represented by  the Tucker decomposition format, which motivates us to come up with the following theorem. 

\begin{theorem} \label{theorem-rank-equality} %
	Given a tensor $\tT \in \mathbb{R}^{d_1 \times d_2 \cdots \times d_K}$ with TR rank $[r_1,\cdots, r_K]$, $\mat{U}_k \in St(d_k, R_k)$, $ ( r_k r_{k+1} \wedge d_k) \leq R_k \leq d_k$, $k\in[K]$ and $\tilde{\tT}$ satisfy ${\tT}=\tilde{\tT} \times_1 \mat{U}_1 \times_2 \mat{U}_2\cdots \times_K \mat{U}_K $, then we have: 
		\begin{equation}
	\|{\tT}\|_{\emph{trnn}} = \|{\tilde{\tensor{T}}} \|_{\emph{trnn}},
	\end{equation}
where $St(d_k, R_k):=\{ \mat{U}_k, \mat{U}_k\in \mathbb{R}^{d_k \times R_k}, \mat{U}_k^{\top} \mat{U}_k=\mat{I}_{R_k}  \}$ denotes the Stiefel manifold.
\end{theorem}

The proof of Theorem \ref{theorem-rank-equality} can be found in Appendix \ref{proof-theorem-rank-equality} of supplementary material.  

Equipped with Theorem \ref{theorem-rank-equality}, we  can  solve the high computational TRNN minimization problem  in a more efficient manner than the existing methods, i.e.,  computing SVD on a much smaller circular unfolding matrix $\tilde{\mat{T}}_{(k,s)} \in \mathbb{R}^{R_{1,k} \times R_{2,k}  }$, where $R_{1,k}=\prod_{j=l+1}^{k-1} R_j $ and $R_{2,k}=\prod_{i=k}^{l} R_i $.  
Though this similar strategy has been used in Tucker tensor nuclear norm minimization problem \cite{Liu2016,Liu2014}, our method is very different from that of these works since our method is implemented in a heterogeneous framework in the sense that equivalently minimizing TRNN in a Tucker decomposition format. 
For example, we have to explore the connection between TR  and Tucker decomposition to meet our analysis. Moreover, it is also crucial for us to establish the relationship of circular unfolding between the original tensor and its Tucker decomposition format. Therefore, these dilemmas make our analysis more complicated than those homogeneous frameworks that directly performing Tucker tensor nuclear norm on its Tucker decomposition format. Note also that these results are very fundamental and crucial to break the gap between different tensor decomposition models, and are also very promising to extend to other tensor decomposition models, e.g., tensor train and tensor tree decomposition. 

Consequently, according to Theorem \ref{theorem-rank-equality},  problem (\ref{eq-tr-noise-model}) can be equivalently reformulated as
\begin{equation}\label{eq-constrainted-problem-faster} 
\begin{split} 
\min_{\tensor{T}, \tilde{\tensor{T}}, \{\mat{U}_k \}_{k=1}^{K}} & \frac{1}{2} \| \mat{y} - \Po{\tT } \|_2^2 + \lambda \trnn{\tTt} + \kappa_{\delta}^{\infty} (\tT)  , \\
 & \text{s.t. } \tT=\tTt \times  {\mat{U}}_1 \cdots \times_K \mat{U}_K , \\ 
 & ~~~~ \mat{U}_k \in  St(d_k, R_k), k\in [K].
\end{split}
\end{equation} 
Now we show that with appropriate choice of $R_k, k \in [K],$ the global optimal solution of (\ref{eq-constrainted-problem-faster}) is the same as that of (\ref{eq-tr-noise-model}).

\begin{theorem} \label{theorem-optimal-solution}
	Suppose ($\{\mat{U}_k' \}_{k=1}^{K}, \tilde{\tensor{T}}' $) and $\tT'$ are the global optimal solutions of problem (\ref{eq-constrainted-problem-faster}) and problem (\ref{eq-tr-noise-model}) with $( r_k r_{k+1} \wedge d_k) \leq R_k \leq d_k, k\in [K]$, respectively, then $\tTt' \times_1 \mat{U}_1' \cdots \times_K \mat{U}_K'$ is also the optimal solution of problem (\ref{eq-tr-noise-model}). 
\end{theorem}
The proof of Theorem \ref{theorem-optimal-solution} can be found in Appendix \ref{proof-theorem-optimal-solution} of supplementary material.

 To tackle the constrained optimization problem (\ref{eq-constrainted-problem-faster}), we formulate its augmented Lagrangian function as 
\begin{equation} \label{eq-Lagrangisn-faster}
	\begin{split}
	\ell_{\eta}^2 (\tT,  \tTt, & \{ \mat{U}_k \}_{k=1}^K, \{\tensor{L}^k\}_{k=1}^{K},\tensor{P}, \{\tensor{R}^k\}_{k=1}^K ) =  \frac{1}{2} \fros{ \mat{y} - \Po{\tT } } \\
	+& \langle \tensor{P},  \tT - \tTt \times_1 \mat{U}_1 \cdots \times_K \mat{U}_K\rangle + \kappa_{\delta}^{\infty} (\tT)  \\
	 +& \frac{\eta}{2} \fros{ \tT - \tTt \times_1 \mat{U}_1 \cdots \times_K \mat{U}_K } + \lambda \sum_{k=1}^{K} \alpha_{k} \| \mat{L}^k_{(k,s)} \|_*  \\
	  +& \sum_{k=1}^{K} \left(   \langle \tensor{L}^k - \tTt, \tensor{R}^k \rangle  +\frac{\eta}{2}  \fros{\tensor{L}^k - \tTt}  \right), \\
	 & \text{s.t. } \mat{U}_k \in St(d_k, R_k), k\in [K], 
	\end{split}
\end{equation}
where $\tensor{P}$ and $\{\tensor{R}^k\}_{k=1}^K $  are dual variables, $\eta > 0$ is penalty parameter. Based on the augmented Lagrangian function (\ref{eq-Lagrangisn-faster}), we solve each subproblem alternatively by fixing the others.

\subsubsection{The $\{ \mat{U}_k\}_{k=1}^{K}$-subproblems} According to (\ref{eq-Lagrangisn-faster}), the $ \mat{U}_k$-subproblem can be formulated as the following maximization problem over factor matrix $\mat{U}_k$: 
\begin{equation}   \label{eq-Uk-subproblem} 
 \max_{ \mat{U}_k}	\langle \tilde{\tT}^{t} \times_{ m\neq k} \mat{U}_m^{t} ,  (\frac{1}{\eta^t} \tensor{P}^t + \tensor{T}^t ) \times_k \mat{U}_k^{\top} \rangle, ~\text{s.t. } \mat{U}_k \in St(d_k, R_k),
	\end{equation}
where $\tilde{\tensor{T}}^{t} \times_{m\neq k} \mat{U}_m^{t}$ denotes $\tilde{\tensor{T}}^{t} \times_1 \mat{U}_1^{t+1} \cdots \times_{m-1}  \mat{U}_{m-1}^{t+1} \times_{ m+1} \mat{U}_{m+1}^{t} \cdots \times_K \mat{U}_K^{t} $, and  it can be further reformulated as
\begin{equation} \label{eq-update-Uk}
	\begin{split}
			\mat{U}_{k}^{t+1}  = & \arg \max_{ \mat{U}_k}\trace{ \mat{U}_{k}^{\top}  ( \frac{1}{\eta^t} \mat{P}_{(k)}^t + {\mat{T}}_{(k)}^t )\mat{B}^t_{(k)} },  \\
			& ~\text{s.t. } \mat{U}_k \in St(d_k, R_k),
	\end{split}
\end{equation}
where $\mat{B}^t_{(k)}$ denotees the canonical mode-$k$  unfolding of $\tilde{\tensor{T}}^{t} \times_{m\neq k} \mat{U}_m^{t} $. 
 Note that  (\ref{eq-update-Uk}) is actually a well-known orthogonal procrustes problem \cite{higham1995matrix}, whose global optimal solution is given by  SVD, i.e.,  
$	\mat{U}_{k}^{t+1} = \tilde{\mat{U}}_k^{t}  \tilde{\mat{V}}_k^{t\top}$, 
where $ \tilde{\mat{U}}_k^t$ and $\tilde{\mat{V}}_k^{t\top}$ denote left and right singular vectors of $ (1/\eta^t \mat{P}_{(k)}^t +\mat{T}_{(k)}^t )\mat{B}_{(k)}^t$, respectively. By alternatively solving the maximization problem (\ref{eq-update-Uk}), we can obtain the optimal solution of $\{ \mat{U}_k \}_{k=1}^{K}$.
\subsubsection{The $\tTt$-subproblem} By fixing all the other variables, the $\tTt$-subproblem is given by:
\begin{equation} 
\begin{split}
	\min_{\tTt  }  ~   \fros{(\frac{1}{\eta^t}\tensor{P}^t + \tT^{t+1} ) - \tTt\times_1 \mat{U}_{1}^{t+1} \cdots \times_K \mat{U}_{K}^{t+1}  } &\\
	 + \sum_{k=1}^{K} \fros{ \frac{1}{\eta^t}\tensor{R}^{k,t} + \tensor{L}^{k,t} - \tTt}&, 
\end{split}
\end{equation}
whose closed-form solution can be obtained  directly:
\begin{equation} \label{eq-update-tTt}
\begin{split}
	\tTt^{t+1} = \frac{1}{K+1}  (\frac{1}{\eta^t} \tensor{P}^t + \tT ^{t}) \times_1 \mat{U}_{1}^{t+1,\top} \cdots \times_K \mat{U}_{K}^{t+1,\top}&  \\
+ \frac{1}{K+1} \sum_{k=1}^{K} \frac{1}{\eta^t} \tensor{R}^{k,t} + \tensor{L}^{k,t}&.
\end{split}
\end{equation} 
\subsubsection{The $\{ \tensor{L}^k \}_{k=1}^K$-subproblems}
The $\tensor{L}^k$ can be solved by minimizing the following subproblem: 
\begin{equation} \label{eq-Lk-subproblem}
	\begin{split}
		\mat{L}_{(k,s)}^{k,t+1}  = \arg \min_{\mat{L}_{(k,s)}^k } ~ &\frac{\lambda \alpha_k}{\eta^t} \| \mat{L}_{(k,s)}^k \|_*  
		 \\
		 + \frac{1}{2}&\| \mat{L}_{(k,s)}^k 
		 +\frac{1}{\eta^t}\mat{R}_{(k,s)}^{k,t} - \tilde{\mat{T}}_{(k,s)}^{t+1} \|_F^2.
	\end{split}
\end{equation}
Similar to (s), it can be easy to obtain its closed-form solution by the SVT method.

\subsubsection{The $ \tT$-subproblem}  We update the variable $\tT$ by fixing the others:
\begin{align}
\nonumber  \text{vec}(&\tT^{t+1}) =  \arg \min_{\tT} \kappa_{\delta}^{\infty}(\tT)  +  \|  \text{vec}(\tT) - (\mat{X}^{\top} \mat{X}   + {\eta} \mat{I}_{D} )^{-1} \\
&\label{eq-update-T} (\mat{X}\mat{y} -   \text{vec} ( \tensor{P}^{t}  - \eta^t \tTt^{t+1} \times_1 \mat{U}_{1}^{t+1} \cdots \times_K \mat{U}_{K}^{t+1} ) ) \|_2^2,
\end{align}
whose closed-form solution can be obtained as (\ref{eq-solving-T-proj}).

\subsubsection{The $(\tensor{P}, \{ \tensor{R}^k \}_{k=1}^K )$-subproblems} 
The dual variables $\tensor{P}$ and $\{\tensor{R}^k \}_{k=1}^K $ can be updated by the gradient ascent method:
\begin{align}
	&\label{eq-update-P} \tensor{P}^{t+1} = \tensor{P}^{t} + \eta^t ( \tT^{t+1} - \tTt^{t+1} \times_1 \mat{U}_{1}^{t+1} \cdots \times_K \mat{U}_{K}^{t+1}), \\
& \label{eq-update-R} \tensor{R}^{k,t+1} = \tensor{R}^{k,t} + \eta^t (\tensor{L}^{k,t+1} - \tTt^{t+1}), ~ k\in [K].
\end{align} 
The details of the optimization procedure are summarized in Algorithm 2.

\subsection{Computational Complexity Analysis} 
\label{sec-computational-complexity}
For simplicity, we assume that the observed tensor is of size $d_1 = \cdots = d_K = d$, the given rank satisfies $R_1 = \cdots = R_K = R$, and  $s=\lceil K/2 \rceil$. For Algorithm 1, the main per-iteration cost lies in (\ref{eq-update-Mk}), which requires computing SVD of $d^{K/2} \times d^{K/2}$ matrices. Therefore, the per-iteration complexity is $\mathcal{O}( Kd^{3K/2}  )$. For Algorithm 2, the main per-iteration cost lies in computing multiplication operation in (\ref{eq-Uk-subproblem}) and (\ref{eq-update-T}), and calculating SVD in (\ref{eq-update-Uk}) and (\ref{eq-Lk-subproblem}). The multiplication operations in (\ref{eq-Uk-subproblem}) and (\ref{eq-update-T}) involve complexity $\mathcal{O}(KRd^{K})$ and $\mathcal{O}(\sum_{k=1}^{K} R^{K-k+1} d^{k} )$, respectively. The computational costs of SVD in  (\ref{eq-update-Uk})  and (\ref{eq-Lk-subproblem}) are $\mathcal{O}(KR^{2} d^{K-1} )$ and $\mathcal{O}( K R^{3K/2} )$, respectively. Hence, the total computational cost in one iteration of Algorithm 2 are given by
\begin{equation}
\mathcal{O}( KRd^{K}  + KR^2d^{K-1} + \sum_{k=1}^{K} R^{K-k+1} d^{k}  + KR^{3K/2}).	
\end{equation} 
Note that when $R \ll d$, the computational cost of Algorithm 2 is significantly lower than that of Algorithm 1 and the other TRNN based methods.

\subsection{Convergence Analysis}
{  The proposed NTRC is the classical two-block convex ADMM optimization algorithm, and the convergence results can be directly induced according to  previous works  \cite{gabay1976dual,lin2010augmented}. }
 FaNTRC is a nonconvex ADMM algorithm for which the theoretical convergence guarantee is difficult to derive.  Instead, we provide weak convergence results for FaNTRC under mild condition. 
\begin{theorem} \label{theorem-convergence-fantrc}
	Let $(\tT^{t},  \tTt^{t},  \{ \mat{U}_k^t \}_{k=1}^K, \{\tensor{L}^{k,t}\}_{k=1}^{K},\tensor{P}^t, \{\tensor{R}^{k,t}\}_{k=1}^K ) $ be the sequences generated by Algorithm 2 (FaNTRC). Suppose that the sequence $\{ \tensor{P}^{t} \}$ is bounded, then we have the following conclusions.
	\begin{enumerate}
		\item $(\tT^{t},  \tTt^{t},  \{ \mat{U}_k^t \}_{k=1}^K, \{\tensor{L}^{k,t}\}_{k=1}^{K} ) $ are  Cauchy sequences.
		\item Any accumulation point of sequences $(\tT^{t},  \tTt^{t},  \{ \mat{U}_k^t \}_{k=1}^K, \{\tensor{L}^{k,t}\}_{k=1}^{K} ) $ satisfy the KKT conditions for problem (\ref{eq-constrainted-problem-faster}).
	\end{enumerate}
\end{theorem}
The proof sketches of  Theorem \ref{theorem-convergence-fantrc}  can be found in Appendix   \ref{sec-proof-convergence-fantrc} of the supplementary material.  In addition, we also experimentally investigate the RE  and PSNR  values versus the iteration number for the proposed NTRC and FaNTRC in Fig. \ref{fig-psnr-history}. We can clearly observe that the proposed methods usually converge after only approximately 50 iterations, indicating  the effective convergence behavior of the proposed algorithms.  

\section{Experimental Results} \label{section-experiments}
In this section, we conduct four experiments to evaluate noisy tensor completion performance of the proposed NTRC and FaNTRC on synthetic and real-world data sets including color images, light field images, and video sequences. All experiments are run in MATLAB 9.4 on a Linux personal computer equipped with dual Intel E5 2640v4 and 128GB of RAM. We then present numerical results to compare with several state-of-the-art  tensor completion methods including   {TMac-inc   \cite{Xu2015a} \footnote{https://xu-yangyang.github.io/TMac/} },  HaLRTC \cite{JiLiu2010b} \footnote{https://www.cs.rochester.edu/u/jliu/publications.html}, SiLRTC-TT \cite{Bengua2017a} \footnote{https://sites.google.com/site/jbengua/home}, TRLRF \cite{Yuan2019a}   \footnote{https://github.com/yuanlonghao/TRLRF}, TRNNM \cite{Yu2019c} \footnote{The code was provided by Dr. Jinshi Yu}, and NoisyTNN  \cite{Wang2019} \footnote{The code was provided by Dr. Andong Wang}. {Due to space limitation, further experimental results are  included in the supplementary material.}

The sampling method of all experiments is random sampling with respect to different sampling ratios (SR), which is given by ${\text{SR} = N/D \times 100\%}$, where \ensuremath{N} denotes the number of observed entries. To evaluate all the compared methods in terms of tensor completion performance, we  adopt two measure metrics,  namely relative  error (RE), and peak signal-to-noise ratio (PSNR). The RE is a common metric for recovery performance between the approximated tensor \ensuremath{\hat{\tT}} and the original one \ensuremath{{\tT}^*}, which is given by \ensuremath{ \text{RE} =  \| \hat{\tT} - \tT^* \|_F / \| \tT^* \|_F}.  The PSNR denotes the ratio between maximum possible power of a signal and the power of corrupted noise, and is given by
\begin{equation}
 \text{PSNR} = 10 \log_{10}  (\frac{D \|\hat{\tT} \|_{\infty} }{ \| \hat{\tT} - \tT^*\|_F }).
\end{equation}

\begin{figure} [t]
	\centering
	\subfigure[]{
		\begin{minipage}[t]{0.22\textwidth}
			\centering
			\includegraphics[width=4.2cm]{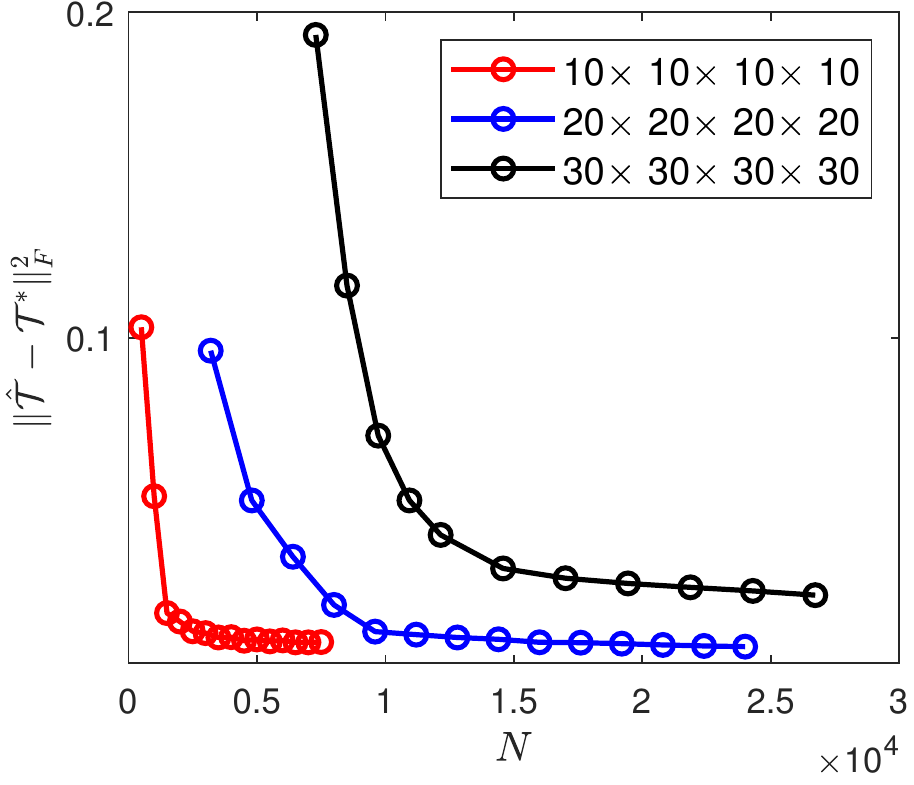}
		\end{minipage}
	}
	\vspace{0.000001cm}
	\subfigure[]{
		\begin{minipage}[t]{0.22\textwidth}
			\centering
			\includegraphics[width=4.2cm]{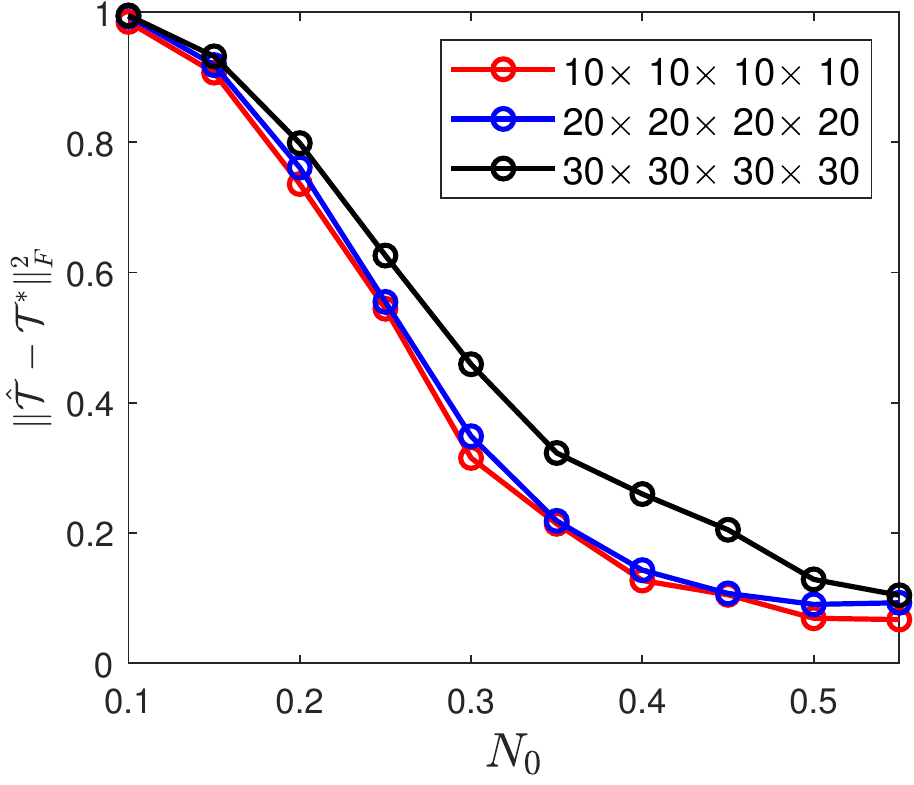}
		\end{minipage}
	}
	\caption{  Comparison of estimation error on different synthetic tensor. (a) Plots of estimation error versus the observation numbers  $N=pd^4$ with $d\in \{10,20,30\}$; (b) Plots of estimation error versus the normalized observation numbers  $N_0={N}/{(r^2 K d^{ \lceil \frac{K}{2} \rceil  }    \log (d^{\lfloor \frac{K}{2} \rfloor }  + d^{\lceil \frac{K}{2} \rceil } ) )}$ with $d\in \{10,20,30\}$. }
	\label{fig-synthetic-reN}
\end{figure}

\begin{figure} [t]
	\centering
	\subfigure[]{
		\begin{minipage}[t]{0.22\textwidth}
			\centering
			\includegraphics[width=4.2cm]{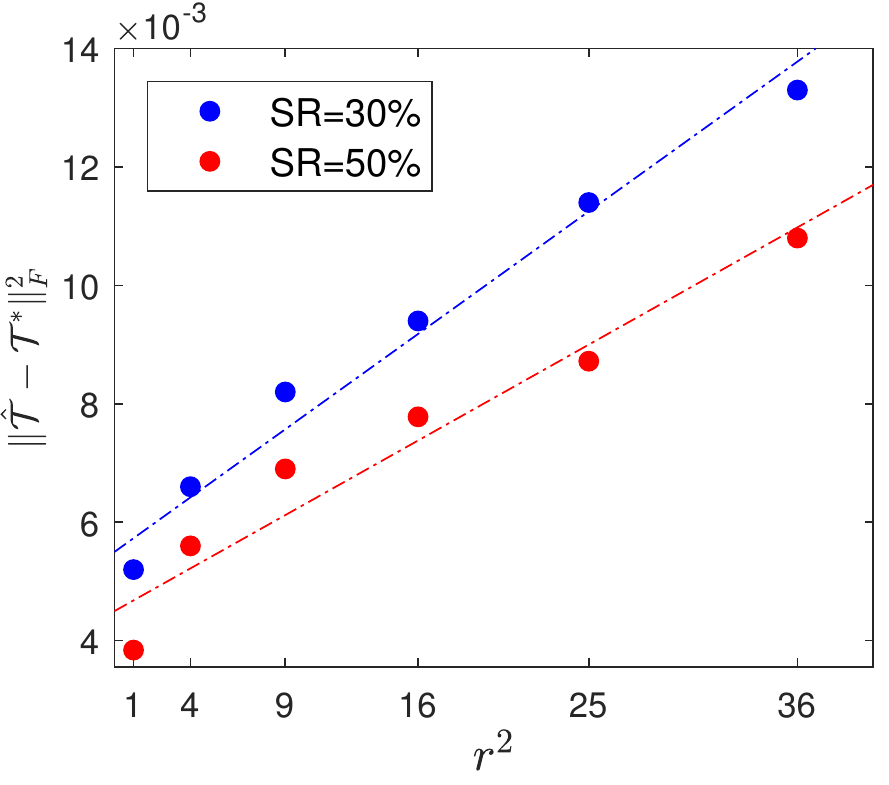}
		\end{minipage}
	}
	\subfigure[]{
		\begin{minipage}[t]{0.22\textwidth}
			\centering
			\includegraphics[width=4.2cm]{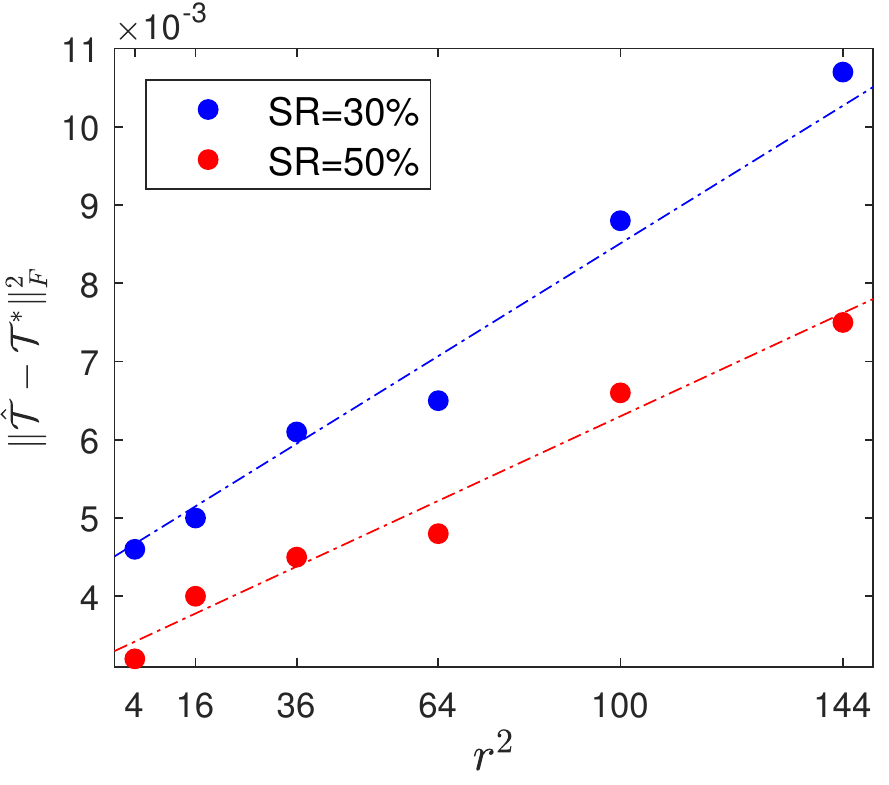}
		\end{minipage}	}
	\caption{  Plots of estimation error versus the square of TR rank. (a) Estimation error versus the square of TR rank on synthetic tensor of  size $\mathbb{R}^{10\times 10\times 10\times 10}$; (b) Estimation error versus the square of TR rank on synthetic tensor of  size $\mathbb{R}^{20\times 20\times 20\times 20}$.}
	\label{fig-synthetic-reR}
\end{figure}

\begin{figure} [t]
	\centering
	\subfigure[]{
		\begin{minipage}[t]{0.22\textwidth}
			\centering
			\includegraphics[width=4.2cm]{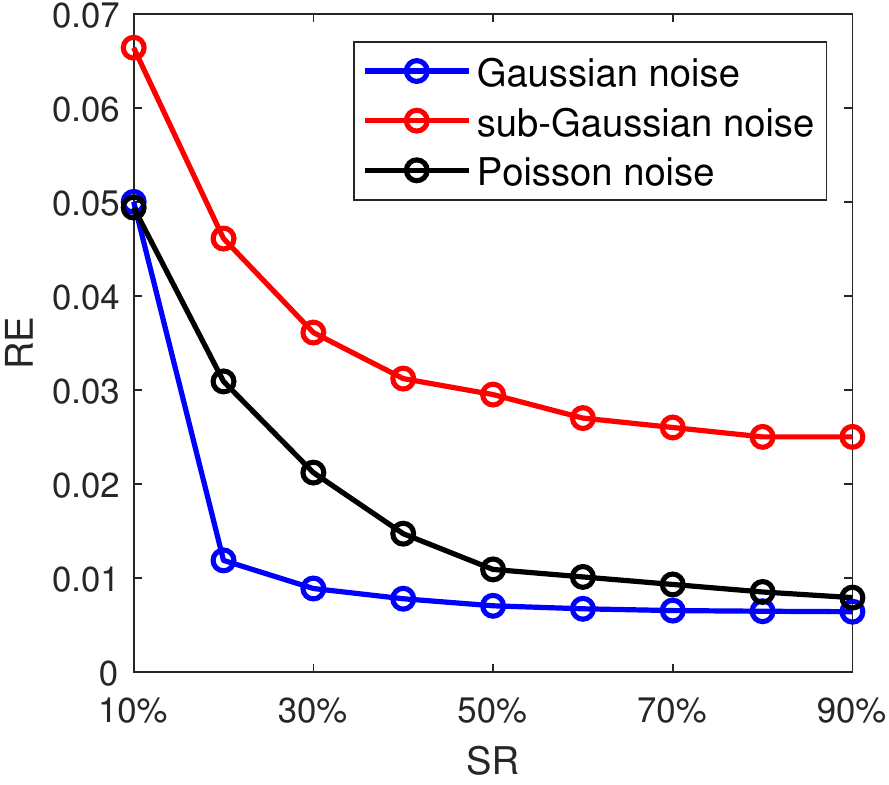}
		\end{minipage}
	}
	\vspace{0.000001cm}
	\subfigure[]{
		\begin{minipage}[t]{0.22\textwidth}
			\centering
			\includegraphics[width=4.2cm]{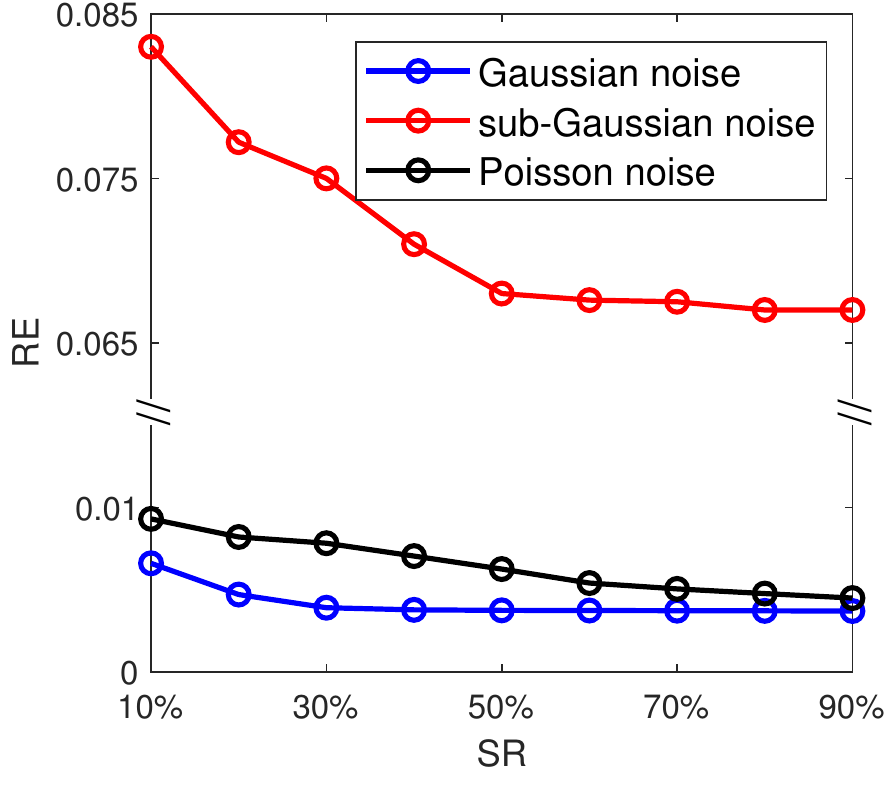}
		\end{minipage}
	}
	\caption{ Plots of RE under different noise distributions versus  SR on synthetic tensors. (a) RE versus SR on synthetic tensor of  size $\mathbb{R}^{10\times 10\times 10\times 10}$ with TR rank $r_k=\lceil \log^{1/2} d_k \rceil, k\in [4]$; (b) RE versus SR on  synthetic tensor of  size $\mathbb{R}^{20\times 20\times 20\times 20}$ with TR rank $r_k=\lceil \log^{1/2} d_k \rceil, k\in [4]$.}
	\label{fig-synthetic-Noise}
\end{figure}

\begin{figure} [t]
	\centering
	\subfigure[]{
		\begin{minipage}[t]{0.22\textwidth}
			\centering
			\includegraphics[width=4.2cm]{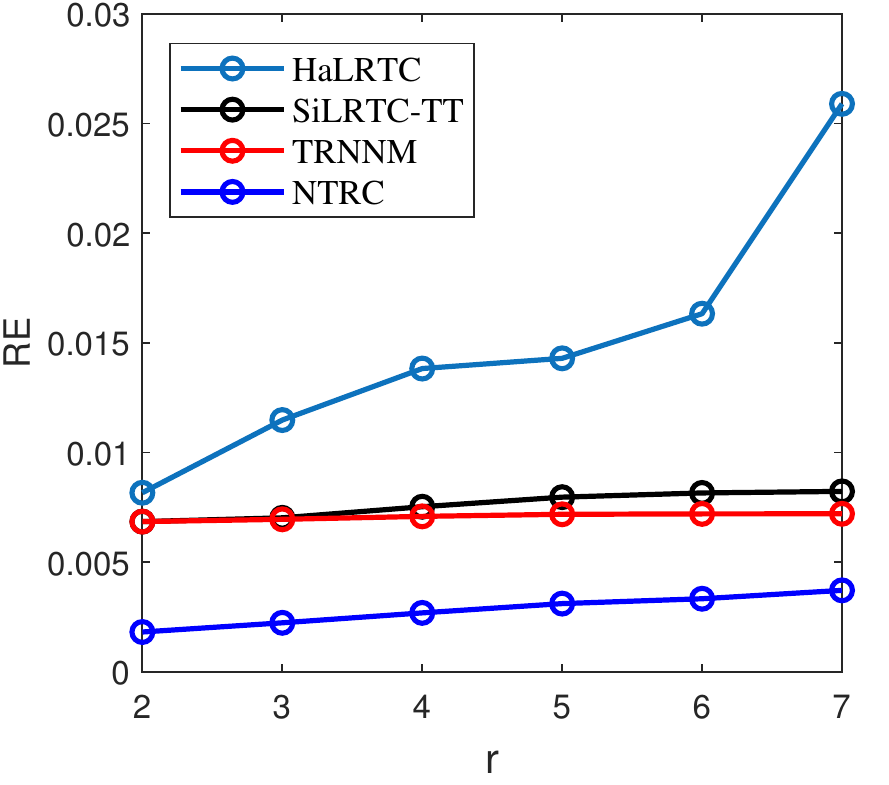}
		\end{minipage}
	}
	\vspace{0.000001cm}
	\subfigure[]{
		\begin{minipage}[t]{0.22\textwidth}
			\centering
			\includegraphics[width=4.2cm]{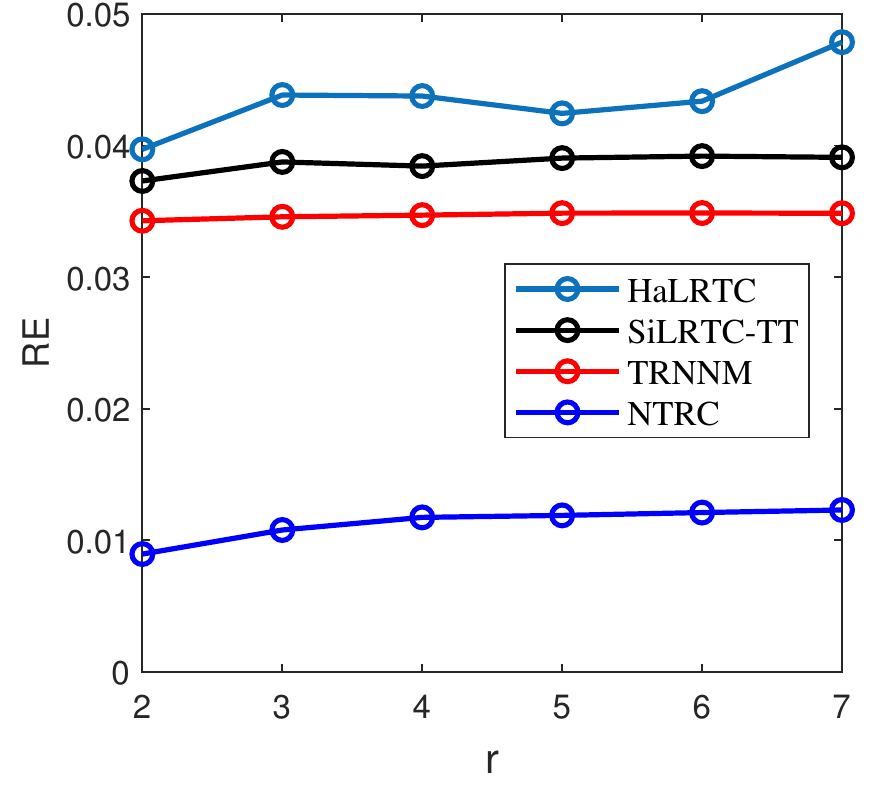}
		\end{minipage}
	}
	\caption{Plots of RE of different methods versus varying TR rank. (a) RE versus TR rank on synthetic tensor of size $\mathbb{R}^{30\times 30 \times 30 \times 30}$ with noise level $c=0.01$; (b) RE versus TR rank on synthetic tensor of size $\mathbb{R}^{30\times 30 \times 30 \times 30}$ with noise level $c=0.05$.}
	\label{fig-supercritical-results}
\end{figure}

\begin{figure} [t]
	\centering
	\subfigure[]{
		\begin{minipage}[t]{0.22\textwidth}
			\centering
			\includegraphics[width=4.2cm]{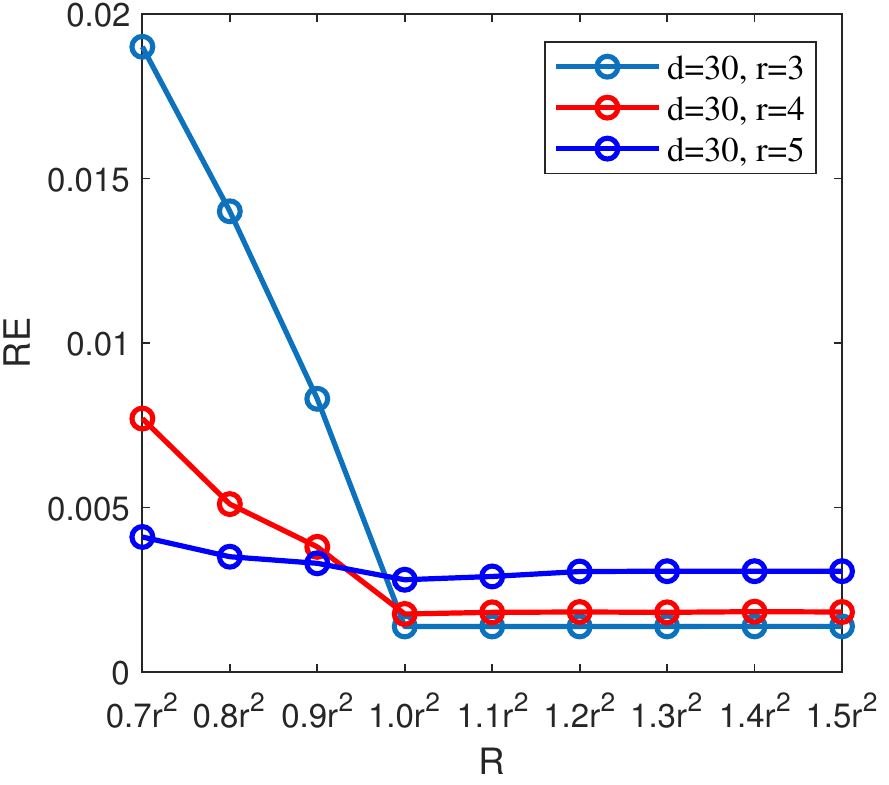}
		\end{minipage}
	}
	\vspace{0.000001cm}
	\subfigure[]{
		\begin{minipage}[t]{0.22\textwidth}
			\centering
			\includegraphics[width=4.2cm]{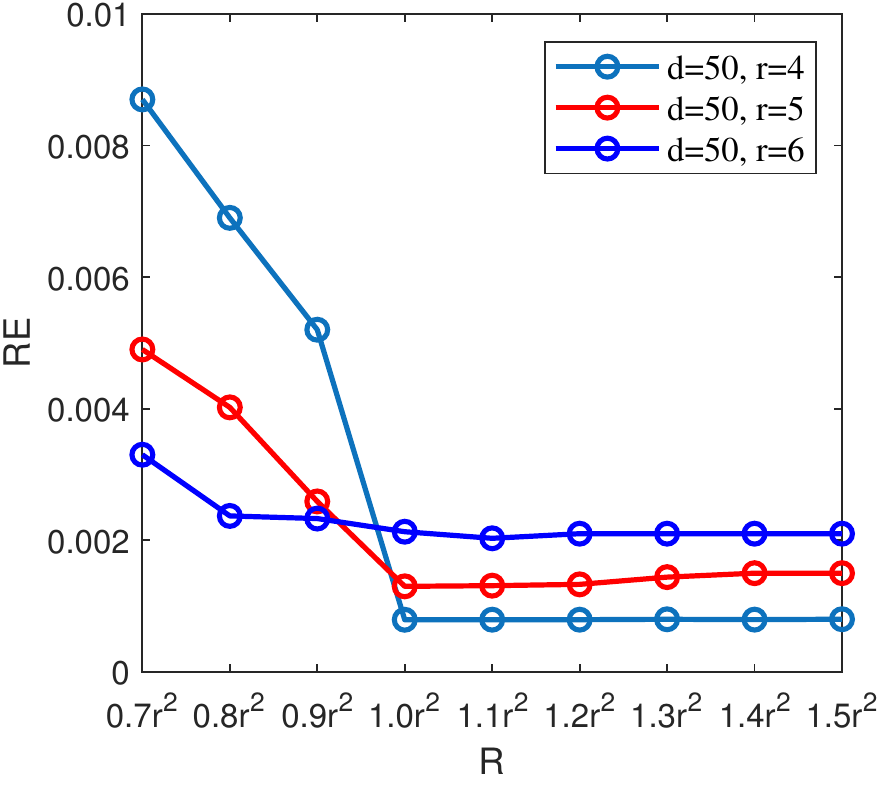}
		\end{minipage}
	}
	\caption{Plots of RE of FaNTRC versus different given rank $R$. (a) RE versus given rank on synthetic tensor of size $\mathbb{R}^{30\times 30 \times 30\times 30}$; (b) RE versus given rank on synthetic tensor of size $\mathbb{R}^{50\times 50 \times 50\times 50}$.}
	\label{fig-synthetic-data-results}
\end{figure}

\begin{figure}[t]
	\centering
	\includegraphics[width=9cm]{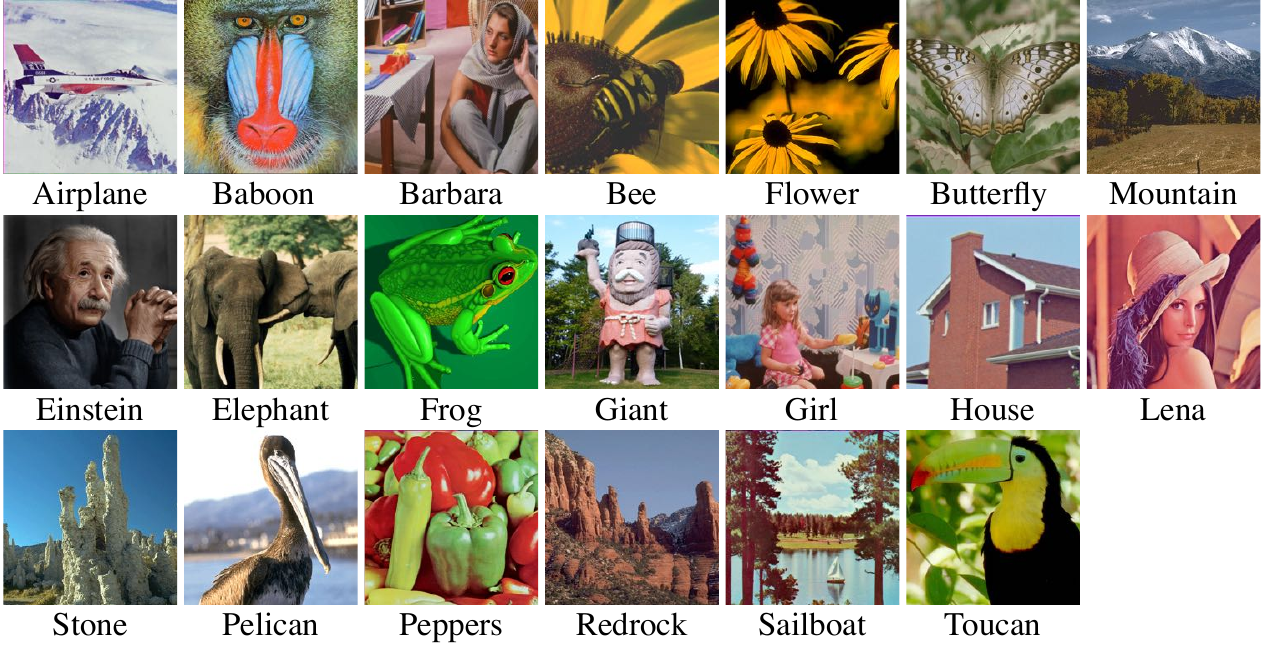}
	\caption{20 benchmark color images whose sizes are all $512\times 512 \times 3$.}
	\label{fig-original-images}
\end{figure}

\begin{figure}[t]
	\centering
	\includegraphics[width=9cm]{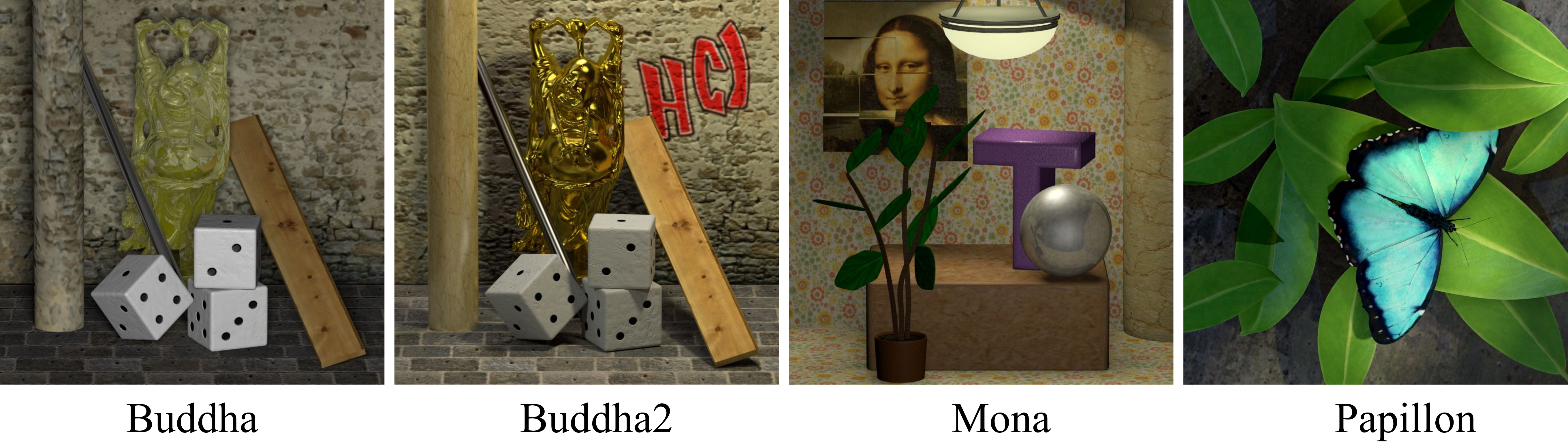}
	\caption{Center view of four light field images.}
	\label{fig-lightfield-images}
\end{figure}

\subsection{Application to Synthetic Tensor Completion} 
\label{sec-synthetic}


\begin{figure*}[t]
	\centering 
	\subfigure[]{
	\includegraphics[width=\linewidth]{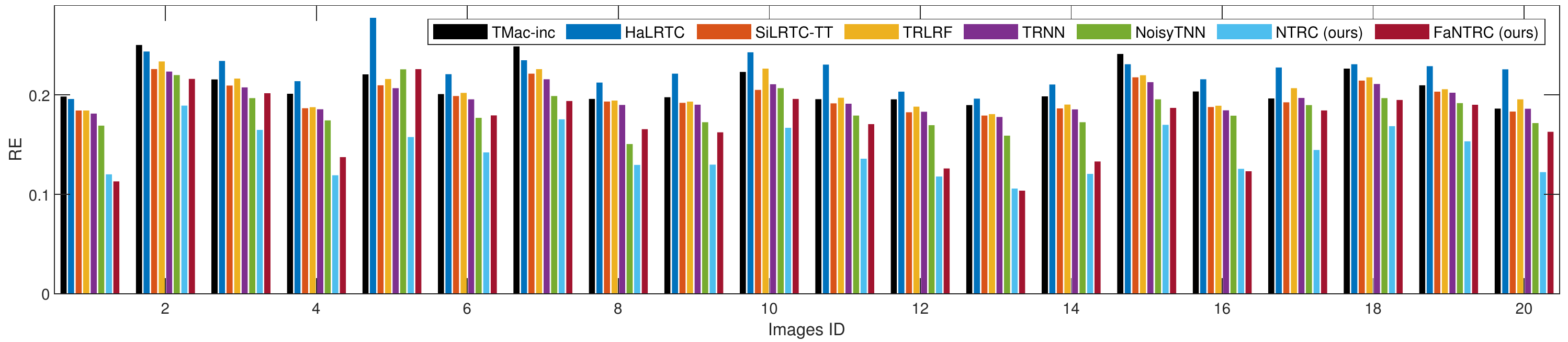}
}

	\vspace{-0.5em}
		\subfigure[]{
		\includegraphics[width=\linewidth]{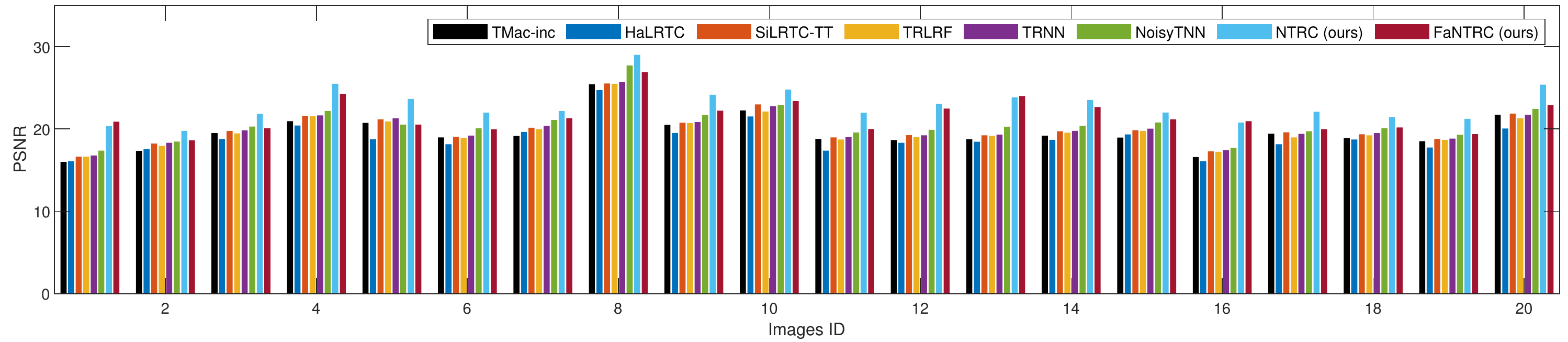}
	}

	\vspace{-0.5em}
	\subfigure[]{
	\includegraphics[width=\linewidth]{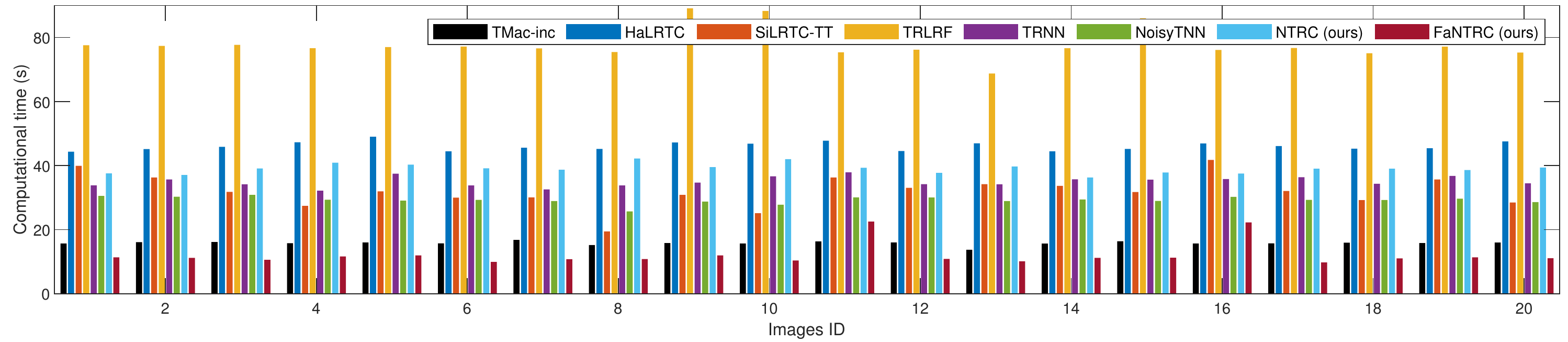}
}
\caption{Comparison of the RE, PSNR and the algorithm running time (seconds) on the benchmark images. (a) Comparison of the PSNR values; (b) comparison of the RE values; (c) comparison of the algorithm running time ( seconds).}
\label{fig-20images-comparison}
\end{figure*}

We conduct two experiments on synthetic data to evaluate the recovery performance of the proposed NTRC and FaNTRC  for recovering incomplete and noisy low-rank tensors. 
We generate a fourth-order tensor $ \tT \in \mathbb{R}^{d_1\times d_2\times d_3\times d_4}$ with TR rank $[r_1, r_2, r_3, r_4]$ as follows. 
First, we  produce four TR cores $\tensor{G}^{(k)} \in \mathbb{R}^{r_{k}\times d_k\times r_{k+1} }, k\in [4]$, using Matlab command $\tensor{G}^{(k)}=$\texttt{rand}$(r_k, d_k, r_{k+1})$. For simplicity, we let $d_k=d$ and  $r_k=r, k\in [4]$, in all experiments below. Then we construct a synthetic tensor by using tensor circular product $\tT^* = \text{TR}(\tensor{G}^{(1)}, \tensor{G}^{(2)}, \tensor{G}^{(3)}, \tensor{G}^{(4)} )$.
In order to meet noisy condition, we produce an additive noise tensor from $\mathcal{N}(0,\sigma^2)$  distribution, where we set the standard deviation $\sigma=c \| \tT^*\|_F /\sqrt{D}$ to keep a constant signal noise ratio. 
Finally, we sample $N=pd^4$ entries uniformly to form the partially observed noisy tensor, where $p$ denotes the SR.

\subsubsection{Sharpness of the Proposed Upper Bound} \label{sec-shapness-of-upperbound} In this experiment, we set the tensor size $d_k =d, d\in \{10,20,30\}, k \in [4]$ and TR rank $r_k=\lceil \log^{1/2} d_k \rceil, k  \in [4]$ following by \cite{wang2017near}. For the parameter $\lambda$, we let $\lambda_0=\sigma \sqrt{N \log(\tilde{d}_{k^*}) / \hat{d}_{k^*} }$ and chose the optimal one by scaling it $\lambda=a \lambda_0$, where $a$ is  experimentally selected in the candidate set $\{10^{-3}, 10^{-2}, 10^{-1}, 1, 10, 10^2, 10^3\}$.  To fix the signal-to-noise ratio (SNR) regardless of tensor size \cite{Negahban2012}, we normalize $\tT^*$ to have unit Frobenius norm, and set the noise level $c=0.01$. We repeat 25 trials using the proposed NTRC algorithm and compute the mean of estimation error $\| \tT^* - \hat{\tT} \|_F^2$ over all the trials.  
Fig. \ref{fig-synthetic-reN} (a) depicts the estimator error obtained by NTRC algorithm versus the observation numbers, 
The curves of Fig. \ref{fig-synthetic-reN} (a) decreases   as the observation number $N$ increases with . Additionally, synthetic tensor with larger  size required a larger number of observations.  According to the constant SNR settings, we substitute $\sigma=0.01\| \tensor{T}^* \|_F/\sqrt{D}$ to the  the right hand-side of (\ref{eq-main-result}), and obtain 
\begin{equation}
	\| \hat{\tT} - \tT^* \|_F^2 \leq c (10^{-4} \vee  D \delta^2)  \frac{r^2 K  d^{ \lceil \frac{K}{2} \rceil } \log(d^{ \lfloor \frac{K}{2} \rfloor}  +  d^{\lceil \frac{K}{2} \rceil  } )  }{N}
\end{equation}
holds with high probability. Following by \cite{Negahban2012}, the proposed upper bound is sharp if the estimation error should scale like 
\begin{equation} \label{eq-sharpness-bound}
	\| \hat{\tT} - \tT^* \|_F^2 = \mathcal{O} (\frac{r^2 K  d^{ \lceil \frac{K}{2} \rceil } \log(d^{ \lfloor \frac{K}{2} \rfloor}  +  d^{\lceil \frac{K}{2} \rceil  } )  }{N}). 
\end{equation}
By setting the rescaled observation number $N_0$ defined by 
\begin{equation}
	N_0 := \frac{N}{r^2 K  d^{ \lceil \frac{K}{2} \rceil } \log(d^{ \lfloor \frac{K}{2} \rfloor}  +  d^{\lceil \frac{K}{2} \rceil  } )  }, 
\end{equation}
the curves should be relatively aligned regardless of tensor sizes. We re-plot the estimation error versus the rescaled observation number in Fig. \ref{fig-synthetic-reN} (b), and observe that all the three curves are well-aligned, which demonstrates the sharpness of the proposed upper bound.

	\subsubsection{Effects of Varying TR ranks} In this experiment, we investigate the influence of TR rank on the estimation error as well as the noise distribution on RE. The experiments are conducted on synthetic tensor with size $d_k=d, d \in \{10,20\}, k \in [4]$, TR rank $r_k=r,r=\{1,2,3,4,5,6\}, k  \in [4]$. We set noise level $c=0.01$ and $\text{SR}=40\%$. We repeat 25 trials using the proposed NTRC algorithm and compute the mean of estimation error over all the trials. 
	We plot the estimation error versus the square of TR rank in Fig. \ref{fig-synthetic-reR}. It can be seen that the estimation error scales linearly against the square of TR rank, which agrees with the result in (\ref{eq-sharpness-bound}).

\subsubsection{Effects of Different Sub-exponential Noises} In this experiment, we investigate the effectiveness of the proposed estimator for dealing with different sub-exponential noises. We generate the synthetic tensor of size $d_k = d, d \in \{10,20\}, k \in [4]$ and TR rank $r_k =\lceil \log^{1/2} d_k \rceil, k  \in [4]$. We generate three types of noises include Gaussian noise, sub-Gaussian noise, and Poisson noise. For Gaussian distribution noise, we  produce the noise tensor from $\mathcal{N}(0,\sigma^2)$, where $\sigma = 0.01 \| \tT^*\|_F / \sqrt{D}$. For sub-Gaussian distribution noise, we generate the noise using Uniform distribution $\mathcal{U}[-0.5,0.5]$. For Poisson distribution noise, we generate the noise tensor from $\text{Pois}(0.01)$.  The incomplete tensors are  generated by uniformly selecting entries at random with SR from 10\% to 90\%. We repeat 25 trials using NTRC algorithm and compute the mean of RE. Fig. \ref{fig-synthetic-Noise} shows the RE versus varying SR. 
We can see that the RE on different noise distributions on two distinct tensor sizes is relatively low  and decreases as the SR increases, which indicates the effectiveness of the proposed estimator for recovering incomplete tensors with different sub-exponential noises.

\subsubsection{Effects of Multiple TR States} To verify the effectiveness of the proposed methods in multiple TR states, we investigate the RE of different methods versus varying TR ranks. We consider the synthetic tensor of size $\mathbb{R}^{30\times 30\times 30\times 30}$, TR rank $r_k = r, r\in \{ 2,3,4,5,6,7 \}, k \in [4]$. The zero-mean Gaussian noise is generated with noise levels  $c=0.01$ and  $c=0.05$, and the SR is set to $40\%$. We repeat 25 trials and compute the mean of RE.  Fig. \ref{fig-supercritical-results}  depicts the RE of different methods on different  noise levels versus varying TR rank $r$. The proposed NTRC achieves the lowest RE in all the cases, especially when the noise level is large. When $r\leq 5$, HaLRTC can well approximate the incomplete tensor. However, when $r\geq 6$, the synthetic tensor is supercritical ($r^2>d$), i.e., full-Tucker-rank,  the RE of HaLRTC grows dramatically as $r$ increases, which verifies the results revealed in Lemma \ref{lemma-multi-state}.

\subsubsection{Effects of the Given Rank for FaNTRC} We investigate the effects of the given rank for FaNTRC. We consider the synthetic tensor of size $d_k = d, d \in \{ 30, 50\}, k\in [4]$, TR rank $r_k = r, r\in \{ 3,4,5,6 \}, k \in [4]$. We produce  the noise tensor using zero-mean Gaussian distribution with noise level $c=0.01$ as in the above experiment, and sample the incomplete tensor using uniform distribution with $\text{SR}=40\%$. 
Fig. \ref{fig-synthetic-data-results}  presents the RE of FaNTRC on different  low-rank synthetic tensor versus the given rank $R_k$, where we set $R_k=R, k\in [4]$, and $R=\mcode{round}(\hat{c}r^2), \hat{c} \in \{ 0.7,0.8,\cdots, 1.5\}$. We can observe that when $R<r^2 $, the RE value is relatively large; and as long as $R\geq r^2 $,  the RE of FaNTRC tends to be stable with small RE value, which verifies the correctness of Theorem \ref{theorem-rank-equality}. 

\subsection{Application to Color Image Inpainting}
\label{sec-color-images-experiments}
\begin{figure*}[t]
	\centering
	\centering
	\subfigure[]{
		\centering
		\begin{minipage}[t]{1\textwidth}
			\centering
			\includegraphics[width=17cm]{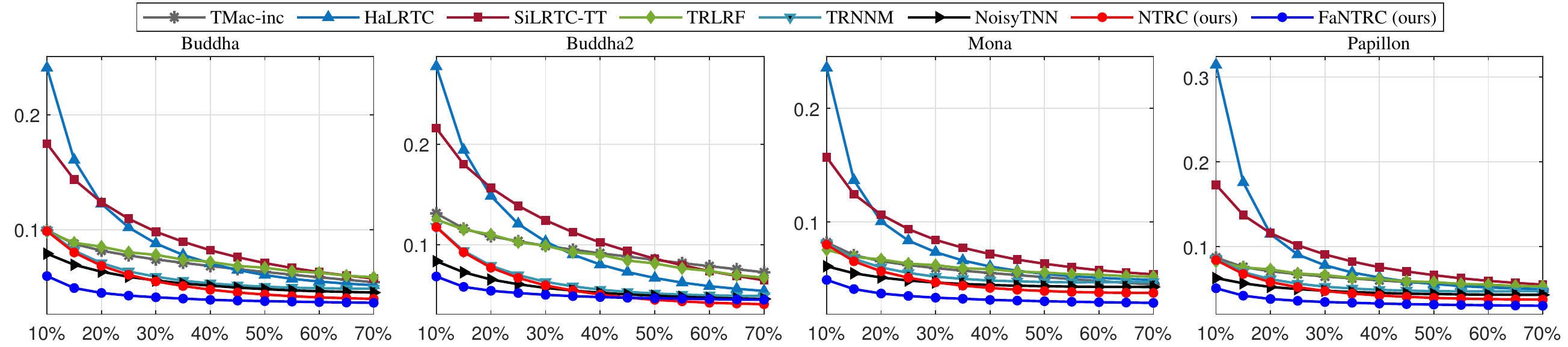}
		\end{minipage}
	}
	\subfigure[]{
		\centering
		\begin{minipage}[t]{1\textwidth}
			\centering
			\includegraphics[width=17cm]{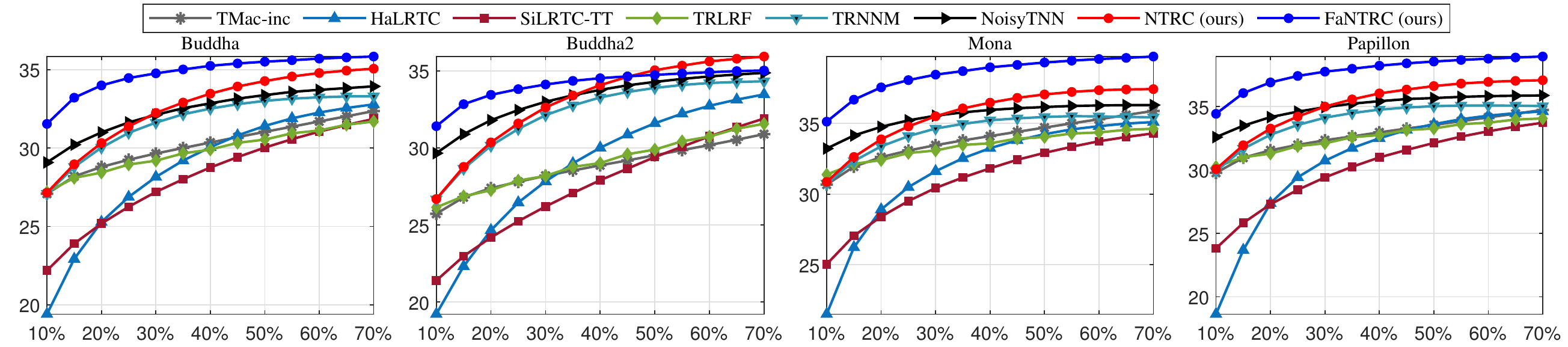}
		\end{minipage}
	}
	\caption{Comparison of performance (RE and PSNR) for all benchmark light field images, and for various SRs (from 10\% to 70\%) obtained utilizing the proposed methods and  state-of-the-art algorithms. (a) Comparison of the PSNR values; (b) comparison of the RE values. }
	\label{fig-lightField-sr}
\end{figure*}

In this part, we evaluate the proposed NTRC and FaNTRC against  state-of-the-art LRTC methods on 20 benchmark  images\footnote{http://decsai.ugr.es/cvg/dbimagenes/c512.php}  which are shown in Fig. \ref{fig-original-images}. Each image can be treated as a third-order tensor of ${512 \times 512 \times 3}$ entries. For each image, $40\%$ of pixels are sampled uniformly at random together with additive Gaussian noise with standard deviation  $\sigma=0.25\|\tT\|_F/\sqrt{D}$.  {  For TMac-inc, we set $\alpha_k = {1}/{3}$,  $r_n^{max}=60$ and maximum iteration  to $500$}. For HaLRTC, SiLRTC-TT and TRNNM, we employ their default parameter settings which lead to good performance. For TRLRF, we set the rank to $[5,5,5]$ as suggested in \cite{Yuan2019a}.  Note that for SiLRTC-TT, TRNNM and the proposed methods, we convert the images to fifth-order tensors of size ${16\times 16\times 32 \times 32 \times 3}$ by using  visual data tensorization (VDT) technique \cite{Bengua2017a,Yuan2019}. This simple and efficient trick has been shown to boost the tensor completion performance for TT-rank and TR-rank based methods \cite{Bengua2017a,Yuan2019}. For the proposed NTRC and FaNTRC, we simply set the weights $\alpha_{k} = 1/K, k\in [K]$,  and the given rank of FaNTRC is set equal to $[10,10,18,18,3]$ empirically. { For NoisyTNN, NTRC and FaNTRC, the penalty parameter $\lambda$ is selected as in Section \ref{sec-shapness-of-upperbound}.}

\begin{figure*}[t]
	\centering
	\includegraphics[width=18cm]{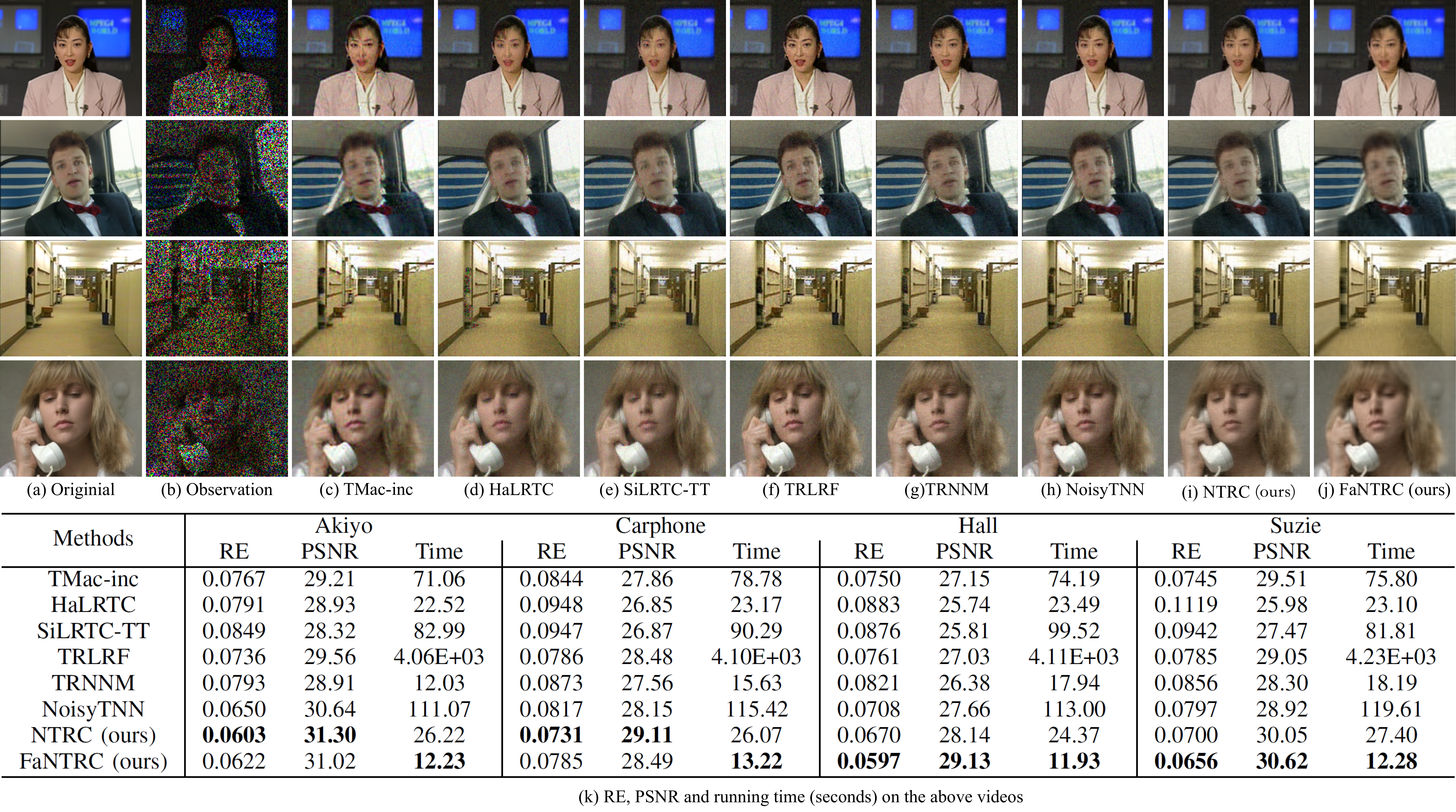}
	\caption{Performance comparison for video inpainting on four video sequences. From top to down are. (a) Original frame; (b) observed frame; (c)-(i) inpainting result obtained by HaLRTC, SiLRTC-TT, TRLRF, TRNNM, NoisyTNN, NTRC, FaNTRC, respectively; (j) comparison of RE, PSNR and running time on the four videos.}
	\label{fig-visual-video}
\end{figure*}

Fig. \ref{fig-20images-comparison} depicts the results of  applying the compared methods to the noisy and incomplete benchmark images. 
From Fig. \ref{fig-20images-comparison} (a) and  \ref{fig-20images-comparison} (b), it can be observed that the proposed NTRC and FaNTRC perform the first and second best in terms of RE and PSNR metrics in most cases, verifying their superiority in completing the real-world colorful images.  Compared with NoisyTNN, our methods can estimate more complex low-rank structure residing along each order of images. Thus, our methods achieve more promising recovery results. Though both SiLRTC-TT and TRNNM utilize the fifth-order VDT images, TRNNM performs a little better than SiLRTC-TT. 
{  TMac-inc is a Tucker rank-based noisy tensor completion method, and achieves better completion performance than HaLRTC in most cases. }
We also report the running time of each image in Fig. \ref{fig-20images-comparison} (c). The proposed FaNTRC is the fastest method, at least three times faster than the existing counterparts, which shows its scalability on such high spatial resolution images.  We also observed that our methods are much faster than TRLRF, since it requires time-consuming TR reconstruction operation at each iteration.  Conversely, our methods merely approximate the low-TR-rank tensor on its unfolding matrices and avoid frequently estimating TR core tensors explicitly.

\subsection{Application to Light Field Image Inpainting}
Different from the conventional image that only records light intensity,  light field image contains both light intensity and light direction in space, giving a high-dimensional representation of visual data for various computer vision applications. In this experiment,  we choose four light field images from the HCI light field image data  set\footnote{http://lightfieldgroup.iwr.uni-heidelberg.de}. Each light field image can be represented as a fifth order tensor of size $768\times 768 \times 3 \times 9 \times 9$, where $768\times 768$ and $9\times 9$ denote spatial and angular resolution, respectively.  For the sake of computational efficiency, we downsample the spatial dimension to $192\times 192$, and convert each image to the grayscale. Thus, we can form a fourth-order tensor of size $192\times 192 \times 9 \times 9$. 
The input incomplete tensors are generated by uniformly selecting pixels at random with SRs from $10\%$ to $70\%$ with increment $5\%$, and the additive noises are produced by i.i.d. Gaussian distribution with standard deviation $\sigma = 0.05 \| \tT \|_F/\sqrt{D}$. { For TMac-inc, we let $\alpha_k = {1}/{4}$ and $r_n^{max}=40$ empirically.} For HaLRTC, $\lambda$ is set to equal $[2,2,10,10]$, since it achieves the best performance in most cases.  For TRLRF, we let $r_1=r_2=r$ and $r_3=r_4=5$, and select the optimal $r$ in the set $\{11,12,\cdots, 17\}$ to achieve the best performance.  For FaNTRC, we simply let $R_1=R_2=110$ and $R_3=R_4=5$ in all cases of these experiments.  
To use NoisyTNN, we concatenate the angular dimensions by reshaping the light field images into the third-order tensors of size $192\times 192\times 81$. The remaining parameters are  set to the same as the above section.

 \begin{table}[t]
 	\centering
 	\caption{Summary of Average Running Time (seconds) of All SRs for Four Light Field Images (\textbf{Best}).}
 	\label{table-light-field-time}
 	\begin{tabular}{l|l|l|l|l|l} \hline \hline
 		Methods   & Buddha & Buddha2 & Mona   & Papillon & Avg.   \\  \hline 
 		TMac-inc    & {58.43}  & {58.92}   & {58.55}  & {58.81}    & {58.68}  \\
 		HaLRTC    & \textbf{15.33}  & \textbf{16.68}   & \textbf{14.62}  & \textbf{16.38}    & \textbf{15.75}  \\
 		SiLRTC-TT & 56.37  & 60.73   & 54.96  & 54.48    & 56.64  \\
 		TRLRF     & 817.30  & 893.25  & 774.59 & 868.93   & 838.52 \\
 		TRNNM     & 242.83 & 247.72  & 236.56 & 236.20    & 240.83 \\
 		NoisyTNN  & 92.69  & 92.20    & 90.53  & 93.69    & 92.28  \\
 		NTRC (ours)     & 196.36 & 193.93  & 198.38 & 191.75   & 195.11 \\
 		FaNTRC (ours)   &  {28.97}  &  {29.15}   &  {28.25}  &  {28.62}    &  {28.75}  \\ \hline \hline
 	\end{tabular}
 \end{table} 
 
Fig. \ref{fig-lightField-sr} (a) and \ref{fig-lightField-sr} (b) show the recovery performance with various SRs in terms of RE and PSNR metrics, respectively. The proposed FaNTRC and NTRC achieve the first and second best recovery performance  with incomplete and noisy observations compared with state-of-the-art tensor completion methods on four light field images. Both TRLRF and NoisyTNN behave well in some cases of low SR,  however their performance tends to improve very slowly as SR increases. TRNNM also achieves remarkable performance in this experiment, it performs only a little lower than NoisyTNN, and consistently outperforms SiLRTC-TT and HaLRTC, demonstrating the  advantages of low-TR-rank model in representing the high-order light field tensor data. We also report the average running time of all SRs for four light field images in Tabel \ref{table-light-field-time}. Despite the fact that FaNTRC only achieves the second best performance in terms of running time, it is still much faster than most of the other methods, especially TR-based methods, e.g., TRLRF and TRNNM. 

\subsection{Application to Video Inpainting}
Finally, we evaluate our methods on publicly available YUV color video sequences\footnote{http://trace.eas.asu.edu/yuv/}, which have been widely used for tensor completion applications. In this part, we test our methods and the other compared methods on four of these sequences, namely, Akiyo, Carphone, Hall, and Suzie. For each video sequence, we adopt its first $50$ frames, and each frame contains  $144\times 176 \times 3$ pixels. Therefore, the evaluated video sequences can be treated as the fourth order tensors of size ${144\times 176 \times 3 \times 50}$. We randomly select $40\%$ of pixels and add i.i.d. Gaussian noise with standard deviation $\sigma=0.1 \| \tT\|_F / \sqrt{D}$. { For TMac-inc, we simply set  $\alpha_k = {1}/{4}$ and $r_n^{max}=40$.} For HaLRTC, we empirically let $\bm{\lambda}=[1,1,10^{-3},10^2]$. For TRLRF and FaNTRC, we set the rank to $[20,20,20,20]$ and $[40,40,3,20]$, respectively. To use NoisyTNN, we reshape the video sequences to $144\times 176\times 150$ tensors.
All the other parameters are consistent with the above section.

\begin{figure}[t]
	\centering
	\subfigure[]{
		\begin{minipage}[t]{0.22\textwidth}
			\centering
			\includegraphics[width=4.4cm]{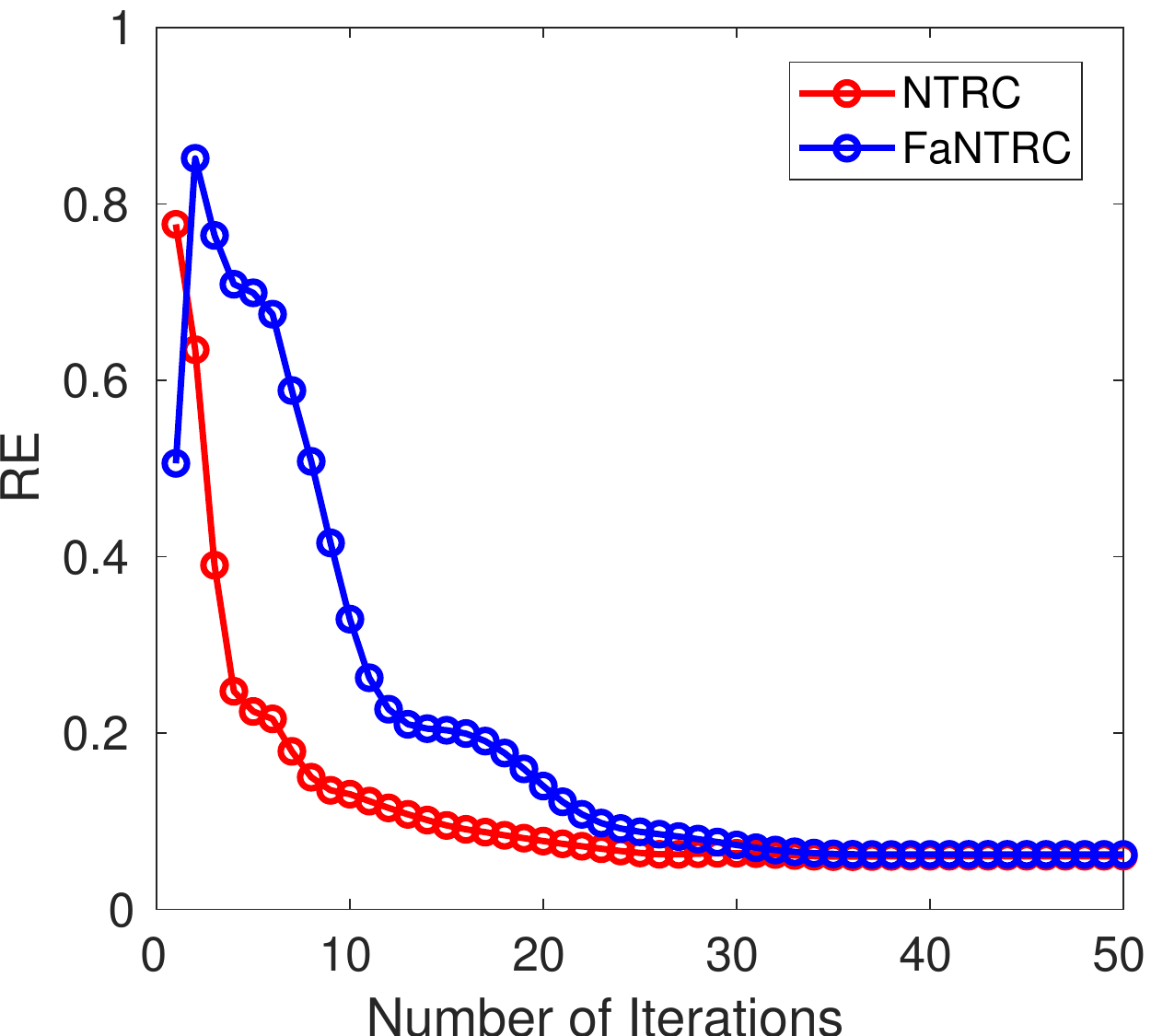}
		\end{minipage}
	}
	\subfigure[]{
		\begin{minipage}[t]{0.22\textwidth}
			\centering
			\includegraphics[width=4.4cm]{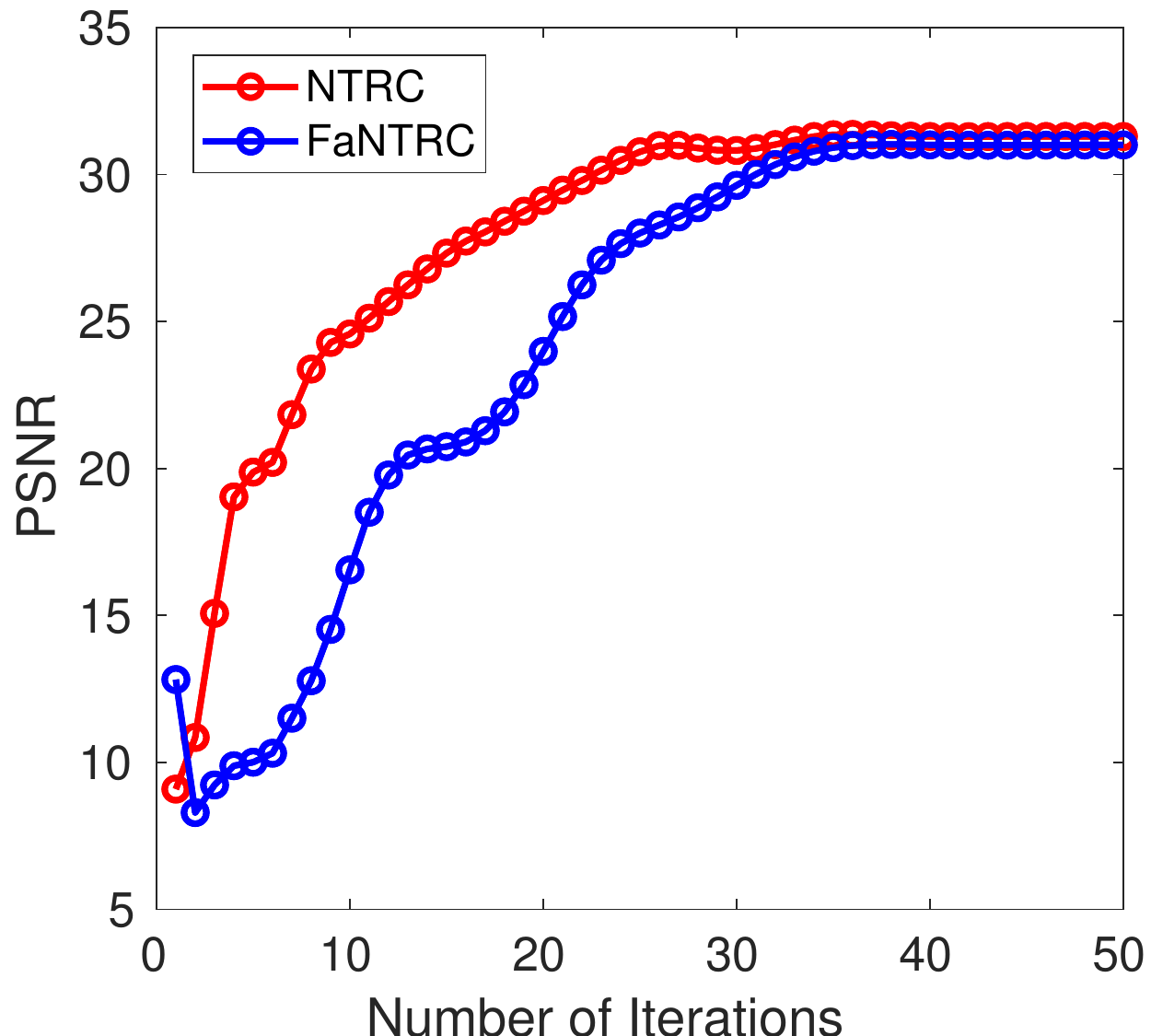}
		\end{minipage}
	}
	\caption{RE and PSNR history during iterations on Akiyo video sequence.}
	\label{fig-psnr-history}
\end{figure}

We present the inpainting result of the 21st frame in Fig. \ref{fig-visual-video} of four video sequences, respectively. It can be observed that both NTRC and FaNTRC obtain more favorable performance in terms of RE and PSNR metrics  than the other compared methods. More specifically, the proposed method can reconstruct more details of video frames.  From the recovered frames of Hall and Suzie,  it can be observed that the proposed methods can present more clean background than the other LRTC methods. 
Moreover, for Akiyo, Carphone, and Suzie, the additive noise on the clothes and face can be efficiently removed by the proposed methods. Analog to the experimental results in the above sections, TR based methods consistently outperform the Tucker and TT  methods. FaNTRC and TRNNM are the first and second fastest methods, respectively. Note that TRLRF is still the slowest method, where the reason has been analyzed above. 
We also present the RE and PSNR history over iterations of our methods on Akiyo video in Fig. \ref{fig-psnr-history}.  It can be observed that both NTRC and FaNTRC in terms of two metrics tend to be stable within merely 50 iterations. 

\section{Conclusions}
We proposed a novel noisy tensor completion model to predict the missing entries with noisy observations based on low-rank TR model. In order to illustrate the statistical performance of the proposed model, we theoretically analyze a non-asymptotic upper bound  of the estimation error. Moreover, we propose two algorithms  to solve the optimization problem with convergence guarantee, namely NTRC and its fast version FaNTRC. The FaNTRC can significantly reduce the high-computational cost of SVD on the large-scale circular unfolding matrix by equivalently minimizing TRNN on a much smaller one in a heterogeneous tensor decomposition framework, and thus accelerates the algorithm. Extensive experiments on both synthetic and visual data evidence the merits of low-TR-rank model for modeling a tensor, the ability of noisy completion setting in recovering incomplete data with noise corruption, and the computational efficiency of heterogeneous tensor decomposition strategy for TRNN minimization.

\section*{Acknowledgment}
The authors would like to thank  Dr. Andong Wang for his helpful discussions. 



\bibliographystyle{IEEEbib}
\bibliography{library,refs_trnnm}

\onecolumn
\pagenumbering{arabic}
\cleardoublepage
    \begin{titlepage}
	\centering
	\vspace*{5cm}	
	\LARGE\bfseries
Supplemental Material for  \\
\vspace*{0.5cm}
	\LARGE \bfseries
	Noisy Tensor Completion via Low-Rank Tensor Ring
	\vspace*{2cm}
	
	{\large Yuning Qiu, Guoxu Zhou, Qibin Zhao, and Shengli Xie}
	
	\vspace*{\fill}
\end{titlepage}

  \begin{appendices}  
	
	\section{More Experimental Results}
	\subsection{Experiments on Synthetic Data}
	\subsubsection{Synthetic experiments on different tensor sizes and noise levels}
	We investigate  RE and running time of different methods on fourth-order tensors with fixing SR $p=20\%$ and varying tensor size, TR rank, and noise level. Due to the computational efficiency, all the experiments are conducted twice, and their mean values of RE and running time are reported in Table  \ref{table-re-running-time}. The given rank of FaNTRC is set to $R_k=R,  k\in [K]$ and $R=\mcode{round}(1.2r^2)$. The remaining parameters are followed in Section \ref{sec-synthetic}. {  It can be observed that both NTRC and FaNTRC obtain lower RE in all experimental settings. 	HaLRTC provides the lowest running time, while the proposed methods are comparable. However, the RE of HaLRTC is significantly higher than the other methods in most cases. 
		In addition, FaNTRC is significantly faster than the TR-based methods, especially when TR rank is small, which is consistent with the analysis in Section \ref{sec-computational-complexity}. }However, when $r^2$ increases to close to $d$, the computational efficiency of FaNTRC degrades to that of TRNNM and NTRC. 
	\begin{table*}[h]
		\centering
		\caption{ Comparison of RE and Running Time (seconds)  for Synthetic Noisy Tensor Completion (\textbf{Best}).}
		\label{table-re-running-time}
		\begin{tabular}{ccccccccccccc} \hline \hline 
			\multirow{2}{*}{tensor size}                                 & \multirow{2}{*}{c}    & \multirow{2}{*}{r} & \multicolumn{2}{c}{HaLRTC}  & \multicolumn{2}{c}{TRNNM} & \multicolumn{2}{c}{SiLRTC-TT} & \multicolumn{2}{c}{NTRC (ours)}              & \multicolumn{2}{c}{FaNTRC (ours)}            \\
			&                       &                    & RE          & Time          & RE          & Time        & RE        & Time              & RE                & Time              & RE                & Time              \\ \hline 
			\multirow{6}{*}{$40\times 40\times 40\times 40$}   & \multirow{3}{*}{0.01} & 2    &9.09E-3  &2.00E1    & 5.08E-3    & 7.69E2    & 5.36E-3  & 4.14E2          & \textbf{2.01E-3} & 2.99E2          & 2.03E-3          & \textbf{1.70E1} \\
			&                       & 4                  &2.19E-2  & \textbf{1.73E1}     & 5.46E-3    & 7.71E2    & 7.71E-3  & 4.23E2          & \textbf{2.44E-3} & 5.65E2          & 2.48E-3          & {7.20E1} \\
			&                       & 6                  & 2.30E-2  &\textbf{1.39E1}      & 5.73E-3    & 7.86E2    & 8.47E-3  & {4.35E2} & \textbf{3.25E-3} & 5.22E2          & 3.42E-3          & 8.20E2          \\
			& \multirow{3}{*}{0.05} & 2         & 4.32E-2 & \textbf{1.54E1}    & 2.54E-2    & 8.75E2    & 6.17E-2  & 4.47E2          & \textbf{2.00E-2} & 4.14E2          & 2.28E-3          & {3.20E1} \\
			&                       & 4                  & 4.99E-2  & \textbf{1.23E1}    & 2.63E-2    & 9.00E2    & 2.70E-2  & 4.33E2          & 2.03E-2          & 4.39E2          & \textbf{1.14E-2} & {6.31E1} \\
			&                       & 6                  &4.16E-2   & \textbf{1.12E1}     & 2.66E-2    & 8.89E2    & 2.30E-2  & 4.42E2          & 2.00E-2          & {4.20E2} & \textbf{1.98E-2} & 8.77E2          \\ \hline 
			\multirow{6}{*}{$60\times 60\times 60 \times 60$}  & \multirow{3}{*}{0.01} & 3     &1.07E-2 & \textbf{7.79E1}    & 5.09E-3    & 7.47E3    & 8.26E-3  & 3.40E3          & 1.31E-3          & 7.44E3          & \textbf{1.25E-3} & {8.31E1} \\
			&                       & 5                  & 1.55E-2   &\textbf{6.98E1}      & 5.30E-3    & 7.77E3    & 2.89E-2  & 3.42E3          & 2.72E-3          & 8.13E3          & \textbf{2.62E-3} & {4.52E2} \\
			&                       & 7                  & 1.56E-2   &\textbf{6.51E1}      & 5.38E-3    & 6.67E3    & 1.88E-2  & {3.32E3} & \textbf{3.46E-3} & 8.34E3          & 4.59E-3          & 1.07E4          \\
			& \multirow{3}{*}{0.05} & 3         & 4.07E-2  &\textbf{6.25E1}      & 2.52E-2    & 7.18E3    & 2.68E-2  & 3.39E2          & 1.07E-2          & 8.01E3          & \textbf{3.03E-3} & {9.50E1} \\
			&                       & 5                  & 4.11E-2  &\textbf{5.46E1}      & 2.56E-2    & 7.13E3    & 3.83E-2  & 3.55E3          & 1.62E-2          & 6.57E3          & \textbf{6.63E-3} & {3.90E2} \\
			&                       & 7                  & 3.53E-2  &\textbf{4.64E1}      & 2.56E-2    & 7.21E3    & 2.99E-2  & {3.43E3} & \textbf{1.40E-2} & 6.66E3          & \textbf{1.40E-2} & 1.06E4          \\ \hline 
			\multirow{6}{*}{$80\times 80 \times 80 \times 80$} & \multirow{3}{*}{0.01} & 4        &1.12E-2  &\textbf{2.34E2}      & 5.11E-3    & 2.81E4    & 3.92E-2  & 1.58E4          & 2.86E-3          & 1.70E4          & \textbf{2.93E-4} & {4.72E2} \\
			&                       & 6                  &1.23E-2   &\textbf{2.25E2}       & 5.24E-3    & 3.03E4    & 2.11E-2  & 1.55E4          & 1.67E-3          & 2.80E4          & \textbf{8.97E-4} & {3.07E3} \\
			&                       & 8                  & 1.25E-2     & \textbf{1.80E2}& 5.25E-3    & 2.99E4    & 1.43E-2  & {1.55E4} & 2.21E-3          & 3.06E4          & \textbf{2.18E-3} & 4.40E4          \\
			& \multirow{3}{*}{0.05} & 4        &3.81E-2      &\textbf{1.61E2}       & 2.52E-2    & 3.04E4    & 4.78E-2  & 1.51E5          & 1.97E-2          & 1.83E4          & \textbf{2.04E-4} & {3.33E2} \\
			&                       & 6                 &3.61E-2& \textbf{1.70E2}     & 2.52E-2    & 2.75E4    & 3.41E-2  & 1.54E4          & 2.00E-2          & 1.89E4          & \textbf{1.83E-2} & {2.17E3} \\
			&                       & 8                & 3.22E-2  &\textbf{1.34E2}    & 2.53E-2    & 2.99E4    & 2.90E-2  & {1.44E4} & \textbf{1.87E-2} & 2.05E4          & 1.89E-2          & 2.82E4         \\ \hline \hline
		\end{tabular}
	\end{table*}
	{ 
		\subsubsection{Synthetic experiments on different tensor orders}
		We also verify the effectiveness of the proposed NTRC and FaNTRC in recovering the high-order incomplete tensor with noise corruptions. We generate the low-rank tensor $\tT^*$ with different orders and sizes, namely, $20\times 20\times 20 $ (3D), $20\times 20\times 20 \times 20 $ (4D), $10\times 10\times 10 \times 10 \times 10$ (5D), $10\times 10\times 10 \times 10 \times 10 \times 10$ (6D) and $10\times 10\times 10 \times 10 \times 10 \times 10 \times 10$ (7D). The TR rank of synthetic tensors are given as $r_k=3, k\in [K]$ for 3D, 4D, 5D and 6D tensor, and as $r_k=2,k \in [K]$ for 7D tensors. The observed entries is sampled uniformly with SR $30\%$ and the additive noise tensor is produced by distribution $\mathcal{N}(0,\sigma^2)$, where $\sigma=0.01\| \tT^*\|_F^*/\sqrt{D}$.  The remaining parameters are set as the above paragraph. From Table \ref{table-different-order} we can see that the proposed methods can recover the noisy and incomplete tensor with the lowest RE in all the cases. Moreover, the proposed methods perform better than the compared methods especially when tensor order is relatively high, which indicates the efficiency of the proposed methods in dealing with high-order and noisy tensor. 
		\begin{table*}[h]
			\centering
			\caption{ The RE Comparison of Different Tensor Completion Methods to 3D, 4D, 5D, 6D and 7D Synthetic Tensors (\textbf{Best}). }
			\label{table-different-order}
			\begin{tabular}{c|c|c|c|c|c|c|c|c}  \hline \hline
				Tensor Order and Tensor Size & TMac-inc & HaLRTC & SiLRTC-TT & TRLRF   & TRNNM   & NoisyTNN & NTRC (ours)    & FaNTRC   (ours)    \\ \hline 
				3D ($20\times 20\times 20$)          & 0.0485   & 0.0531 & 0.0515   & 0.0228  & 0.0401  & 0.0461   & 0.0420   & \textbf{0.0134}  \\
				4D ($20\times 20\times 20 \times 20$)          & 0.0180    & 0.0284 & 0.0102   & 0.0057 & 0.0067 & -        & 0.0039 &\textbf{0.0021} \\
				5D ($10\times 10\times 10 \times 10 \times 10$)            & 0.0207   & 0.0552 & 0.0082   & 0.0058 & 0.0068  & -        &\textbf{0.0051}  & 0.0051 \\
				6D  ($10\times 10\times 10 \times 10 \times 10 \times 10$)          & 0.0174   & 0.1342 & 0.0065   & 0.0056 & 0.0059 & -        & \textbf{0.0016 }& 0.0025 \\
				7D  ($10\times 10\times 10 \times 10 \times 10 \times 10 \times 10$)         & {0.0132}   & 0.1308 &  0.0062    & 0.0055 & 0.0055 & -        & 0.0014 & \textbf{0.0012} \\ \hline \hline 
			\end{tabular} 
		\end{table*}
	}

	
	{ 
		\subsection{Application to Real-world Noisy Images Inpainting} In this part, we compared the proposed methods and state-of-the-art LRTC methods on real-world noisy tensor data. We chose the newly proposed real-world noisy images benchmark \cites{Suppxu2018realworld}. The authors record the real-world images of different natural scenes. To obtain the corresponding "ground truth" images, the authors capture the same scenes for many times, and average these images as the "ground truth" image. In our experiment, we use five of the cropped images produced by the authors\footnote{https://github.com/csjunxu/PolyU-Real-World-Noisy-Images-Dataset}. These images cane be viewed as a third-order tensor of size $512\times 512\times 3$. We generate the incomplete tensor by uniformly sample the pixels with $\text{SR}=40\%$. Similar with the experimental setting in Section \ref{sec-color-images-experiments}, we transpose the incomplete tensor into fifth-order tensor of size $16\times 16\times 32\times 32\times 3$ using VDT technique for SiLRTC-TT, TRNNM and the proposed methods.  
		The remaining parameters are set to the same as Section  \ref{sec-color-images-experiments}. Fig \ref{fig-real-world-noisy} depicts the visual quality as well as the PSNR and RE of different images. We can see that the proposed methods achieve the best completion performance in terms of PSNR and RE. Additionally, from the recovered images, we can observe that the proposed methods can effectively eliminate most of the noise  while recovering the missing pixels. 
		\begin{figure}[t]
			\centering
			\includegraphics[width=\linewidth]{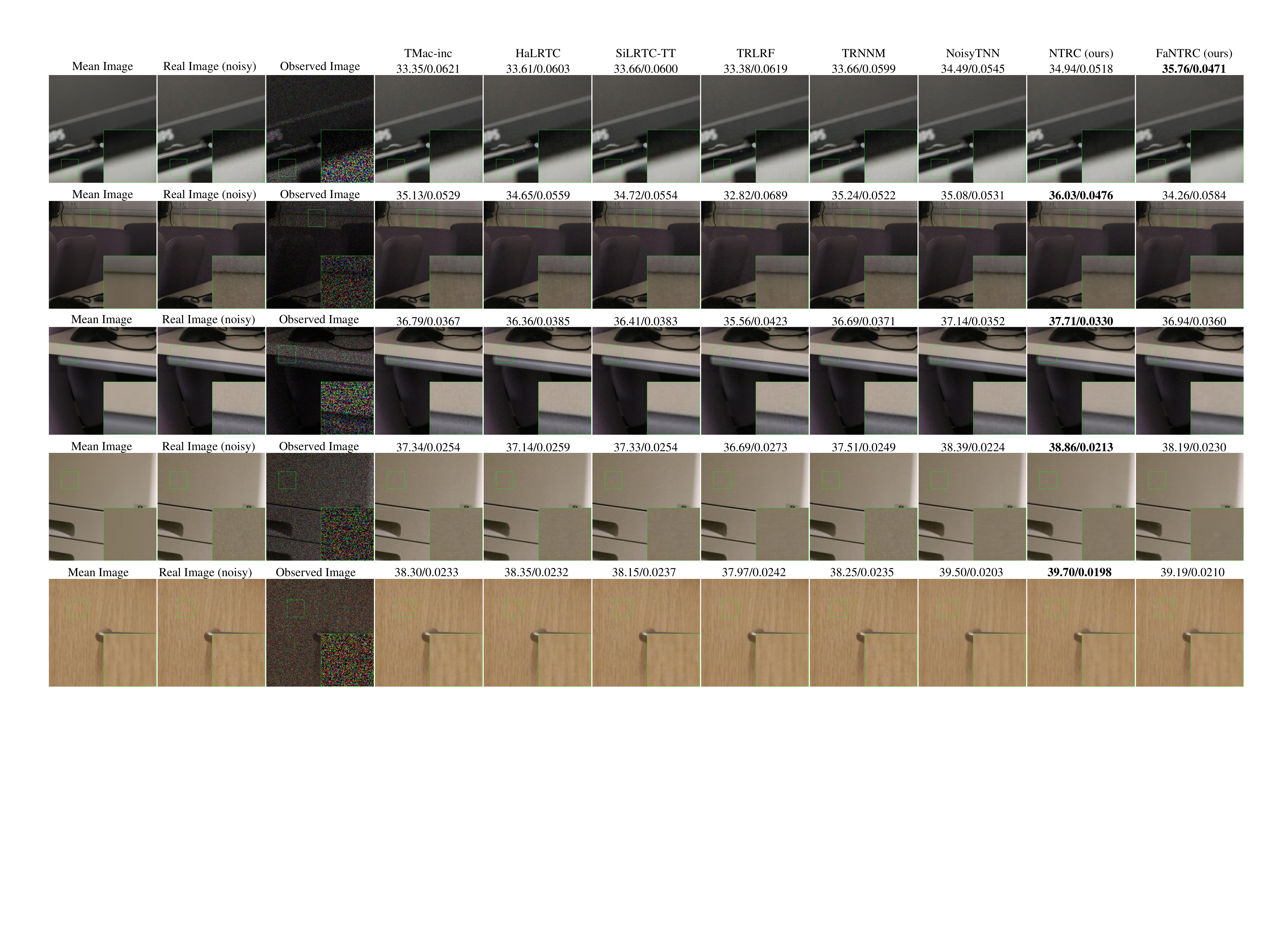}
			\caption{ Comparison of PSNR/RE results of the real-world noisy image by different methods (\textbf{Best}).}
			\label{fig-real-world-noisy}
		\end{figure}
	}
	
	{  
		\section{Proof of Lemma \ref{lemma-dual-norm}}
		\label{sec-proof-dual-norm}
		According to the definition of dual norm, we give the following maximization problem
		\begin{equation} 
			\trop{\tT} := \sup_{\tensor{W} } \left\langle \tT, \tensor{W} \right\rangle, \text{s.t. } \trnn{\tensor{W}} \leq 1. 
		\end{equation}
		Since the maximization problem satisfies the Slater's condition, the strong duality holds. We only need to prove that its dual problem satisfies 
		\begin{equation} \label{eq-maxization-problem}
			\inf_{\sum_{k=1}^{K} \tensor{Y}^k={\tT}} \max_{ k} \alpha_k^{-1} \| \mat{Y}^{k}_{(k,s)} \|. 
		\end{equation}
		According to the Fenchel's duality theorem, we have 
		\begin{equation} \label{eq-dual-fenchel}
			\sup_{\tensor{W} } \left\langle \tT, \tensor{W} \right\rangle - \kappa( \trnn{\tensor{W}} \leq 1 ) = \inf_{ \{\tensor{Y}^k \}_{k=1}^{K} } \left(\kappa( \sum_{k=1}^{K} \tensor{Y}^k={\tT} ) + \max_{ k} \alpha_k^{-1} \| \mat{Y}^k_{(k,s)} \| \right), 
		\end{equation}
		where $\kappa(\cdot)$ denotes the indicator function:
		\begin{equation}
			\kappa(x) = \left\{ \begin{split}
				&0, \quad \quad x \text{ is true}, \\
				&+\infty, \text{ otherwise}.
			\end{split}\right.
		\end{equation}
		The right-hand side of (\ref{eq-dual-fenchel})  meets (\ref{eq-maxization-problem}), thus the proof of Lemma \ref{lemma-dual-norm} is completed.
	}

	\section{Proof of Theorem \ref{theorem-main-result}} \label{appendices-a}
	In this section, we give the proof for Theorem \ref{theorem-main-result} by first introducing the estimation error bound of observation entries in Lemma \ref{lemma-estimation-error}, then presenting the overall estimation error with specific setting of regularization parameter $\lambda$ in Lemma \ref{prop-fro-bound}, finally giving the upper bound of tensor ring spectral norm of two random variables in Lemmas \ref{lemma-bound-tr-op} and \ref{lemma-bound-tr-expectation}.
	\begin{lemma} \label{lemma-estimation-error}
		Let $\lambda \geq  2\sigma \tropL{\Poo{ \bm{\xi}} }$, we have:
		\begin{equation}
			\frac{1}{2}|| \Po{\errorT} ||_F^2 \leq   \frac{3}{2}\sqrt{2} \lambda \fro{\errorT} \sum_{k=1}^{K} \alpha_k\sqrt{r_k r_{k+s}},
		\end{equation}
		where $\Poo{\bm{\xi}} = \sum_{n=1}^{N} \xi_n \tensor{X}_n$, and $ \errorT$ denotes the residual error $\errorT = \hatT - \starT $.
	\end{lemma}
	Lemma \ref{lemma-estimation-error} denotes the estimation error of observations can be upper bound by weighted sum of square root of TR rank. Next, we present a general upper bound on the estimation error based on (\ref{eq-tr-noise-model}). 
	\begin{lemma} \label{prop-fro-bound}
		Suppose $\lambda \geq 2\sigma \tropL{ \Poo{\bm{\xi}} }$, then we have probability at least $1-{2}/{\tilde{d}^*}$, such that the estimation error satisfies
		\begin{equation}
			\begin{split}
				\fros{\errorT} \leq \max \left\{  c_2\delta^2 D\sqrt{ \frac{64 \log \tilde{d}_{k^*}}{\log(6/5) N}},  c_3\left( \frac{D}{N}  \sum_{k=1}^{K} \alpha_{k}\sqrt{r_k r_{k+s}} \right)^2 \left(\lambda^2 + \delta^2 \expectations{ \tropL{\Poo{\varepsilon}} }  \right) \right\}.
			\end{split}	
		\end{equation}
	\end{lemma}
	It can be observed that there are two stochastic variables needed to be bounded. We present two results in following lemmas.
	\begin{lemma} \label{lemma-bound-tr-op}
		Suppose $N\geq 2 \hat{d}_{k^*} \log ^3 (\tilde{d}_{k^*} )$, then with probability at least  $1-1/\tilde{d}_{k^*}$, we have
		\begin{equation}
			\tropL{\Poo{\bm{\xi}}}    \leq c_5 \sqrt{\frac{N \log (\tilde{d}_{k^*})}{\hat{d}_{k^*}}}.
		\end{equation}
	\end{lemma}
	
	\begin{lemma} \label{lemma-bound-tr-expectation}
		Suppose $N\geq 2 \hat{d}_{k^*} \log ^3 (\tilde{d}_{k^*} )$, then  we have
		\begin{equation}
			\expectation{ \tropL{ \Poo{\bm{\varepsilon} } }} \leq c_6 \sqrt{\frac{N \log (\tilde{d}_{k^*})}{\hat{d}_{k^*}}}.
		\end{equation}
	\end{lemma}
	Combining Lemmas \ref{lemma-estimation-error} to \ref{lemma-bound-tr-expectation},  we can directly obtain the result of Theorem \ref{theorem-main-result}. 
	
	Proof of Theorem \ref{theorem-main-result} is completed.

	\section{Proof of Lemmas \ref{lemma-estimation-error} and \ref{prop-fro-bound}}
	\subsection{Proof of Lemma \ref{lemma-estimation-error}}
	\begin{proof} [Proof of Lemma \ref{lemma-estimation-error}]
		From (\ref{eq-tr-noise-model}), we have 
		\begin{equation}
			\frac{1}{2} \fros{ \mat{y} - \mathfrak{X}(\hatT) } + \lambda \trnn{\hatT} \leq \frac{1}{2} \fros{\mat{y} - \mathfrak{X}(\starT) } + \lambda \trnn{\starT}, 
		\end{equation}
		where $\hatT$ and $\starT$ denote the flexible solution of the proposed model and the  true tensor, respectively. By performing some algebra, it can be easy to see that
		\begin{equation}
			\frac{1}{2} || \frakX{ \errorT } ||_2^2 + \langle \frakXt{\sigma \mat{\xi} } , \errorT\rangle  + \lambda \trnn{\hatT} \leq \lambda \trnn{\starT}.
		\end{equation}
		According to Lemma \ref{lemma-dual-norm}, we can obtain that
		\begin{equation} \label{eq-overall-algebra}
			\frac{1}{2}  || \frakX{ \errorT } ||_2^2  + \lambda \trnn{\hatT} \leq \lambda \trnn{\starT} + \trop{\frakXt{\sigma \mat{\xi} }} \trnn{ \errorT}.
		\end{equation}
		Let $\mathcal{P}_{ \mat{T}_k^* }(\cdot)$ and $\mathcal{P}_{\mat{T}_k^* }^{\perp} (\cdot)$ be the $k$th orthogonal and orthogonal complement projection of $\starT$, respectively, and is given by
		\begin{align}
			& \mathcal{P}_{\mat{T}_k^{*}} (\mat{A}) = (\mat{I}_{d_{1,k}} -\mat{U}_k  \trans{\mat{U}}_k ) \mat{A} (\mat{I}_{d_{2,k}} - \mat{V}_k \trans{\mat{V}}_k ),  \\
			& \mathcal{P}_{\mat{T}_k^{*}}^{\perp} (\mat{A}) = \mat{A} -  (\mat{I}_{d_{1,k}} -\mat{U}_k  \trans{\mat{U}}_k ) \mat{A} (\mat{I}_{d_{2,k}} - \mat{V}_k \trans{\mat{V}}_k ), 
		\end{align}
		where $\mat{U}_k$ and $\mat{V}_k$ are left and right singular vectors of $\mat{T}_{(k,s)}^*$. Then we have 
		\begin{align}
			\nonumber \trnn{ \hatT } & =\sum_{k=1}^{K} \alpha_k || \starTm + \mathcal{P}_{\mat{T}_k^*}(\errorM )  + \mathcal{P}_{\mat{T}_k^* }^{\perp} ( \errorM ) ||_* \\
			\nonumber  \geq \sum_{k=1}^{K } &\alpha_k (|| \starTm + \mathcal{P}_{\mat{T}_k^*}^{\perp} ( \errorM ) ||_* - || \mathcal{P}_{\mat{T}_k^*} (\errorM) ||_*)  \\
			\label{eq-split-Tstar}	 = \sum_{k=1}^{K} & \alpha_k (  \nn{\starTm} + \nn{\mathcal{P}_{\mat{T}_k^*}^{\perp} ( \errorM )} - \nn{\mathcal{P}_{\mat{T}_k^*} (\errorM)}).
		\end{align}
		Combining (\ref{eq-overall-algebra}) and (\ref{eq-split-Tstar}),  we obtain that 
		\begin{align}
			\frac{1}{2}  || \frakX{ \errorT } ||_2^2 &\leq  \lambda \sum_{k=1}^{K} \alpha_k (\nn{\mathcal{P}_{\mat{T}_k^*} (\errorM)} - \nn{\mathcal{P}_{\mat{T}_k^*}^{\perp} ( \errorM )} )  +\trop{\frakXt{\sigma \mat{\xi} }} \trnn{ \errorT}  \\
			\label{ineq-projection}= &\lambda \sum_{k=1}^{K} \alpha_k ( \frac{3}{2} \nn{\mathcal{P}_{\mat{T}_k^*} (\errorM)}   - \frac{1}{2} \nn{ \mathcal{P}_{\mat{T}_k^*}^{\perp} (\errorM)} )\\
			\leq &  \frac{3}{2} \lambda \sum_{k=1}^{K} \alpha_k \sqrt{2 r_kr_{k+s}}\fro{ \errorT} .
		\end{align}
		Proof of Lemma \ref{lemma-estimation-error} is completed.
	\end{proof}

	\subsection{Proof of Lemma \ref{prop-fro-bound}}
	We assume that the residual error $\errorT$ drops in a restricted set as follows: 
	\begin{equation}\label{eq-constrained-set}
		\nonumber \mathcal{C}(\bm{r}) = \left\{ \errorT\in \mathbb{R}^{d_1 \times d_2 \times \cdots \times d_K} : \trnn{\errorT} \leq \sum_{k=1}^{K} \alpha_{k} \sqrt{r_k r_{k-s}}\fro{\errorT},  || \errorT ||_{\infty} = 1, \fros{\errorT} \geq D \sqrt{\frac{64 \log \tilde{d}_{k^*} }{\log(6/5) N }} \right\},
	\end{equation}
	where $\bm{r} = [r_1, \cdots, r_{K}]$. Next, we show that any tensor $\errorT$ in the above set $\mathcal{C} (\bm{r})$, the following lemma holds.
	\begin{lemma}[Restricted Strong Convexity]  \label{lemma-rsc}
		Suppose $\errorT \in \mathcal{C}(\bm{r})$, then we have probability at least $1-{2}/{\tilde{d}_{k^*}}$, such that
		\begin{equation}
			\frac{1}{N}\fros{ \Po{ \errorT } } \geq \frac{1}{D}  \fros{\errorT} - 44  \frac{D }{N^2}  \left(\expectation{ \tropL{\Poo{\varepsilon}} }  \sum_{k=1}^{K} \alpha_{k}\sqrt{r_k r_{k+s}}\right)^2 .
		\end{equation} 
	\end{lemma}
	We present the proof of Lemma \ref{lemma-rsc} in Appendix \ref{sec:proof-rsc}, and give the proof of Lemma \ref{prop-fro-bound} as follows.
	\begin{proof} [Proof of Lemma \ref{prop-fro-bound}]
		From (\ref{ineq-projection}), we can observe that $\trnn{\mathcal{P}^{\perp}_{\mat{T}_k^*}(\errorT)} \leq 3 \trnn{\mathcal{P}_{\mat{T}_k^*}(\errorT)} $. Therefore, it is easy to see that 
		\begin{align}
			&\nonumber \trnn{\errorT}  =\trnn{ \mathcal{P}^{\perp}_{\mat{T}_k^*}(\errorT)} + \trnn{\mathcal{P}_{\mat{T}_k^*}(\errorT)} \leq 4 \trnn{\mathcal{P}_{\mat{T}_k^*}(\errorT)} \\
			& \leq 4 \sum_{k=1}^{K} \alpha_{k}\sqrt{2r_k r_{k-s}} \fro{\mathcal{P}_{\mat{T}_k^*}(\errorT)  }   \leq  \sum_{k=1}^{K}\alpha_{k}\sqrt{32r_k r_{k+s} } \fro{\errorT} .
		\end{align}
		According to assumption \textbf{A}\textbf{\ref{assumption-bound-infty}}, we can obtain that $||\errorT ||_{\infty} \leq || \starT ||_{\infty} + || \hatT ||_{\infty} < 2\delta$. In order to bound $\fros{\errorT}$,  we discuss whether the normalized residual error tensor ${(2\delta)^{-1}}\errorT$ is in the set $ \mathcal{C} \left( 4\sqrt{2} \bm{r} \right)$ in the following two cases. \\
		\emph{Case 1}: If $(2\delta)^{-1}\errorT$ is not in the constrained set $\mathcal{C} \left( 4\sqrt{2} \bm{r} \right)$, then we have 
		\begin{equation}
			{\fros{ \errorT }} \leq 4 \delta^2 D\sqrt{ \frac{64 \log \tilde{d}_{k^*} }{\log(6/5) N}}. 
		\end{equation}
		\emph{Case 2}: If $(2\delta)^{-1}\errorT$ is in the constrained set $\mathcal{C} \left( 4\sqrt{2} \bm{r} \right)$, according to Lemma \ref{lemma-estimation-error} and  Lemma \ref{lemma-rsc} we have
		\begin{equation}
			\begin{split}
				\fros{\errorT} \leq \frac{D}{N} 3\sqrt{2} \lambda \fro{\errorT} \sum_{k=1}^{K} \alpha_{k}\sqrt{r_k r_{k+s}} + \frac{44 D^2  }{N^2} \left( \expectation{ \trop{ \Poo{\varepsilon}} } \sum_{k=1}^{K} \alpha_{k}\sqrt{32 r_k r_{k+s} }\right)^2, 
			\end{split}
		\end{equation}
		with probability at least ${2}/\hat{d}_{k^*}$. Then by performing some algebra, we have 
		\begin{equation}
			\fros{\errorT} \leq  c_3\left( \frac{D}{N}  \sum_{k=1}^{K} \alpha_{k}\sqrt{r_k r_{k+s}} \right)^2 \left(\lambda^2 + \delta^2 \expectations{ \trop{\Poo{\varepsilon}} }  \right).
		\end{equation}
		Combining the above two cases, we can obtain Lemma \ref{prop-fro-bound} directly. Proof of Lemma \ref{prop-fro-bound} is completed.
	\end{proof}
	\section{Proof of Lemmas \ref{lemma-bound-tr-op} and \ref{lemma-bound-tr-expectation}}
	\subsection{Proof of Lemma \ref{lemma-bound-tr-op}}
	
	According to the definition of dual norm of tensor ring nuclear norm, we have
	\begin{equation}
		\trop{ \Poo{\bm{\xi}} } = \inf_{\tensor{Y}^1 + \cdots + \tensor{Y}^K = \Poo{\bm{\xi}} } \max_{k=1,\cdots, K} \frac{1}{\alpha_k} \| \mat{Y}_{(k,s)}^{k} \|,
	\end{equation}
	it can be easy to see that the decomposition of noise tensor $\Poo{\bm{\xi}} $ into the sum $\Poo{\bm{\xi}} =\sum_{k=1}^{K} \tensor{Y}^{k}$ is arbitrary.
	Similar to \cites{SuppTomioka2013a}, we can simply set  a  singleton decomposition of $\Poo{\bm{\xi}}$ on one mode as
	\begin{equation} \label{eq-spectral-singleton}
		\tensor{Y}^{k'}= \left\{ \begin{split}
			&\Poo{\bm{\xi}}   ~~ \text{if } k' = \arg\max_k \| \mat{Y}_{(k,s)}^k \|, \\
			&0.
		\end{split} \right.
	\end{equation}
	
	\begin{lemma} \label{lemma-non-commutative}
		For independent centered random tensor sequence $\tensor{Z}^1, \cdots, \tensor{Z}^N$ of size $d_1 \times d_2 \times \cdots \times d_K$, suppose their unfolding matrices $\mat{Z}_{(k,s)}^1, \mat{Z}_{(k,s)}^{2}, \cdots, \mat{Z}_{(k,s)}^N$ of size $d_{1,k} \times d_{2,k}$ satisfy 
		\begin{equation}
			\sigma_0^k \geq  \left\{ \| \frac{1}{N} \sum_{n=1}^{N} \expectation{\mat{Z}_{(k,s)}^n \mat{Z}_{(k,s)}^{n,\top}} \|^{1/2}, \| \frac{1}{N} \sum_{i=1}^{N}  \expectation{\mat{Z}_{(k,s)}^{n, \top} \mat{Z}_{(k,s)}^n  } \|^{1/2} \right\}, k \in [K], 
		\end{equation}
		and $\forall n\in [N], U_n = \inf \{ K_0>0: \expectation{ \exp (\| \mat{Z}_{(k,s)}^{n}\| / K_0) } \leq e \}$.  Let $U>U_n, \forall n\in [N]$, then there exist an absolute constant $c_4$, such that, for all $t>0$, with probability at least $1-e^{-t}$ we have
		\begin{equation}
			\| \frac{1}{N} \sum_{n=1}^{N} \mat{Z}_{(k,s)}^{n} \| \leq c_4 \max \{ \sigma_0^k \sqrt{\frac{t+\log(\tilde{d}_k)}{N}}, U\log( \frac{U}{\sigma_0^k} ) \frac{t+\log (\tilde{d}_k)}{N}\}.
		\end{equation}
	\end{lemma}
	Lemma \ref{lemma-non-commutative} is an extension of standard matrix version of Bernstein's inequality \cites{Suppkoltchinskii2013remark,SuppKlopp2014} defined on circular unfolding matrix.

	\begin{proof} [Proof of Lemma \ref{lemma-bound-tr-op}]
		We let $\tensor{Z}^n = \xi_n \tensor{X}_n$ and $U=K_{\xi}$.  Since  $\expectation{ \mat{Z}_{(k^*)}^{n}}=\mat{0}$,  we have
		\begin{equation}
			\| \frac{1}{N}\sum_{n=1}^{N}\expectation{ \mat{Z}_{(k^*)}^{n} {\mat{Z}}_{(k^*)}^{n\top} } \| =  \frac{1}{d_{1,k^*}},~~ \| \frac{1}{N}\sum_{n=1}^{N} \expectation{ {\mat{Z}}_{(k,s)}^{n \top} {\mat{Z}}_{(k,s)}^{n} }\| = \frac{1}{d_{2,k^*}},
		\end{equation}
		therefore we can choose $\sigma_0^k=1/\sqrt{ \hat{d}_{k^*} }$. According to Lemma \ref{lemma-non-commutative} and (\ref{eq-spectral-singleton}), it holds that
		\begin{equation} \label{eq-bound-xi}
			\begin{split} 
				& \trop{\Poo{\bm{\xi}}} =    \| \Poo{\bm{\xi}}_{(k',s)} \|  =\|\sum_{n=1}^{N} \mat{Z}_{(k',s) }^{n} \| \\
				&\leq  c_2  \max  \{ \sqrt{ \frac{N(t+\log( \tilde{d}_{k^*} ))}{\hat{d}_{k^*}} },  K_{\xi}  \log (K_{\xi} \sqrt{\hat{d}_{k^*} })  (t+ \log (\tilde{d}_{k^*} ) )\}.
			\end{split}
		\end{equation}
		By letting $t=\log (\tilde{d}_{k^*})$, with probability at least $1-1/\tilde{d}_{k^*}$, we have 
		\begin{equation}
			\trop{ \Poo{\bm{\xi}} }  \leq c_4 \max\{ \sqrt{\frac{2N \log(\tilde{d}_{k^*}) }{\hat{d}_{k^*} }}, 2K_{\xi} \log(K_{\xi} {\hat{d}_{k^*}}) \log (\tilde{d}_{k^*})  \} .
		\end{equation}
		Therefore, suppose $N \geq 2 \hat{d}_{k^*} K_{\xi}^2 \log ^2 ( 
		K_{\xi} \tilde{d}_{k^*} ) \log(\tilde{d}_{k^*})$, with probability at least $1-1/\tilde{d}_{k^*}$, we have 
		\begin{equation}
			\trop{\Poo{\bm{\xi}}}    \leq c_{5} \sqrt{\frac{N \log (\tilde{d}_{k^*})}{\hat{d}_{k^*}}},
		\end{equation}
		where $c_{5}$ depends on $K_{\xi}$.  
		Proof of Lemma \ref{lemma-bound-tr-op} is completed.
	\end{proof}
	
	{ 
		\subsection{Proof of Lemma \ref{lemma-bound-tr-expectation}}
		\begin{proof}
			The proof sketch is similar with that of Lemma 6 in \cites{SuppKlopp2014}. For the sake of completeness, we give the proof details here.
			Similar to (\ref{eq-bound-xi}), for any $t>0$, there exists a constant $\tilde{c}$ such that 
			\begin{equation} \label{eq-inequality-bound}
				\trop{\Poo{\bm{{\varepsilon} }}} \leq \tilde{c} \max  \{ \sqrt{ \frac{N(t+\log( \tilde{d}_{k^*} ))}{\hat{d}_{k^*}} },    \log ( {\hat{d}_{k^*} })  (t+ \log (\tilde{d}_{k^*} ) )\}, 
			\end{equation}
			holds with probability at least  $1-e^{-t}$.   By letting the equality of the right hand-side of (\ref{eq-inequality-bound}), we set $t^*={N}/({\hat{d}_{k^*} \log^2(\hat{d}_{k^*})}) - \log(\tilde{d}_{k^*})$, which implies that 
			\begin{equation} \label{eq-t-star-inqeuality-1}
				\mathbb{P}(  \trop{\Poo{\bm{\varepsilon}}}  > t ) \leq \tilde{d}_{k^*} \exp(   -t^2 \hat{d}_{k^*} /(\tilde{c}^2 N) ) \qquad t \leq t^*,
			\end{equation}
			and 
			\begin{equation} \label{eq-t-star-inqeuality-2}
				\mathbb{P}(  \trop{\Poo{\bm{\varepsilon}}}  > t ) \leq \tilde{d}_{k^*} \exp(   -t  /(\tilde{c}  \log(\hat{d}_{k^*} )) ) \qquad t \geq  t^*. 
			\end{equation}
			By letting $\varsigma_1=\hat{d}_{k^*} /(\tilde{c}^2 N)$ and $\varsigma_2=1/ (\tilde{c}  \log(\hat{d}_{k^*} )) $. According to H$\ddot{\text{o}}$lder's inequality, we have 
			\begin{equation}
				\expectation{  \trop{\Poo{\bm{\varepsilon}}}  } \leq \left( \expectation{  \trop{\Poo{\bm{\varepsilon}}}  } ^{2 \log( \tilde{d}_{k^*}) } \right)^{1/(2 \log(\tilde{d}_{k^*}))}. 
			\end{equation}
			According to (\ref{eq-t-star-inqeuality-1}) and (\ref{eq-t-star-inqeuality-2}), we have 
			\begin{equation} \label{eq-t-star-inqueality-com}
				\begin{split}
					&\left( \expectation{  \trop{\Poo{\bm{\varepsilon}}} }^{2\log( \tilde{d}_{k^*} )}   \right) {1/(2 \log(\tilde{d}_{k^*}))}  \\
					=&\left( \int_{0}^{+\infty}  \mathbb{P}\left(  \trop{\Poo{\bm{\varepsilon}}} > t^{1/(2\log(\tilde{d}_{k^*} ))}  \right)  \text{d}t  \right)^{1/(2 \log(\tilde{d}_{k^*}))} \\
					\leq & \left( \tilde{d}_{k^*} \int_{0}^{+\infty} \exp(-t^{1/\log(\tilde{d}_{k^*} )} \varsigma_1) \text{d}t + \tilde{d}_{k^*}  \int_{0}^{+\infty} \exp(-t^{1/(2\log(\tilde{d}_{k^*} ) )} \varsigma_2 ) \text{d}t  \right)^{1/(2 \log(\tilde{d}_{k^*}))} \\
					\leq & \sqrt{e} \left( \log(\tilde{d}_{k^*} ) \varsigma_1^{-\log( \tilde{d}_{k^*} )}  \Gamma(\log(\tilde{d}_{k^*} )) + 2 \log(\tilde{d}_{k^*}) \varsigma_2^{-2\log(\tilde{d}_{k^*})} \Gamma( \log(\tilde{d}_{k^*} ) ) \right)^{1/(2\log(\tilde{d}_{k^*}))}  . 
				\end{split}
			\end{equation}
			According to (47) in \cites{Suppklopp2011rank}, the Gamma function satisfies the bound:
			\begin{equation} \label{eq-gamma-function}
				\text{for } x\geq 2, ~\Gamma(x) \leq (\frac{x}{2})^{x-1}.  
			\end{equation}
			Therefore, combining (\ref{eq-t-star-inqueality-com}) and (\ref{eq-gamma-function}), we have 
			\begin{equation} \label{eq-expectation-exp}
				\expectation{ \trop{\Poo{\bm{\varepsilon}}}  }   \leq \sqrt{e} \left( (\log(\tilde{d}_{k^*} ))^{\log(\tilde{d}_{k^*})} \varsigma_1^{-\log( \tilde{d}_{k^*} )}  2^{1-\log(\tilde{d}_{k^*} )}+ 2 (\log(\tilde{d}_{k^*}))^{2\log(\tilde{d}_{k^*})} \varsigma_2^{-2\log(\tilde{d}_{k^*})}  \right) ^{1/(2\log(\tilde{d}_{k^*}))} .
			\end{equation}
			By setting $t=\log(\tilde{d}_{k^*})$ and letting the first term of (\ref{eq-inequality-bound}) be the maximum, we have $N\geq 2\hat{d}_{k^*} \log^2(\hat{d}_{k^*}) \log(\tilde{d}_{k^*})$. Substituting it to $\varsigma_1$, we can obtain  $\varsigma_1 \log (\tilde{d}_{k^*}) \leq \varsigma_2^2$, and reformulate (\ref{eq-expectation-exp}) as
			\begin{equation}
				\expectation{ \trop{\Poo{\bm{\varepsilon}}}  }   \leq c_6 \sqrt{\frac{ N \log(\tilde{d}_{k^*})}{\hat{d}_{k^*} }}. 
			\end{equation}
			Proof of Lemma \ref{lemma-bound-tr-expectation} is completed.
		\end{proof}
		\section{Proof of Lemma \ref{lemma-rsc}} \label{sec:proof-rsc}
		\begin{proof} 
			We define the absolute deviation between the  estimation error of observed entries  $ {\| \mathfrak{X}(\errorT)  \|_2^2}/{N}$ between the overall  estimation error ${\|\errorT \|_F^2}/{D}$ as 
			\begin{equation}
				Z_{T}: = \sup_{\errorT \in 	\mathcal{C}(\bm{r}, T) } \left|  \frac{\| \mathfrak{X}(\errorT)  \|_2^2}{N} -   \frac{\|\errorT \|_F^2}{D} \right|,
			\end{equation}
			where $\mathcal{C}(\bm{r}, T)$ is given as
			\begin{equation} \label{eq-set-r-T}
				\mathcal{C}(\bm{r}, T): = \left\{ \errorT: \errorT \in  \mathcal{C}(\bm{r}),  \frac{\|\errorT \|_F^2}{D} \leq T \right\}. 
			\end{equation}
			Then we show that the variable $Z_{T}$  concentrates around its expectation in the following Lemma and give the proof in Appendix \ref{sec:proof-concentration}.
			\begin{lemma} \label{lemma-concentration}
				There exists a constant $c_7$ such that 
				\begin{equation}
					\mathbb{P}( Z_T >  \frac{5}{12}T +  \frac{44 D^2  }{N^2} \left( \expectation{ \tropL{ \Poo{\varepsilon}} } \sum_{k=1}^{K} \alpha_{k}\sqrt{ r_k r_{k+s} }\right)^2) \leq \exp(-c_7 NT^2), 
				\end{equation}
				with $c_7=\frac{1}{128}$.
			\end{lemma}
			We first define the disjoint subsets of set $\mathcal{C}(\bm{r})$ as 
			\begin{equation}
				\mathcal{C}(\bm{r}, l) := \{ \errorT: \errorT \in  \mathcal{C}(\bm{r}),  \zeta \rho^{l-1}  \leq  \frac{\| \errorT \|_F^2}{ D} \leq \zeta \rho^{l} \},
			\end{equation}
			where $l \in \mathbb{N}_{+}$,  $\rho=6/5$, and $\zeta = \sqrt{\frac{64 \log \tilde{d}_{k^*} }{\log(\rho) N }} $. Then we define the following event 
			\begin{equation}
				\tilde{\mathcal{C}}:= \left\{ \exists \errorT \in  \mathcal{C}(\bm{r}), \text{s.t. } \left| \frac{\| \mathfrak{X}(\errorT)  \|_2^2}{N}  - \frac{\|\errorT \|_F^2}{D}\right| > \frac{\| \errorT \|_F^2}{2D} + 44  \frac{D }{N^2}  \left(\expectation{ \trop{\Poo{\varepsilon}} }  \sum_{k=1}^{K} \alpha_{k}\sqrt{r_k r_{k+s}}\right)^2  \right\}. 
			\end{equation}
			We can observe that the event $\tilde{\mathcal{C}}$ is the complement of the event that Lemma \ref{lemma-rsc} reveals. Thus, the aim is to prove the event   $\tilde{\mathcal{C}}$ holds with small probability. The sub-events of $\tilde{\mathcal{C}}$ can be given by 
			\begin{equation}
				\tilde{\mathcal{C}}_l:= \left\{ \exists \errorT \in  \mathcal{C}(\bm{r}, l), \text{s.t. } \left| \frac{\| \mathfrak{X}(\errorT)  \|_2^2}{N}  - \frac{\|\errorT \|_F^2}{D}\right| > \frac{5}{12} \zeta \rho^{l} + 44  \frac{D }{N^2}  \left(\expectation{ \trop{\Poo{\varepsilon}} }  \sum_{k=1}^{K} \alpha_{k}\sqrt{r_k r_{k+s}}\right)^2  \right\}, \forall l \in \mathbb{N}_{+}. 
			\end{equation}
			Note that  the event $ \tilde{\mathcal{C}}  = \cup_{l=1}^{\infty} \tilde{\mathcal{C}}_l$, which indicates that we can solve the sub-events separately. According to Lemma  \ref{lemma-concentration}, we have 
			\begin{equation} \label{eq-final-rsc}
				\mathbb{P}_{ \tilde{\mathcal{C}} } \leq \sum_{l=1}^{\infty} \mathbb{P}( 	\tilde{\mathcal{C}}_l  ) \leq \sum_{l=1}^{\infty} \exp (-c_7 N \zeta^2 \rho^{2l}) \leq \sum_{l=1}^{\infty} \exp( -2c_7 N\zeta^2 l \log(\rho) ) \leq \frac{\exp(-2c_7 N \zeta^2 \log(\rho))}{1-\exp(-2c_7 N \zeta^2 \log(\rho))}.
			\end{equation}
			By plugging $\zeta = \sqrt{\frac{64 \log \tilde{d}_{k^*} }{\log(\rho) N }} $ into (\ref{eq-final-rsc}), we conclude our proof.
		\end{proof}
		
		\section{Proof of Lemma \ref{lemma-concentration}}
		\label{sec:proof-concentration}
		We first give the expectation of $Z_T$ as 
		\begin{equation}
			\begin{split}
				\expectation{Z_T} =& \mathbb{E}\left[  \sup_{\errorT \in \mathcal{C}(\bm{r}, T) } \left|  \frac{1}{N} \sum_{n=1}^{N} \left\langle \tensor{X}_n,\errorT \right \rangle - \mathbb{E} \left[ { \| \mathfrak{X}(\errorT) \|_2^2 }  \right] \right|  \right] \\
				\leq & 2 \mathbb{E} \left[  \sup_{\errorT \in \mathcal{C}(\bm{r}, T) } \left| \frac{1}{N} \sum_{n=1}^{N} \epsilon_n \left\langle \errorT, \tensor{X}_n \right\rangle^2 \right| \right],
			\end{split}
		\end{equation}
		where the inequality uses the standard symmetrization argument \cites{Suppledoux2013probability}, $\{  \epsilon_n\}_{n=1}^{N}$ is i.i.d. Rademather sequence. According to the contraction inequality and the assumption $\| \errorT \|_{\infty} = 1$ in the set $\mathcal{C}(\bm{r})$, we have 
		\begin{equation}
			\expectation{Z_T} \leq \frac{8}{N}  \mathbb{E}\left[ \sup_{\errorT \in \mathcal{C}(\bm{r}, T) } \left|   \left\langle \errorT,  \Poo{\bm{\varepsilon} } \right\rangle \right|  \right], 
		\end{equation}
		then we can obtain that 
		\begin{equation}
			\begin{split}
				\expectation{Z_T} & \leq \frac{8}{N}  \mathbb{E}\left[ \sup_{\errorT \in \mathcal{C}(\bm{r}, T) } \trnn{\errorT} \trop{\Poo{\bm{\varepsilon} }}  \right]   \\
				&\leq \frac{8}{N}    \mathbb{E}\left[ \sup_{\errorT \in \mathcal{C}(\bm{r}, T) } \sum_{k=1}^{K} \alpha_k \sqrt{r_k r_{k+s}}  \| {\errorT}\|_F \trop{\Poo{\bm{\varepsilon} }}   \right] \\
				& \leq  \frac{8 \sqrt{DT} }{N}  \expectation{\trop{\Poo{\bm{\varepsilon} }} }  \sum_{k=1}^{K} \alpha_k \sqrt{r_k r_{k+s}},
			\end{split}
		\end{equation}
		where the first inequality uses the dual norm of TRNN, the second inequality uses the relationship between TRNN and Frobenius norm and the third inequality uses the assumption in (\ref{eq-set-r-T}).  Then we have the inequality
		\begin{equation} \label{eq-inequality-tvariable}
			\frac{1}{9}(\frac{5}{12} T) + \frac{8\sqrt{DT}}{N} \expectation{\trop{\Poo{\bm{\varepsilon} }}}  \leq (\frac{1}{9} + \frac{8}{9}) \frac{5}{12}T +  \frac{44{D}}{N^2} (\expectation{\trop{\Poo{\bm{\varepsilon} }}} )^2.
		\end{equation}
		Then, according to the Massart's concentration inequality 
		\begin{equation} \label{eq-massart-inequality}
			\mathbb{P}(Z_T \geq  \expectation{Z_T} + \frac{1}{9} (\frac{5}{12}T ) ) \leq \exp(-c_7nT^2),
		\end{equation}
		with $c_7=1/128$. Combining (\ref{eq-inequality-tvariable}) and (\ref{eq-massart-inequality}), we have  
		\begin{equation}
			\mathbb{P}(Z_T > \frac{5}{12} T +  \frac{44{D}}{N^2} (\expectation{\trop{\Poo{\bm{\varepsilon} }}} )^2 ) \leq \exp(-c_7NT^2).
		\end{equation}
		This completes the proof of Lemma \ref{lemma-concentration}.
	}

	\section{Proof of Theorem \ref{theorem-minimax-result}} \label{proof-minimax-result}
	
	Given a $K$th-order tensor $\tT$, without loss of generality, we assume that $d_{1,k^*} \geq d_{2,k^*}$, where $k^* = \arg\min_{k\in[K]} (d_{1,k} \wedge d_{2,k} )$. For a given constant $\gamma \leq 1$, define a set 
	\begin{equation}
		\bar{\mathcal{T}} = \left\{ \bar{\mat{T}}_{(k^*,s)} = \bar{\mat{T}}_{(k^*,s)}(i,j) \in \mathbb{R}^{d_{1,k^*} \times r_{k^*} r_{k^*+s} } : \bar{\mat{T}}_{(k^*,s)}(i,j)\in \left\{  0, \gamma (\delta \wedge \sigma) \sum_{k=1}^{K} \alpha_k\sqrt{\frac{K \check{d}_{k^*} }{N} r_{k} r_{k+s}  }\right\}  \right\},  
	\end{equation}
	and its augmenting matrix
	\begin{equation}
		\mathcal{T}=\left\{ \tT \in \mathbb{R}^{d_1 \times \cdots \times d_K}: \mat{T}_{(k^*,s)} = ( \bar{\mat{T}}_{(k^*,s)} \cdots \bar{\mat{T}}_{(k^*,s)} \mat{0}  ) \in \mathbb{R}^{d_{1,k^*} \times d_{2,k^*} }, \text{where }  \bar{\mat{T}}_{(k^*,s)} \in \bar{\mathcal{T}} \right\},
	\end{equation}
	where $\mat{0}$ is the zero matrix of size $d_{1,k^*} \times (d_{2,k^*} -r_{k^*} r_{k^*+s} \lfloor d_{2,k^*} /(r_{k^*} r_{k^*+s}) \rfloor  ) $. 
	
	According to the Varshamov-Gilbert bound (Lemma 2.9 in \cites{tsybakov2008introduction}), it can be guaranteed that there exists a subset $\mathcal{T}_0 \subset \mathcal{T}$ with cardinality $\text{Card}(\mathcal{T}_0) \geq 2^{r_{k^*} r_{k^*+s} \check{d}_{k^*} } + 1$, and for any distinct elements $\tT_1$ and $\tT_2$ of $\mathcal{T}_0$, we have 
	\begin{equation} \label{eq-fro-kl}
		\| \tT_1 - \tT_2 \|_F^2 \geq \frac{\check{d}_{k^*} r_{k^*} r_{k^*+s}  }{8} \left(  \gamma^2 (\delta^2 \wedge \sigma^2) \frac{K \check{d}_{k^*}  }{N} (\sum_{k=1}^{K} \alpha_k\sqrt{r_k r_{k+s}} )^2 \right) \lfloor \frac{ d_{2,k^*} }{r} \rfloor \geq \frac{\gamma^2}{16} (\delta^2 \wedge \sigma^2) \frac{D K \check{d}_{k^*} }{N}(\sum_{k=1}^{K} \alpha_k\sqrt{r_k r_{k+s}} )^2.
	\end{equation}
	Since the distribution of $\xi_n, n \in [N] $ is Gaussian, we get that, for any $\tT \in \mathcal{T}_0$, the Kulback-Leibler divergence $K(\mathbb{P}_{\mat{0}}, \mathbb{P}_{\tT} )$ between $\mathbb{P}_{\mat{0}}$ and $\mathbb{P}_{\tT} $ satisfies 
	\begin{equation} \label{eq-kl-divergence}
		K(\mathbb{P}_{\mat{0}}, \mathbb{P}_{\tT} ) = \frac{N}{2D\sigma^2} \| \tT \|_F^2 \leq \frac{\gamma^2 \check{d}_{k^*} r_{k^*} r_{k^*+s}}{2}.
	\end{equation} 
	From (\ref{eq-kl-divergence}), we can deduce the following condition
	\begin{equation}
		\frac{1}{\text{Card}(\mathcal{T}_0 ) - 1} \sum_{\tT \in \mathcal{T}_0} K(\mathbb{P}_{\mat{0}}, \mathbb{P}_{\tT} ) \leq \vartheta \log ( \text{Card}(\mathcal{T}_0 ) -1 ), 
	\end{equation}
	is satisfied if $\gamma >0$ and $\vartheta >0 $ are chosen as a sufficiently small numerical constant. According to Theorem 2.5 in \cites{tsybakov2008introduction},  combining (\ref{eq-fro-kl})  and  (\ref{eq-kl-divergence}), there exists a constant $c>0$, such that 
	\begin{equation}
		\inf_{\hat{\tT}} \sup_{\tT^*} \mathbb{P}_{\tT^*} \left( \frac{\|\hat{\tT} - \tT^* \|_F^2}{D} > \frac{c (\delta^2 \wedge \sigma^2) D K \check{d}_{k^*}  }{N}  (\sum_{k=1}^{K} \alpha_k \sqrt{r_k r_{k+s}} )^2\right) \geq \varrho,
	\end{equation}
	for some absolute constants $0<\varrho <1$.  Proof of Theorem \ref{theorem-minimax-result} is completed.
	\section{Proof of Lemma \ref{lemma-tr-tucker}} \label{proof-tr-tucker}
	\begin{proof}[Proof of Lemma \ref{lemma-tr-tucker}]
		Let   $\tT \in \mathbb{R}^{d_1 \times \cdots d_K}$ be a $K$th-order  tensor whose TR decomposition is $\tT = \text{TR} (\tensor{G}^{(1)}, \cdots, \tensor{G}^{(K)} )$. Since there are subcritical ($r_{k} r_{k+1} < d_k $), critical ($r_k r_{k+1} = d_k $) and supercritical $(r_k r_{k+1} > d_k)$ states in TR decomposition \cites{Suppye2018tensor}, we give discussions in  the following two cases.  
		
		If $d_k > r_k r_{k+1}$, we perform the skinny SVD on $\mat{G}_{(2)}^{(k)}$, that is, $[\mat{U}_k, \mat{S}_k, \mat{V}_k^\top] = \text{SVD}( \mat{G}_{(2)}^{(k)})$, and  reconstruct the core tensor $\tilde{\tensor{G}}^{(k)} = \text{fold}_{(2)}( \mat{S}_k \mat{V}_k^{\top} )$, where $\text{fold}_{(2)}(\cdot)$ denotes the canonical mode-2 folding operation,  and  $\tilde{\tensor{G}}^{(k)} \in \mathbb{R}^{r_k \times r_kr_{k+1}\times  r_{k+1} }$.  
		If $d_k \leq  r_k r_{k+1}$, we simply let $\mat{U}_k = \mat{I}_{d_k}$. 
		
		Therefore, each core tensor can be equivalently rewritten as
		\begin{equation}
			\tensor{G}^{(k)} = \tilde{\tensor{G}}^{(k)} \times_2 \mat{U}_k, k \in [K].
		\end{equation} 
		By directly extending the Theorem 4.2 in \cites{Suppzhao2016tensor} to the matrix case, the given tensor $\tT$ can be equivalently represented by 
		\begin{equation}
			\tT = \tilde{\tT} \times_1 \mat{U}_1 \cdots \times_K \mat{U}_K,
		\end{equation}
		where $\tilde{\tT} = \text{TR}(\tilde{\tensor{G}}^{(1)},\cdots, \tilde{\tensor{G}}^{(K)} ) \in \mathbb{R}^{R_1 \times R_2 \cdots \times R_K}$, and 
		\begin{equation}
			R_k = \left\{ 
			\begin{split}
				r_k r_{k+1}, &\text{   if } r_kr_{k+1} < d_k,\\
				d_k, &\text{   otherwise}.
			\end{split} \right.
		\end{equation}
		Therefore, tensor data with TR supercritical or critical state is actually a full-Tucker-rank tensor, i.e., the factor matrix $U_k$ is the identity matrix. 
		This completes the proof of Lemma \ref{lemma-multi-state} and Lemma \ref{lemma-tr-tucker}.
	\end{proof}
	
	\section{Proof of Theorem \ref{theorem-rank-equality}} 
	\label{proof-theorem-rank-equality}
	We first introduce a supporting Lemma to show the following equivalent relationship.
	\begin{lemma} \label{lemma-mode-k-shifting}
		Let $\tT \in \mathbb{R}^{d_1 \times d_2 \cdots \times d_K}$ be $K$th-order tensor, $\tilde{\tT}$ and $\mat{U}_k \in St(d_k, R_k), k\in [K]$ satisfy $\tensor{T}  = \tilde{\tensor{T}}  \times \mat{U}_1 \times_2 \cdots \times_K \mat{U}_K$, then its circular mode-$(k,s)$  unfolding can be formulated as   
		\begin{equation}    \label{eq-mode-k-shifting}
			\mat{T}_{(k,s)} = \left( \mat{U}_{K-s} \otimes \cdots \otimes \mat{U}_{l+1} \right) \tilde{\mat{T}}_{(k,s)} \trans{\left( \mat{U}_{l} \otimes \cdots \otimes \mat{U}_{k}\right) },
		\end{equation}
		where $l$ is defined in (\ref{eq-definition-circular-k}).
	\end{lemma}
	\begin{proof}[Proof of Lemma \ref{lemma-mode-k-shifting}]
		Note that for a given tensor $\tT$, its circular mode-$(K-s+1,s)$ unfolding is given by $\mat{T}_{(K-s+1,s)}$ of size $d_1d_2 \cdots d_{K-s} \times d_{K-s+1} \cdots d_K$.
		Thus, we have the  following vectorization form of $\tT$:
		\begin{equation}
			\begin{split}
				\vectorization{\mat{T}_{(K-s+1,s)} }=&	\vectorization{\tT} = \text{vec}({\tilde{\tT} \times_1 \mat{U}_1 \cdots \times_K \mat{U}_K })\\
				=&  \left( \mat{U}_{K} \otimes \cdots  \otimes \mat{U}_1  \right) \text{vec}({\tilde{\tensor{T}}}).
			\end{split}
		\end{equation}
		For circular mode-$(K-s+1,s)$  unfolding of $\tT$, we have 
		\begin{equation}\label{eq-unfold-1-conclusion}
			\begin{split}
				\ &\text{vec}\left( ( \mat{U}_{K-s} \otimes \cdots \otimes \mat{U}_1 ) \tilde{\mat{T}}_{(K-s+1,s)} ( \mat{U}_{K} \otimes \cdots \otimes \mat{U}_{K-s+1} )^{\top} \right) \\
				\	=&  ( \mat{U}_K \otimes \cdots   \otimes \mat{U}_{K-s+1} \otimes \mat{U}_{K-s} \otimes \cdots \otimes \mat{U}_1 ) \text{vec}\left( \tilde{\mat{T}}_{(K-s+1,s)} \right)  \\
				= & \text{vec}(\mat{T}_{(K-s+1,s)}), 
			\end{split}
		\end{equation}
		where the first equality uses $ ( \mat{A} \otimes\mat{B} ) \text{vec}(\mat{C}) = \text{vec}( \mat{B} \mat{C} \mat{A}^{\top} ) $. 
		Therefore, we have 
		\begin{equation} \label{eq-unfold-1-conclusion-2}
			\mat{T}_{(K-s+1,s)} =( \mat{U}_{K-s} \otimes \cdots \otimes \mat{U}_1 ) \tilde{\mat{T}}_{(K-s+1,s)} ( \mat{U}_K \otimes \cdots \otimes \mat{U}_{K-s+1} )^{\top}.
		\end{equation}
		
		Now, we are able to extend (\ref{eq-unfold-1-conclusion-2}) to any circular mode-$(k,s)$  unfolding by shifting $\tT$ and $\tilde{\tensor{T}}$ circularly into $\tT_{k} $ and $\tilde{\tensor{T}}_k $, respectively, 
		\begin{equation} \label{eq-permutation}
			\tT_{k} = \tilde{\tensor{T}}_{k} \times_1 \mat{U}_{l+1}  \cdots \times_{K-s} \mat{U}_{k-1} \times_{K-s+1} \mat{U}_{k}  \cdots \times_K \mat{U}_{l},
		\end{equation}
		where $\tT_k \in \mathbb{R}^{d_{l+1}\times \cdots  \times d_{K-s}  \times d_{K-s+1}  \cdots \times d_{l} }$ and $\tilde{\tensor{T}}_k \in \mathbb{R}^{R_{l+1}\times \cdots  \times R_{K-s}  \times R_{K-s+1}  \cdots \times R_{l} }$, and $l$ is defined in (\ref{eq-definition-circular-k}).
		Combining (\ref{eq-permutation}) and (\ref{eq-unfold-1-conclusion-2}), we can obtain (\ref{eq-mode-k-shifting}) directly. Proof of Lemma \ref{lemma-mode-k-shifting} is completed.
	\end{proof}
	
	\begin{proof}[Proof of Theorem \ref{theorem-rank-equality}]
		Suppose $\mat{U}_k \in St(d_k, R_k)$, and $R_k \geq r_k r_{k+1}$, $k \in [K]$. According to Lemma \ref{lemma-mode-k-shifting} and the determinant result of Kronecker product \cites{schacke2004kronecker}, we have:
		\begin{equation}
			\begin{split}
				&\left( \mat{U}_{K-s} \otimes \cdots \otimes \mat{U}_{l+1} \right) ^{\top} \left( \mat{U}_{K-s} \otimes \cdots \otimes \mat{U}_{l+1} \right)   \\
				=& \mat{U}_{K-s}^{\top} \mat{U}_{K-s} \otimes \cdots \otimes \mat{U}_{l+1}^{\top} \mat{U}_{l+1} \\
				= & \mat{I}_{J},
			\end{split}
		\end{equation}
		and similarly 
		\begin{equation}
			\begin{split}
				\trans{\left( \mat{U}_{l} \otimes \cdots \otimes \mat{U}_{k}\right) } {\left( \mat{U}_{l} \otimes \cdots \otimes \mat{U}_{k}\right) } =\mat{I}_{M},\\
			\end{split}
		\end{equation}
		where $J= \prod_{j=l+1}^{K-s}R_{j}$ and $M=\prod_{m=k}^{l} R_m$. 
		According to the Lemma 3 in \cites{shang2015robust}, we have $\| \mat{T}_{(k,s)} \|_* = \| \tilde{\mat{T}}_{(k,s)} \|_*$, and thus $\trnn{\tT} = \trnn{\tilde{\tT}}$.
		
		Proof of Theorem \ref{theorem-rank-equality} is completed.
	\end{proof}
	
	\section{Proof of Theorem \ref{theorem-optimal-solution}} 
	\label{proof-theorem-optimal-solution}
	\begin{proof}[Proof of Theorem \ref{theorem-optimal-solution}]
		Let $\tT'$ and $(\tTt', \{ \mat{U}_k' \}_{k=1}^{K} )$  be the global optimal solutions of (\ref{eq-constrainted-problem}) and (\ref{eq-constrainted-problem-faster}), respectively. Thus, we have
		\begin{equation} \label{eq-suboptimal-1}
			\begin{split}
				&\frac{1}{2} \fros{\mat{y} - \Po{\tT'}} +  \lambda \trnn{\tT'} \\
				\leq &\frac{1}{2}  \fros{\mat{y} - \Po{\tTt' \times_1 \mat{U}_1' \cdots \times_K \mat{U}_K' } } + \lambda \trnn{\tTt'}. 
			\end{split}
		\end{equation}
		Suppose $( r_k r_{k+1} \wedge d_k) \leq R_k \leq d_k$, $k\in [K]$, then we can perform Tucker decomposition $\tT' = \tTt^* \times_1 \mat{U}_1^* \cdots \times_K \mat{U}_K^*$, where  $\mat{U}_k^* \in St(d_k,R_k), k\in [K] $. Since $(\tTt', \{ \mat{U}_k' \}_{k=1}^{K} )$ is the global optimal solution of (\ref{eq-constrainted-problem-faster}), we have 
		\begin{equation} \label{eq-subpotimal-2}
			\begin{split}
				& \frac{1}{2} \fros{\mat{y} - \Po{\tTt' \times_1 \mat{U}_1' \cdots \times_K \mat{U}_K'  }} + \lambda \trnn{\tTt'} \\
				\leq & \frac{1}{2} \fros{\mat{y} - \Po{\tTt^* \times \mat{U}_1^* \cdots \times_K \mat{U}_K^* }} + \lambda \trnn{\tTt^*} \\
				= & \frac{1}{2} \fros{\mat{y} - \Po{\tT'}} + \lambda \trnn{\tT'}.
			\end{split} 
		\end{equation}
		According to (\ref{eq-suboptimal-1}) and (\ref{eq-subpotimal-2}), it can be observed that ($\{ \mat{U}_k' \}_{k=1}^K, \tTt' $) is  also the optimal solution of problem (\ref{eq-constrainted-problem}).
	\end{proof}

		\section{Proof of Theorem \ref{theorem-convergence-fantrc}}
		\label{sec-proof-convergence-fantrc}
			We first introduce the important lemma:
					\begin{lemma}\cites{Supplin2010augmented} \label{lemma-lin-hilbert}
							Let $\mathcal{H}$ be the real Hilbert space endowed with an inner product $\left\langle \cdot, \cdot \right\rangle$ and a corresponding norm $\| \cdot \|$, and $y \in \partial \| x \|$, where $\partial f(x)$ is the subgradient of $f(x)$. Then $\| y \|^* = 1$ if $x \neq 0$, and $\| y \|^* \leq 1$ if $x=0$, where $\| \cdot \|^*$ is the dual norm of $\|  \cdot \|$.
						\end{lemma}

		
		\emph{Part 1:}	We first show that the boundness of sequences $( \tT^{t},  \tTt^{t},  \{ \mat{U}_k^t \}_{k=1}^K, \{\tensor{L}^{k,t}\}_{k=1}^{K}, \{ \tensor{R}^k \}_{k=1}^K   )$ generated by Algorithm 2. 
		According to the first-order optimal conditions of problem (\ref{eq-Lk-subproblem}), we have 
		\begin{equation}
			0 \in  \frac{\lambda \alpha_k}{\eta^t}  \partial \| \mat{L}_{(k,s)}^{k,t+1} \|_*  +  \mat{L}_{(k,s)}^{k,t+1} 
			+\frac{1}{\eta^t}\mat{R}_{(k,s)}^{k,t} - \tilde{\mat{T}}_{(k,s)}^{t+1}, 
		\end{equation}
		and by (\ref{eq-update-R}), we have
		\begin{equation}
			- \frac{1}{\lambda \alpha_k} \mat{R}_{(k,s)}^{k,t+1} \in   \partial \| \mat{L}_{(k,s)}^{k,t+1} \|_*,
		\end{equation}
		according to Lemma \ref{lemma-lin-hilbert}, we have 
		\begin{equation} \label{eq-boundness-R}
			\frac{1}{\lambda \alpha_k} \| \mat{R}_{(k,s)}^{k,t+1} \| \leq 1.
		\end{equation}
		Thus, the sequence $ \{ \tensor{R}^{k,t} \}_{k=1}^K $ are bounded sequences. 
		
		According to (\ref{eq-update-Uk}), (\ref{eq-update-tTt}) and (\ref{eq-update-T}), we have 
		\begin{equation}
			\begin{split}
				& \ell_{\eta^t}^2 (\tT^{t+1},  \tTt^{t+1},  \{ \mat{U}_k^{t+1} \}_{k=1}^K, \{\tensor{L}^{k,t+1}\}_{k=1}^{K},\tensor{P}^{t}, \{\tensor{R}^{k,t}\}_{k=1}^K )   \\
				\leq & \ell_{\eta^{t}}^2 ( \tT^{t},  \tTt^{t},  \{ \mat{U}_k^{t} \}_{k=1}^K, \{\tensor{L}^{k,t}\}_{k=1}^{K},\tensor{P}^{t}, \{\tensor{R}^{k,t}\}_{k=1}^K ) \\
				\leq & \ell_{\eta^{t}}^2 ( \tT^{t},  \tTt^{t},  \{ \mat{U}_k^{t} \}_{k=1}^K, \{\tensor{L}^{k,t}\}_{k=1}^{K},\tensor{P}^{t-1}, \{\tensor{R}^{k,t-1}\}_{k=1}^K ) + \frac{\eta^{t-1} + \eta^{t} }{2(\eta^{t-1})^2} (\sum_{k=1}^{K} \| \tensor{R}^{k,t} -  \tensor{R}^{k,t-1} \|_F^2 + \| \tensor{P}^{k,t} -  \tensor{P}^{k,t-1} \|_F^2 ).
			\end{split}
		\end{equation}
		According to the boundness of sequence $\tensor{P}^t $ and (\ref{eq-boundness-R}), we can obtain that $ \ell_{\eta^t}^2 (\tT^{t+1},  \tTt^{t+1},  \{ \mat{U}_k^{t+1} \}_{k=1}^K, \{\tensor{L}^{k,t+1}\}_{k=1}^{K},\tensor{P}^{t}, \{\tensor{R}^{k,t}\}_{k=1}^K ) < \infty$.  Then according to (\ref{eq-Lagrangisn-faster}), we have 
		\begin{equation} \label{eq-bound-obj}
			\begin{split}
				&\frac{1}{2} \fros{ \mat{y} - \Po{\tT ^t} }  + \kappa_{\delta}^{\infty} (\tT^t) +  \lambda \sum_{k=1}^{K} \alpha_{k} \| \mat{L}^{k,t}_{(k,s)} \|_*   \\
				= & \ell_{\eta^t}^2 (\tT^{t},  \tTt^{t},  \{ \mat{U}_k^{t} \}_{k=1}^K, \{\tensor{L}^{k,t}\}_{k=1}^{K},\tensor{P}^{t-1}, \{\tensor{R}^{k,t-1}\}_{k=1}^K )  -    \sum_{k=1}^{K} (   \langle \tensor{L}^{k,t} - \tTt^t, \tensor{R}^{k,t} \rangle  +\frac{\eta^t}{2}  \fros{\tensor{L}^{k,t} - \tTt^{t} }  )  \\
				& - \langle \tensor{P}^{t-1},  \tT^{t} - \tTt^{t} \times_1 \mat{U}_1^t \cdots \times_K \mat{U}_K^t \rangle -\frac{\eta^t}{2} \fros{ \tT^t - \tTt^ \times_1 \mat{U}_1^t \cdots \times_K \mat{U}_K^t } \\
				= & \ell_{\eta^t}^2 (\tT^{t},  \tTt^{t},  \{ \mat{U}_k^{t} \}_{k=1}^K, \{\tensor{L}^{k,t}\}_{k=1}^{K},\tensor{P}^{t-1}, \{\tensor{R}^{k,t-1}\}_{k=1}^K )  - \frac{1}{2\eta^{t-1} }   \sum_{k=1}^{K} ( \| \tensor{R}^{k,t} \|_F^2 - \| \tensor{R}^{k,t-1} \|_F^2 ) \\
				& - \frac{1}{2\eta^{t-1}}  ( \| \tensor{P}^{t} \|_F^2 - \| \tensor{P}^{t-1} \|_F^2 ) < \infty.
			\end{split}
		\end{equation}
		By (\ref{eq-bound-obj}) and $\mat{U}_k^t \in St(d_k, R_k), k\in [K]$,  the sequences $( \tT^{t},  \tTt^{t},  \{ \mat{U}_k^t \}_{k=1}^K, \{\tensor{L}^{k,t}\}_{k=1}^{K}, \{ \tensor{R}^k \}_{k=1}^K   )$ generated by Algorithm 2 are bounded. 
		
		\emph{Part 2:}  Then we prove that the sequences generated by Algorithm 2 are Cauchy sequences.  By (\ref{eq-update-R}), we have
		\begin{equation}
			\sum_{t=1}^{\infty} \| \tensor{L}^{k,t} - \tTt^t \|_F = \sum_{t=1}^{\infty} \frac{1}{\eta^t} \| \tensor{R}^{k,t+1} - \tensor{R}^{k,t} \|_F. 
		\end{equation}
		Since  sequences $\{\tensor{R}^{k,t} \}_{k=1}^K$ are bounded and $\lim_{t \rightarrow \infty} \eta^t =  \infty$, we have $\lim_{t\rightarrow \infty} \| \tensor{L}^{k,t}  - \tTt^{t}\|_F  = 0$, that is, $\tensor{L}^{k,\infty}  = \tTt^{\infty}, k \in [K]$. 
		
		According to the first-order optimal condition of problem (\ref{eq-Lk-subproblem}) and $\exists \mat{F}^{k,t} \in \partial \| \mat{L}^{k,t}_{(k,s)} \|_*$ such that 
		\begin{equation}
			\frac{\lambda \alpha_k}{\mu^t} \mat{F}^{k,t}+ \mat{L}_{(k,s)}^{k,t} - \tMt^{t} + \frac{1}{\eta^t} \mat{R}_{(k,s)}^{k,t} = 0. 
		\end{equation}
		Similar to (\ref{eq-bound-M}), we have
		\begin{equation}
			\begin{split}
				\sum_{t=1}^{\infty} \| \mat{L}_{(k,s)}^{k,t} - \mat{L}_{(k,s)}^{k,t-1} \|_F &=  \sum_{t=1}^{\infty} \frac{\| (\nu-1)\mat{R}_{(k,s)}^{k,t} - \nu \mat{R}_{(k,s)}^{k,t-1} + {\lambda \alpha_k}\mat{F}^{k,t} \|_F}{\eta^{t}} \\
				& < {b}_{max} \sum_{t=1}^{\infty} \frac{1}{\eta^t} \\
				& = \frac{{b}_{max}}{\eta^0} \sum_{t=1}^{\infty} {\nu^{-t}}  \\
				&= \frac{{b}_{max}  }{\eta^0(\nu-1 )} < \infty.
			\end{split}
		\end{equation}
		where $b_{max} = \max(b_1, b_2,\cdots, b_{\infty})$, where  $b_t= \| (\nu-1)\mat{R}_{(k,s)}^{k,t} - \nu \mat{R}_{(k,s)}^{k,t-1} + {\lambda \alpha_k}\mat{F}^{k,t} \|_F$. Thus,  $  \{ \tensor{L}^{k,t} \}_{k=1}^K  $ are  Cauchy sequences. Similarly, it can be verified that  $\{  \tT^t, \tTt^{t}  \}$ are also Cauchy sequence.  
		
		\emph{Part 3:} We show that any limit point generated by Algorithm 2 satisfied the KKT conditions of problem (\ref{eq-constrainted-problem-faster}). The KKT conditions of problem (\ref{eq-constrainted-problem-faster}) are 
		\begin{equation} 
			\begin{split} 
				 0 \in & \mat{y} - \Po{\tT^* } + \lambda  \sum_{k=1}^{K}  \alpha_{k} \partial \| \tTt^{*} \|_* + \partial \kappa_{\delta}^{\infty} (\tT^*)  , \\
				& \tT^*=\tTt^* \times  {\mat{U}}_1^* \cdots \times_K \mat{U}_K^* , \\ 
				& \mat{U}_k^* \in  St(d_k, R_k), k\in [K].
			\end{split}
		\end{equation} 
		According to  the first-order optimal condition of problem (\ref{eq-update-Mk}), we have 
		\begin{equation} \label{eq-first-condition-L}
			0 \in \eta^t ( \tensor{L}^{k,t}  - \tTt^{t+1} )+  \tensor{R}^{k,t} -  {\lambda \alpha_k} \text{fold}(	\partial \| \mat{L}_{(k,s)}^{k,t} \|_*), k \in [K]. 
		\end{equation}
		Similarly, for problem (\ref{eq-update-tTt}), we have
		\begin{equation} \label{eq-first-condition-Ttilde}
			0 =  \tTt^{t+1} -   \frac{1}{K+1}  (\frac{1}{\eta^t} \tensor{P}^t + \tT ^{t}) \times_1 \mat{U}_{1}^{t+1,\top} \cdots \times_K \mat{U}_{K}^{t+1,\top} - \frac{1}{K+1} \sum_{k=1}^{K} \frac{1}{\eta^t} \tensor{R}^{k,t} - \tensor{L}^{k,t}.
		\end{equation}
		Since $\{ \tensor{L}^{k,t} \}_{k=1}^{K} $   and $\tTt^{t} $ are Cauchy sequences, and $\tensor{L}^{k,\infty} = \tTt^{\infty}$. Combining (\ref{eq-update-T}), (\ref{eq-first-condition-L}) and (\ref{eq-first-condition-Ttilde}), we can obtain that  
		\begin{equation}
			\begin{split}
				0 \in & {\lambda \alpha_k} \text{fold}(	\partial \| \tMt^{\infty} \|_*) +    \mat{y} - \Po{\tT^\infty} +  \partial  \kappa_{\delta}^{\infty} (\tT^\infty), \\
				&  \tT^\infty =\tTt^\infty \times  {\mat{U}}_1^\infty \cdots \times_K \mat{U}_K^\infty , \\ 
				& \mat{U}_k^\infty \in  St(d_k, R_k), k\in [K].
			\end{split}
		\end{equation}
		Therefore, the sequences $( \tT^t, \tTt^t,  \{ \tensor{L}^{k,t} \}_{k=1}^{K}, \{ \mat{U}_k^{t} \}_{k=1}^{K})$ generated by Algorithm 2 converges to the KKT point of problem (\ref{eq-constrainted-problem-faster}).
\end{appendices}


\bibliographystyles{IEEEbib}
\bibliographys{library_supp,refs_trnnm_supp}

\end{document}

%% file: mathCommand.tex
\usepackage[mathscr]{eucal}
\usepackage{relsize}
\usepackage{stmaryrd} 
\usepackage{amsmath}
\usepackage{amsthm} 

\theoremstyle{plain}

\newtheorem{definition}{Definition}
\newtheorem{remark}{Remark}
\newtheorem{theorem}{Theorem}
\newtheorem{lemma}{Lemma}
\newtheorem{assumption}{A}

\newcommand{\mat}[1]{\ensuremath{\boldsymbol{\mathbf{#1}}}}

\newcommand{\tensor}[1]{\ensuremath{\boldsymbol{\mathscr{{#1}}}}}
\newcommand{\vectorization}[1]{\ensuremath{\text{vec}\left(#1\right)}}

\newcommand{\trace}[1]{\ensuremath{\text{Tr}(#1)}}

\newcommand{\trans}[1]{\ensuremath{{#1}^{\top}}}

\newcommand{\hide}[1]{}